\documentclass[letter]{article}

\usepackage{ijcai16}
\usepackage{times}
\usepackage{amsthm,amstext,amssymb,amsmath}
\usepackage{multibib}
\newcites{A}{References of Appendix}

\usepackage{graphicx}
\usepackage[usenames,dvipsnames]{xcolor}
\usepackage{float}
\usepackage{enumerate}
\usepackage{bbm}
\usepackage{xspace}
\usepackage{pifont} 
\usepackage{misc} 
\usepackage{soul} 
\usepackage[normalem]{ulem} 
\usepackage{thmtools,thm-restate}
\usepackage{peterdefs} 

\usepackage{microtype}

\declaretheorem{theorem}
\declaretheorem[sibling=theorem]{lemma}

\declaretheorem[sibling=theorem]{proposition}
\declaretheorem[sibling=theorem,style=definition]{definition}

\declaretheorem[sibling=theorem]{example}

\newcommand{\ELplus}{\ensuremath{\mathcal{E\kern-0.1emLHF}\kern-0.2em_\bot}}
\newcommand{\Hornplus}{\ensuremath{\mathcal{E\kern-0.1emLHF}\kern-0.2em_\bot}}
\renewcommand{\vec}{\mathbf}

\def\cabox{\Bmc}
\def\ckabox{\ensuremath{\Bmc_{k_0}}}

\title{First Order-Rewritability and Containment of Conjunctive
  Queries in \\
Horn Description Logics}

\author{
     \hspace{-7mm}Meghyn Bienvenu \\[-0.5mm]
     \hspace{-7mm}{\normalsize CNRS, Univ. Montpellier, Inria, France} 
     \\[-1mm]
     \hspace{-7mm}{\normalsize meghyn@lirmm.fr} 
     \And
     \hspace{7mm}Peter Hansen \and Carsten Lutz \\[-0.5mm]
     \hspace{7mm}{\normalsize University of Bremen, Germany} \\[-1mm]
     \hspace{7mm}{\normalsize \{hansen, clu\}@informatik.uni-bremen.de}
     \And
     \hspace{15mm}Frank Wolter \\[-0.5mm]
     \hspace{15mm}{\normalsize University of Liverpool, UK} \\[-1mm]
     \hspace{15mm}{\normalsize frank@csc.liv.ac.uk}
}

\begin{document}
\maketitle

\begin{abstract}
  We study FO-rewritability of conjunctive queries in the
  presence of ontologies formulated in a description logic  
  between \EL and Horn-\SHIF, along with related query containment
  problems. Apart from providing characterizations, we
  establish complexity results ranging from \ExpTime via \NExpTime to
  2\ExpTime, pointing out several interesting effects. In particular,
  FO-rewriting is more complex for conjunctive queries than for atomic queries
  when inverse roles are present, but not otherwise.
\end{abstract}

\section{Introduction}

When ontologies are used to enrich incomplete and heterogeneous data
with a semantics and with background
knowledge~\cite{DBLP:conf/rweb/CalvaneseGLLPRR09,DBLP:conf/rweb/KontchakovRZ13,DBLP:conf/rweb/BienvenuO15},
efficient query anwering is a primary concern. Since classical
database systems are unaware of ontologies and implementing new
ontology-aware systems that can compete with these would be a huge
effort, a main approach used today is \emph{query rewriting}: the
user query $q$ and the ontology~\Omc are combined into a new query
$q_\Omc$ that produces the same answers as $q$ under \Omc (over all
inputs) and can be handed over to a database system for execution.
Popular target languages for the query $q_\Omc$ include SQL and
Datalog. In this paper, we concentrate on ontologies formulated in
description logics (DLs) and on rewritability into SQL, which we
equate with first-order logic~(FO).

FO-rewritability in the context of query answering under DL ontologies
was first studied in \cite{calvanese-2007}. Since FO-rewritings are
not guaranteed to exist when ontologies are formulated in traditional
DLs, the authors introduce the DL-Lite family of DLs specifically for
the purpose of ontology-aware query answering using SQL database
systems; in fact, the expressive power of DL-Lite is seriously
restricted, in this way enabling existence guarantees for
FO-rewritings. While DL-Lite is a successful family of DLs, there
are many applications that require DLs with greater expressive
power. The potential non-existence of FO-rewritings in this case is
not necessarily a problem in practical applications.  In fact,
ontologies emerging from such applications typically use the available
expressive means in a harmless way in the sense that efficient
reasoning is often possible despite high worst-case complexity. One
might thus hope that, in practice, FO-rewritings can often be
constructed also beyond DL-Lite.

This hope was confirmed in \cite{ijcai-2013,IJCAI15}, which consider
the case where ontologies are formulated in a DL of the \EL family
\cite{baader-2005} and queries are atomic queries (AQs) of the
form~$A(x)$. 
To describe the obtained results in more detail, let an
ontology-mediated query (OMQ) be a triple $(\Tmc,\Sigma,q)$ with \Tmc
a description logic TBox (representing an ontology), $\Sigma$ an ABox
signature (the set of concept and role names that can occur in the
data), and $q$ an actual query. Note that \Tmc and $q$ might use
symbols that do not occur in $\Sigma$; in this way, the TBox enriches
the vocabulary available for formulating $q$. We use $(\Lmc,\Qmc)$ to
denote the OMQ language that consists of all OMQs where \Tmc is
formulated in the description logic \Lmc and $q$ in the query language
\Qmc. In \cite{ijcai-2013}, FO-rewritability is characterized in terms
of the existence of certain tree-shaped ABoxes, covering a range of
OMQ languages between $(\EL,\text{AQ})$ and
$(\text{Horn-}\mathcal{SHI},\text{AQ})$. On the one hand, this
characterization is used to clarify the complexity of deciding whether
a given OMQ is FO-rewritable, which turns out to be \ExpTime-complete.
On the other hand, it provides the foundations for developing
practically efficient and complete algorithms for computing
FO-rewritings.  The latter was explored further in \cite{IJCAI15},
where a novel type of algorithm for computing FO-rewritings of OMQs
from $(\EL,\text{AQ})$ is introduced, crucially relying on the
previous results from \cite{ijcai-2013}. 
Its
evaluation shows excellent performance and confirms the hope that, in
practice, FO-rewritings almost always exist. In fact, rewriting fails
in only 285 out of 10989 test cases.

A limitation of the discussed results is that they concern only AQs
while in many applications, the more expressive conjunctive queries
(CQs) are required. The aim of the current paper is thus to study
FO-rewritability of OMQ languages based on CQs, considering ontology
languages between \EL and Horn-\SHIF.  In particular, we provide
characterizations of FO-rewritability in the required OMQ languages
that are inspired by those in \cite{ijcai-2013} (replacing tree-shaped
ABoxes with a more general form of ABox), and we analyze the
complexity of FO-rewritability using an automata-based approach. While
practically efficient algorithms are out of the scope of this
article, 
we believe that our work also lays important ground for the subsequent
development of such algorithms. 
Our approach actually \emph{does} allow the construction of
rewritings, but it is not tailored towards doing that in a practically
efficient way.
It turns out that the studied FO-rewritability problems are closely
related to OMQ containment problems as considered in
\cite{LuWo-KR12,BouLu-KR16}. In fact, being able to decide OMQ containment allows
us to concentrate on connected CQs when deciding FO-rewritability,
which simplifies technicalities considerably. For this reason, we also
study characterizations and the complexity of query containment in the
OMQ languages considered.

Our main complexity results are that FO-rewritability and containment
are \ExpTime-complete for OMQ languages between $(\EL,\text{AQ})$ and
$(\Hornplus,\text{CQ})$ and 2\ExpTime-complete for OMQ languages
between $(\ELI,\text{CQ})$ and $(\text{Horn-}\SHIF,\text{CQ})$. The
lower bound for containment applies already when both OMQs share the
same TBox. Replacing AQs with CQs thus results in an increase of
complexity by one exponential in the presence of inverse roles
(indicated by \Imc), but not otherwise. Note that the effect that
inverse roles can increase the complexity of querying-related problems
was known from expressive DLs of the \ALC family
\cite{DBLP:conf/cade/Lutz08}, but it has not previously been observed
for Horn-DLs such as \ELI and \text{Horn-}\SHIF. While 2\ExpTime might
appear to be very high complexity, we are fortunately also able to
show that the runtime is double exponential only in the size of the
actual queries (which tends to be very small) while it is only single
exponential in the size of the ontologies. We also show that the
complexity drops to {\sc NExpTime} when we restrict our attention to
\emph{rooted} CQs, that is, CQs which contain at least one answer
variable and are connected.  Practically relevant queries are
typically of this kind. 

A slight modification of our lower bounds yields new lower bounds for
monadic Datalog containment. In fact, we close an open problem from
\cite{DBLP:conf/pods/ChaudhuriV94} by showing that containment of a
monadic Datalog program in a rooted CQ is \coNExpTime-complete. We
also improve the 2\ExpTime lower bound for containment of a monadic
Datalog program in a CQ from \cite{DBLP:conf/icalp/BenediktBS12} by
showing that it already applies when the arity of EDB relations is
bounded by two, rule bodies are tree-shaped, and there are no
constants (which in this case correspond to nominals); the existing
construction cannot achieve the latter two conditions simultaneously.

\smallskip

Full proofs are provided at
http://www.informatik.uni-bremen.de/tdki/research/papers.html.

\smallskip

{\bf Related work.}  
Pragmatic approaches to OMQ rewriting beyond DL-Lite often consider
Datalog as a target language
\cite{rosati07on,perezurbina10tractable,DBLP:conf/aaai/EiterOSTX12,kaminski14,DBLP:journals/ws/TrivelaSCS15}.
These approaches might produce a non-recursive (thus FO) rewriting if
it exists, but there are no guarantees. FO-rewritability of OMQs based
on expressive DLs is considered in \cite{todswe}, and based on
existential rules in \cite{montpellier-2011}. A problem related to
ours is whether \emph{all} queries are FO-rewritable when combined
with a given TBox \cite{lutz-2012,DBLP:conf/cilc/CiviliR15}. There are
several related works in the area of Datalog; recall that a
Datalog program is bounded if and only if it is FO-rewritable
\cite{AG94}. For monadic Datalog programs, boundedness is known to be decidable
\cite{DBLP:conf/stoc/CosmadakisGKV88} and
2\ExpTime-complete~\cite{DBLP:conf/lics/BenediktCCB15}; containment is
also
2\ExpTime-complete~\cite{DBLP:conf/stoc/CosmadakisGKV88,DBLP:conf/icalp/BenediktBS12}.
OMQs from $(\text{Horn-$\mathcal{SHI}$},\text{CQ})$ can be translated
to monadic Datalog with an exponential blowup, functional roles
(indicated by \Fmc) are not expressible.

%

\section{Preliminaries and Basic Observations}
\label{sect:prelim}

Let $\NC$ and $\NR$ be disjoint and countably infinite sets of
\emph{concept} and \emph{role names}. A \emph{role} is a role name $r$
or an \emph{inverse role} $r^{-}$, with $r$ a role name.  A
\emph{Horn-$\mathcal{SHIF}$ concept inclusion (CI)} is of the form $L
\sqsubseteq R$, where $L$ and $R$ are concepts defined by the syntax
rules
%
\begin{align*} 
    R,R' &::= \top \mid \bot \mid A \mid \neg A \mid R \sqcap R' \mid 
\neg L \sqcup R \mid \exists r . R \mid
   \forall r . R \\
    L,L' &::= \top \mid \bot \mid A \mid L \sqcap L' \mid L \sqcup L' \mid 
    \exists r . L 
\end{align*}
with $A$ ranging over concept names and $r$ over roles.  In DLs,
ontologies are formalized as TBoxes. A \emph{Horn-$\mathcal{SHIF}$
  TBox}~$\Tmc$ is a finite set of Horn-$\mathcal{SHIF}$ CIs,
\emph{functionality assertions} ${\sf func}(r)$, \emph{transitivity
  assertions} ${\sf trans}(r)$, and \emph{role inclusions (RIs)} $r
\sqsubseteq s$, with $r$ and $s$ roles.  It is standard to assume that
functional roles are not transitive and neither are transitive roles
included in them (directly or indirectly). We make the slighty
stronger assumption that functional roles do not occur on the
right-hand side of role inclusions at all. This assumption seems
natural from a modeling perspective and mainly serves the purpose of
simplifying constructions; all our results can be extended to
the milder standard assumption.  
%
An \emph{$\mathcal{ELIHF}_{\bot}$ TBox} is a
\emph{Horn-$\mathcal{SHIF}$ TBox} that contains neither transitivity
assertions nor disjunctions in~CIs, an \emph{\ELI TBox} is an
$\mathcal{ELIHF}_{\bot}$ TBox that contains neither functionality
assertions nor RIs, and an \ELplus TBox is an
$\mathcal{ELIHF}_{\bot}$ TBox that does not contain inverse roles. 

An \emph{ABox} is a finite set of \emph{concept assertions} $A(a)$ and
\emph{role assertions} $r(a,b)$ where $A$ is a concept name, $r$ a
role name, and $a,b$ individual names from a countably infinite set
\NI. We sometimes write $r^-(a,b)$ instead of $r(b,a)$ and use
$\mn{Ind}(\Amc)$ to denote the set of all individual names used in
\Amc. A \emph{signature} is a set of concept and role names. We will
often assume that the ABox is formulated in a prescribed signature,
which we then call an \emph{ABox signature}.  An ABox that only uses
concept and role names from a signature $\Sigma$ is called a
\emph{$\Sigma$-ABox}. 

The semantics of DLs is given in terms of \emph{interpretations}
$\Imc=(\Delta^\Imc,\cdot^\Imc)$, where $\Delta^\Imc$ is a non-empty
set (the \emph{domain}) and $\cdot^\Imc$ is the \emph{interpretation
  function}, assigning to each $A\in \NC$ a set $A^\Imc \subseteq
\Delta^\Imc$ and to each $r\in \NR$ a relation $r^\Imc \subseteq
\Delta^\Imc \times \Delta^\Imc$.
The interpretation $C^{\Imc}\subseteq \Delta^{\Imc}$ of a concept $C$
in $\Imc$ is defined as usual, see \cite{baader-2003-dl-handbook}. An
interpretation $\Imc$ \emph{satisfies} a CI $C \sqsubseteq D$ if
$C^{\Imc} \subseteq D^{\Imc}$, a functionality assertion
$\mn{func}(r)$
if $r^\Imc$ is a partial function, a transitivity assertion ${\sf
  trans}(r)$ if $r^{\Imc}$ is transitive, an RI $r\sqsubseteq s$ if
$r^{\Imc} \subseteq s^{\Imc}$, a concept assertion $A(a)$ if $a \in
A^{\Imc}$,
and a role
assertion $r(a,b)$ if $(a,b)\in r^{\Imc}$. 
We say that $\Imc$ is
a \emph{model of} a TBox or an ABox if it satisfies all inclusions and
assertions in it. An ABox $\Amc$ is \emph{consistent} with~a TBox
$\Tmc$ if $\Amc$ and $\Tmc$ have a common model.  If $\alpha$ is a
CI, RI, or functionality assertion, we
write $\Tmc \models \alpha$ if all models of \Tmc satisfy $\alpha$.

A \emph{conjunctive query (CQ)} takes the form $q=\exists \xbf \, \vp
(\xbf,\ybf)$ with $\xbf,\ybf$ tuples of variables and $\vp$ a
conjunction of atoms of the form $A(x)$ and $r(x,y)$ that uses only
variables from $\xbf \cup \ybf$. The variables in $\ybf$ are called
\emph{answer variables}, the \emph{arity} of $q$ is the length of
$\ybf$, and $q$ is \emph{Boolean} if it has arity zero.  
An
\emph{atomic query (AQ)} is a conjunctive query of the form $A(x)$.
A \emph{union of conjunctive queries (UCQ)} is a disjunction of
CQs that share the same answer variables.
Ontology-mediated queries (OMQs) and the notation $(\Lmc,\Qmc)$
for OMQ languages were already defined in the introduction.
We generally assume that if a role name $r$ occurs in~$q$ and $\Tmc\models s
\sqsubseteq r$, then $\mn{trans}(s) \notin \Tmc$. This is 
common since allowing transitive roles in the query poses serious
additional complications, which are outside the scope of this paper;
see e.g.\
\cite{DBLP:conf/dlog/BienvenuELOS10,DBLP:conf/icalp/GottlobPT13}.

Let $Q=(\Tmc,\Sigma,q)$ be an OMQ, $q$ of arity $n$, \Amc a
$\Sigma$-ABox and $\abf \in \mn{Ind}(\Amc)^n$. We write 
$\Amc \models Q(\abf)$ if $\Imc
\models q(\abf)$ for all models \Imc of \Tmc and \Amc. In this case,
$\abf$ is a \emph{certain answer} to $Q$ on \Amc.  We use
$\mn{cert}(Q,\Amc)$ to denote the set of all certain answers to~$Q$ on
\Amc.
%

A first-order query (FOQ) is a first-order formula $\vp$ constructed
from atoms $A(x)$, $r(x,y)$, and $x=y$; here, concept names are viewed
as unary predicates, role names as binary predicates, and predicates
of other arity, function symbols, and constant symbols are not
permitted. We write $\vp(\xbf)$ to indicate that the free
variables of $\vp$ are among $\xbf$ and call $\xbf$ the \emph{answer
  variables} of $\vp$. The number of answer variables is the
\emph{arity} of $\vp$ and $\vp$ is \emph{Boolean} if it has
arity zero.  We use $\mn{ans}(\Imc,\vp)$ to denote the set of 
answers to the FOQ $\vp$ on the interpretation~\Imc; that is, if $\vp$
is $n$-ary, then $\mn{ans}(\Imc,\vp)$ contains all tuples $\dbf \in
(\Delta^{\Imc})^n$ with $\Imc\models \vp(\dbf)$.  To bridge the gap
between certain answers and answers to FOQs, we
sometime view an ABox $\Amc$ as an interpretation $\Imc_\Amc$, defined
in the obvious way. 


%

For any syntactic object $O$ (such as a TBox, a query, an OMQ), we use $|O|$ to
denote the \emph{size} of $O$, that is, the number of symbols needed
to write it (concept and role names counted as a single symbol).
\begin{definition}[FO-rewriting]
  An FOQ $\vp$ is an \emph{FO-rewriting} of an OMQ $Q=(\Tmc,\Sigma,q)$
  if $\mn{cert}(Q,\Amc)=\mn{ans}(\Imc_\Amc,\varphi)$ for all
  $\Sigma$-ABoxes \Amc that are consistent with \Tmc\!\!\!. If there is
  such a $\vp$, then $Q$ is \emph{FO-rewritable}.
\end{definition}
\begin{example}\label{ex:1}
(1) Let $Q_{0}=(\Tmc_{0},\Sigma_{0},q_{0}(x,y))$, where $\Tmc_{0} = \{\exists r.A \sqsubseteq A,B \sqsubseteq \forall r.A\}$,
$\Sigma_{0}=\{r,A,B\}$ and $q_{0}(x,y)=B(x) \wedge r(x,y) \wedge A(y)$. Then $\vp_{0}(x,y)=B(x) \wedge r(x,y)$ is an
FO-rewriting of $Q_{0}$. 

We will see in Example~\ref{ex:2} that the query $Q_{A}$ obtained from
$Q_{0}$ by replacing $q_{0}(x,y)$ with the AQ $A(x)$ is not
FO-rewritable (due to the unbounded propagation of $A$ via $r$-edges
by $\Tmc_0$). Thus, an FO-rewritable OMQ can give raise to AQ
`subqueries' that are not FO-rewritable.

(2) Let $Q_{1}=(\Tmc_{1},\Sigma_{1},q_{1}(x))$, where
$\Tmc_{1}=\{\exists r.\exists r.A \sqsubseteq \exists r.A\}$,
$\Sigma_{1}=\{r,A\}$, and $q_{1}(x)= \exists y (r(x,y)\wedge
A(y))$. Then $Q_{1}$ is not FO-rewritable (see again Example~\ref{ex:2}), but
all AQ subqueries that $Q_{1}$ gives raise to are FO-rewritable.
\end{example}
%
The main reasoning problem studied in this paper is to decide whether
a given OMQ $Q=(\Tmc,\Sigma,q)$ is FO-rewritable.  We assume without
loss of generality that every symbol in $\Sigma$ occurs in~\Tmc or in
$q$.  We obtain different versions of this problem by varying the OMQ
language used.  Note that we have defined FO-rewritability relative to
ABoxes that are consistent with the TBox. It is thus important for the
user to know whether that is the case. Therefore, we also consider
FO-rewritability of ABox inconsistency. More precisely, we say that
\emph{ABox inconsistency is FO-rewritable} relative to a TBox \Tmc and
ABox signature $\Sigma$ if there is a Boolean FOQ $\vp$ such that for
every $\Sigma$-ABox~$\Amc$, \Amc is inconsistent with $\Tmc$ iff
$\Imc_{\Amc}\models \varphi()$.
%

Apart from FO-rewritability questions, we will also study OMQ containment.
Let $Q_i=(\Tmc_i,\Sigma,q_i)$ be two OMQs over the same
ABox signature.  We say that $Q_1$ \emph{is contained in} $Q_2$, in
symbols $Q_1 \subseteq Q_2$, if $\mn{cert}(Q_1,\Amc) \subseteq
\mn{cert}(Q_2,\Amc)$ holds for all $\Sigma$-ABoxes $\Amc$ that are
consistent with $\Tmc_1$ and $\Tmc_2$.


%
%


We now make two basic observations that we use in an essential way in
the remaining paper. We first observe that it suffices to concentrate
on \ELIHFbot TBoxes \Tmc in \emph{normal form}, that is, all CIs are
of one of the forms
$ A \sqsubseteq \bot, A \sqsubseteq \exists r . B, \top \sqsubseteq A,
B_1 \sqcap B_2 \sqsubseteq A, \exists r . B \sqsubseteq A$
%
%
with $A,B,B_1,B_2$ concept names and $r$ a role. We
use $\mn{sig}(\Tmc)$ to denote the concept and role names that occur
in \Tmc.
%
%
%
\begin{restatable}{proposition}{THMreduce}\label{thm:reduce}
  Given a Horn-$\mathcal{SHIF}$ (resp.\ \ELplus) TBox~$\Tmc_1$ and
  ABox signature $\Sigma$, one can construct in polynomial time an
  \ELIHFbot (resp.\ \ELplus) TBox $\Tmc_2$ in normal
  form such that for every $\Sigma$-ABox~$\Amc$,
\begin{enumerate}
\item $\Amc$ is consistent with $\Tmc_1$ iff $\Amc$ is consistent
with $\Tmc_2$;
\item if $\Amc$ is consistent with $\Tmc_1$, then for any CQ $q$ 
that does not use symbols from $\mn{sig}(\Tmc_2) \setminus
  \mn{sig}(\Tmc_1)$, we have $\mn{cert}(Q_1,\Amc)=
\mn{cert}(Q_2,\Amc)$ where $Q_i=(\Tmc_i,\Sigma,q)$.
\end{enumerate}
\end{restatable}
Theorem~\ref{thm:reduce} yields polytime reductions of
FO-rewritability in $(\text{Horn-}\mathcal{SHIF},\Qmc)$ to
FO-rewritability in $(\ELIHFbot,\Qmc)$ for any query language $\Qmc$,
and likewise for OMQ containment and FO-rewritability of ABox
inconsistency. It also tells us that, when working with $\ELplus$
TBoxes, we can assume normal form.  Note that transitioning
from $(\text{Horn-}\mathcal{SHF},\Qmc)$ to $(\ELHFbot,\Qmc)$ is not as
easy as in the case with inverse roles since universal restrictions on
the right-hand side of concept inclusions cannot easily be eliminated;
for this reason, we do not consider
$(\text{Horn-}\mathcal{SHF},\Qmc)$.  From now on, we work with
TBoxes formulated in $\ELIHFbot$ or $\ELplus$ and assume without
further notice that they are in normal form.
%
%

Our second observation is that, when deciding FO-rewritability, we can
restrict our attention to connected queries provided that we have
a way of deciding containment (for potentially disconnected queries).  We use conCQ
to denote the class of all connected CQs.
%
\begin{restatable}{theorem}{THMtoconnected}\label{thm:toconnected}
  Let $\Lmc \in \{ \mathcal{ELIHF_\bot}, \ELplus \}$.
  Then FO-rewritability in $(\Lmc,\text{CQ})$ can be solved in polynomial time
  when there is access to oracles for containment in $(\Lmc,\Qmc)$ and
  for FO-rewritability in $(\Lmc,\text{conCQ})$.
\end{restatable}
To prove Theorem~\ref{thm:toconnected}, we observe that
FO-rewritability of an OMQ $Q=(\Tmc,\Sigma,q)$ is equivalent to
FO-rewritability of all OMQs $Q=(\Tmc,\Sigma,q_c)$ with $q_c$ a maximal
connected component of $q$, excluding certain redundant such components
(which can be identified using containment).  Backed by
Theorem~\ref{thm:toconnected}, we generally assume connected
queries when studying FO-rewritability, which allows to avoid
unpleasant technical complications and is a main reason for studying
FO-rewritability and containment in the same paper.

\section{Main Results}
\label{sect:mainres}

In this section, we summarize the main results established in this paper. 
We start with the following theorem.
\begin{theorem}
\label{thm:main}
  FO-rewritability and containment are 
  \begin{enumerate}

  \item 2\ExpTime-complete for any OMQ language between
    $(\ELI,\text{CQ})$ and $(\text{Horn-}\SHIF,\text{CQ})$, and 

  \item \ExpTime-complete for any OMQ language between $(\EL,\text{AQ})$ and $(\Hornplus,\text{CQ})$.

  \end{enumerate}
  Moreover, given an OMQ from $(\text{Horn-}\SHIF,\text{CQ})$ that is
  FO-rewritable, one can effectively construct a UCQ-rewriting.
\end{theorem}
Like the subsequent results, Theorem~\ref{thm:main} illustrates the
strong relationship between FO-rewritability and containment. Note
that inverse roles increase the complexity of both reasoning tasks.
We stress that this increase takes place only when the actual queries
are conjunctive queries, since FO-rewritability for OMQ languages with
inverse roles and atomic queries is in \ExpTime
\cite{ijcai-2013}.

The 2\ExpTime-completeness result stated in Point~1 of
Theorem~\ref{thm:main} might look discouraging. However, the
situation is not quite as bad as it seems. To show this, we state
the upper bound underlying Point~1 of
Theorem~\ref{thm:main} a bit more carefully.
\begin{theorem}
\label{thm:mainzoom}
  Given OMQs $Q_i=(\Tmc_i,\Sigma_i,q_i)$, $i \in \{1,2\}$, from
  $(\text{Horn-}\SHIF,\text{CQ})$, it can be decided 
  \begin{enumerate}

  \item in time $2^{2^{p(|q_1| + \mn{log}(|\Tmc_1|)) }}$ whether $Q_1$ is
    FO-rewritable and

  \item in time $2^{2^{p(|q_1| + |q_2| + \mn{log}(|\Tmc_1| +
        |\Tmc_2|)) }}$ whether $Q_1 \subseteq
    Q_2$, 

  \end{enumerate}
  for some polynomial $p$. 
\end{theorem}
Note that the runtime is double exponential only in the size of the
actual queries $q_1$ and $q_2$, while it is only single exponential in
the size of the TBoxes $\Tmc_1$ and~$\Tmc_2$. This is good news since
the size of $q_1$ and $q_2$ is typically very small compared to the
sizes of $\Tmc_1$ and $\Tmc_2$. For this reason, it can even be
reasonable to assume that the sizes of $q_1$ and $q_2$ are constant,
in the same way in which the size of the query is assumed to be
constant in data complexity. Note that, under this
assumption, Theorem~\ref{thm:mainzoom} yields \ExpTime upper
bounds.

One other way to relativize the seemingly very high complexity stated
in Point~1 of Theorem~\ref{thm:main} is to observe that the lower
bound proofs require the actual query to be Boolean or
disconnected. In practical applications, though, typical queries are
connected and have at least one answer variable. We call such CQs
\emph{rooted} and use rCQ to denote the class of all rooted CQs. Our
last main result states that, when we restrict our attention to rooted
CQs, then the complexity drops to \coNExpTime. 
\begin{theorem}
\label{thm:mainrooted}
  FO-rewritability and containment are \coNExpTime-complete in any OMQ
  language between $(\ELI,\text{rCQ})$ and
  $(\text{Horn-}\SHIF,\text{rCQ})$.
\end{theorem}
%

\section{Semantic Characterization}
\label{sect:charact}


The upper bounds stated in Theorems~\ref{thm:main}
and~\ref{thm:mainzoom} are established in two steps. We first
give characterizations of FO-rewritability in terms of the
existence of certain (almost) tree-shaped ABoxes, and then utilize
this characterization to design decision procedures based on
alternating tree automata. The semantic characterizations are of
independent interest.

An ABox \Amc is \emph{tree-shaped} if the undirected graph with nodes
${\sf Ind}(\Amc)$ and edges $\{\{a,b\}\mid r(a,b)\in \Amc \}$ is
acyclic and connected and $r(a,b)\in \Amc$ implies that (i)~$s(a,b)
\notin \Amc$ for all $s \neq r$ and (ii)~$s(b,a) \notin \Amc$ for all
role names $s$.  For tree-shaped ABoxes \Amc, we often distinguish an
individual used as the root, denoted with $\rho_\Amc$.  \Amc
is \emph{ditree-shaped} if the directed graph with nodes ${\sf
  Ind}(\Amc)$ and edges $\{(a,b)\mid r(a,b)\in \Amc \}$ is a tree and
$r(a,b)\in \Amc$ implies (i) and (ii).  The (unique) root of a ditree-shaped
ABox \Amc is also denoted with $\rho_\Amc$.  

An ABox \Amc is a \emph{pseudo tree} if it is the union of ABoxes
$\Amc_0,\dots,\Amc_k$ that satisfy the following conditions:
\begin{enumerate}


\item $\Amc_{1},\dots,\Amc_k$ are tree-shaped;

\item $k \leq |\mn{Ind}(\Amc_{0})|$;

\item $\Amc_i \cap \Amc_0 = \{ \rho_{\Amc_i} \}$ and $\mn{Ind}(\Amc_{i})
  \cap \mn{Ind}(\Amc_{j}) = \emptyset$, for $1 \leq i < j \leq k$.

\end{enumerate}
We call $\Amc_0$ the \emph{core} of $\Amc$ and $\Amc_1,\dots,\Amc_k$
the \emph{trees} of $\Amc$. The \emph{width} of $\Amc$ is
$|\Ind(\Amc_0)|$, its \emph{depth} is the depth of the deepest tree of
\Amc, and its \emph{outdegree} is the maximum outdegree of the ABoxes
$\Amc_{1},\dots,\Amc_k$.  For a pseudo tree ABox \Amc and \mbox{$\ell
  \geq 0$}, we write $\Amc|_{\leq \ell}$ to denote the restriction of
\Amc to the individuals whose minimal distance from a core individual
is at most~$\ell$, and analogously for $\Amc|_{>\ell}$.  A
\emph{pseudo ditree ABox} is defined analogously to a pseudo tree
ABox, except that $\Amc_1,\dots,\Amc_k$ must be ditree-shaped.
%
%

When studying FO-rewritability and containment, we can restrict our
attention to pseudo tree ABoxes, and even to pseudo ditree ABoxes when
the TBox does not contain inverse roles. The following statement makes
this precise for the case of containment. Its proof uses unraveling
and
compactness.
\begin{restatable}{proposition}{PROPconttree}\label{prop:conttree}
\label{prop:conttree}
Let $Q_i=(\Tmc_i,\Sigma,q_i)$, $i \in \{1,2\}$, be OMQs from
(\ELIHFbot, CQ). Then $Q_1 \not\subseteq Q_2$ iff there is a pseudo
tree $\Sigma$-ABox \Amc of outdegree at most $|\Tmc_1|$ and width at
most~$|q_{1}|$ that is consistent with both $\Tmc_1$ and $\Tmc_2$ and a
tuple \abf from the core of \Amc such that $\Amc \models Q_1(\abf)$
and $\Amc \not\models Q_2(\abf)$.

If $Q_1,Q_2$ are from $(\ELplus,\text{CQ})$, then we can find a 
pseudo ditree ABox with these properties.
\end{restatable}
%
%
We now establish a first version of the announced characterizations of
FO-rewritability. Like Proposition~\ref{prop:conttree}, they are based on 
pseudo tree ABoxes. 
\begin{restatable}{theorem}{LEMchar}\label{lem:char}
  Let $Q=(\Tmc,\Sigma,q)$ be an OMQ from
  (\ELIHFbot, conCQ). If
  the arity of $q$ is at least one, then
%
%
the following conditions are equivalent:
\begin{enumerate}

\item $Q$ is FO-rewritable; 

\item there is a $k \geq 0$ such that for all pseudo tree
  $\Sigma$-ABoxes \Amc that are consistent with \Tmc and of outdegree 
  at most~$|\Tmc|$ and width
  at most $|q|$: if 
$\Amc \models Q(\vec{a})$ with $\vec{a}$ from
  the core of $\Amc$, then $\Amc|_{\leq k} \models 
  Q(\vec{a})$;

\end{enumerate}
If $q$ is Boolean, this equivalence holds with (2.) replaced
by
\begin{itemize}

\item[2$'$.] there is a $k \geq 0$ such that for all pseudo tree
  $\Sigma$-ABoxes \Amc that are consistent with \Tmc and of outdegree
  at most~$|\Tmc|$ and
  of width at most $|q|$: if $\Amc \models Q$, then $\Amc|_{>0}\models
  Q$ or $\Amc|_{\leq k} \models Q$.

\end{itemize}
If $Q$ is from $(\ELplus,\text{conCQ})$, then the above equivalences
hold also when pseudo tree $\Sigma$-ABoxes are replaced with 
pseudo ditree $\Sigma$-ABoxes.
\end{restatable}
The proof of Proposition~\ref{prop:conttree} gives a good intuition of
why FO-rewritability can be characterized in terms of ABoxes that are
pseudo trees. In fact, the proof of ``$2 \Rightarrow 1$'' of
Theorem~\ref{lem:char} is similar to the proof of
Proposition~\ref{prop:conttree}. The proof of ``$1 \Rightarrow 2$''
uses locality arguments in the form of Ehrenfeucht-Fra\"iss\'e
games. The following examples further illustrate
Theorem~\ref{lem:char}.
%
\begin{example}\label{ex:2}
(1) Non FO-rewritability of the OMQs $Q_{A}$ and $Q_{1}$ from Example~\ref{ex:1}
is shown by refuting Condition~2 in Theorem~\ref{lem:char}: let 
$\Amc_{k}=\{r(a_{0},a_{1}),\ldots,r(a_{k},a_{k+1}),A(a_{k+1})\}$, for all $k\geq 0$.
Then $\Amc_{k}\models Q(a_{0})$ but $\Amc_{k}|_{\leq k}\not\models Q(a_{0})$ for $Q\in \{Q_{A},Q_{1}\}$.

(2) Theorem~\ref{lem:char} only holds for connected CQs:
consider $Q_{2}=(\Tmc_{2},\Sigma_{2},q_{2})$, where $\Tmc_{2}$ is the empty TBox, $\Sigma_{2}=\{A,B\}$,
and $q_{2}= \exists x\exists y(A(x) \wedge B(y))$. $Q_{2}$ is FO-rewritable ($q_{2}$ itself is a rewriting),
but Condition~$2'$ does not hold: for $\Bmc_{k}= \{A(a_{0}),R(a_{0},a_{1},\ldots,R(a_{k},a_{k+1}),B(a_{k+1})\}$
we have $\Bmc_{k}\models Q_{2}$ but $\Bmc_{k}|_{>0}\not\models Q_{2}$ and $\Bmc_{k}|_{\leq k}\not\models Q_{2}$.

(3) The modification $2'$ of Condition~2 is needed to characterize FO-rewritability of Boolean OMQs: 
obtain $Q_{B}$ from $Q_{2}$ by replacing $q_{2}$ with $\exists x B(x)$. Then $Q_{B}$ is FO-rewritable,
but the ABoxes $\Bmc_{k}$ show that Condition $2$ does not hold.
\end{example}
Theorem~\ref{lem:char} does not immediately suggest a decision
procedure for FO-rewritability since there is no bound on the depth of
the pseudo tree ABoxes \Amc used. The next result establishes such a
bound.
%
%
\begin{restatable}{theorem}{LEMwitabox}\label{lem:witabox}
  Let \Tmc be an \ELIHFbot TBox. 
  Then Theorem~\ref{lem:char} still holds with the following modifications:
  %
  \begin{enumerate}
%

  \item if $q$ is not Boolean or $\Tmc$ is an \ELHFbot TBox, 
       ``there is a $k\geq 0$'' is replaced with ``for $k=|q|+2^{4(|\Tmc|+|q|)^{2}}$'';

  \item if $q$ is Boolean, ``there is a $k\geq 0$'' is replaced with ``for $k=|q|+2^{4(|\Tmc| + 2^{|q|})^{2}}$''.
  \end{enumerate}
\end{restatable}
The proof of Theorem~\ref{lem:witabox} uses a pumping argument based
on derivations of concept names in the pumped ABox by~\Tmc. Due to the
presence of inverse roles, this is not entirely trivial and uses what
we call \emph{transfer sequences}, describing the derivation history
at a point of an ABox. Together with the proof of
Theorem~\ref{lem:char}, Theorem~\ref{lem:witabox} gives rise to an
algorithm that constructs actual rewritings when they exist.

\section{Constructing Automata}
\label{sect:automata}

We show that Proposition~\ref{prop:conttree} and
Theorem~\ref{lem:witabox} give rise to automata-based decision
procedures for containment and FO-rewritability that establish the
upper bounds stated in Theorems~\ref{thm:main} and~\ref{thm:mainzoom}.
By Theorem~\ref{thm:toconnected}, it suffices to consider connected
queries in the case of FO-rewritability. We now observe that we can
further restrict our attention to Boolean queries. We use BCQ (resp.\
conBCQ) to denote the class of all Boolean CQs (resp.\ connected
Boolean CQs).
\begin{restatable}{lemma}{LEMbooleancontainment}\label{lem:boolean-containment} 
  Let $\Lmc \in \{
  \mathcal{ELIHF_\bot}, \ELplus\}$. Then
  \begin{enumerate}

  \item FO-rewritability in $(\Lmc,\text{conCQ})$ can be reduced in polytime 
    to FO-rewritability in $(\Lmc,\text{conBCQ})$;

  \item Containment in $(\Lmc,\text{CQ})$ can be reduced in polytime 
    to containment in $(\Lmc,\text{BCQ})$.

  \end{enumerate}
\end{restatable}
The decision procedures rely on building automata that accept pseudo
tree ABoxes which witness non-containment and non-FO-rewritability as
stipulated by Proposition~\ref{prop:conttree} and
Theorem~\ref{lem:witabox}, respectively.  We first have
to encode pseudo tree ABoxes in a suitable way.

%
%
A \emph{tree}
is a non-empty (and potentially infinite) set $T \subseteq \Nbbm^*$
closed under prefixes. We say that $T$ is \emph{$m$-ary} if for every
$x \in T$, the set $\{ i \mid x \cdot i \in T \}$ is of cardinality at
most $m$.  For an alphabet $\Gamma$, a \emph{$\Gamma$-labeled tree} is
a pair $(T,L)$ with $T$ a tree and $L:T \rightarrow \Gamma$ a node
labeling function. Let $Q=(\Tmc,\Sigma,q)$ be an OMQ from
$(\mathcal{ELIHF}_\bot,\text{conBCQ})$.
We encode pseudo tree ABoxes of width at most $|q|$
and outdegree at most $|\Tmc|$
by $(|\Tmc|\cdot |q|)$-ary $\Sigma_\varepsilon \cup \Sigma_N$-labeled
trees, where $\Sigma_\varepsilon$ is an alphabet used for labeling
root nodes and $\Sigma_N$ is for non-root nodes.

The alphabet $\Sigma_\varepsilon$ consists of all $\Sigma$-ABoxes
\Amc such that $\Ind(\Amc)$ only contains individual names from a
fixed set $\Ind_{\mn{core}}$ of size $|q|$ 
and \Amc satisfies all functionality statements in
$\Tmc$.
The alphabet $\Sigma_N$ consists of all subsets
$
  \Theta \subseteq (\NC \cap \Sigma) \uplus 
  \{ r, r^{-} \mid r \in \NR \cap \Sigma \} \uplus \Ind_{\mn{core}} 
  $ that contain exactly one (potentially inverse) role and at most
  one element of $\Ind_{\mn{core}}$. A $(|\Tmc| \cdot |q|)$-ary
  $\Sigma_\varepsilon \cup\Sigma_N$-labeled tree is \emph{proper} if
  (i)~the root node is labeled with a symbol from~$\Sigma_\varepsilon$,
  (ii)~each child of the root is labeled with a symbol from $\Sigma_N$ that contains
  an element of $\Ind_{\mn{core}}$, (iii)~every other non-root
  node is labeled with a symbol from $\Sigma_N$ that contains no individual name,
  and (iv)~every non-root node has at most $|q|$ successors and (v)~for
  every $a \in \mn{Ind}_{\mn{core}}$, the  root node has at most $|q|$
  successors whose label includes~$a$.
  %

  A proper $\Sigma_\varepsilon \cup \Sigma_N$-labeled tree $(T, L)$ represents 
  a pseudo tree ABox $\Amc_{(T, L)}$ whose
  individuals are those in the ABox \Amc that labels the root of $T$ plus
  all non-root nodes of $T$, and whose assertions are 
$$
\begin{array}{l}
   \Amc \cup \{ A(x) \mid A \in L(x) \}\\[0.5mm] 
   \cup\; \{r(b,x) \mid \{b,r\} \subseteq L(x)\} \cup \{r(x,b) \mid 
  \{b,r^{-}\} \subseteq L(x)\}\\[0.5mm]
   \cup\; \{ r(x,y) \mid r \in L(y), y \text{ is a child of } x, 
    L(x) \in \Sigma_N\}\\[0.5mm]  
   \cup\; \{ r(y,x) \mid  r^{-} \in L(y), y \text{ is a child of } x, 
  L(x) \in \Sigma_N \}. 
\end{array}
$$
%
%

As the automaton model, we use two-way alternating parity automata on
finite trees (TWAPAs). As usual, $L(\Amf)$ denotes the tree language
accepted by the TWAPA \Amf.  Our central observation is the following.
%

%
\begin{restatable}{proposition}{PROPautocentral}\label{prop:autocentral} 
\label{prop:automataexists}
For every OMQ $Q=(\Tmc,\Sigma,q)$ from
$(\mathcal{ELIHF}_\bot,\text{BCQ})$, there is a TWAPA
\begin{enumerate}

\item $\Amf_Q$ that accepts a $(|\Tmc| \cdot |q|)$-ary
  $\Sigma_\varepsilon \cup \Sigma_N$-labeled tree $(T, L)$ iff
  it is proper, 
  $\Amc_{(T, L)}$ is consistent
  with \Tmc\!\!\!, and $\Amc_{(T, L)} \models Q$;
%
%

  $\Amf_Q$ has at most $2^{p(|q|+\mn{log}(|\Tmc|))}$ states, and at
  most $p(|q|+|\Tmc|)$ states if \Tmc is an
$\mathcal{ELHF}_\bot$ TBox, $p$ a polynomial.

\item $\Amf_\Tmc$ that accepts a $(|\Tmc| \cdot |q|)$-ary
  $\Sigma_\varepsilon \cup \Sigma_N$-labeled tree $(T, L)$ iff
  it is proper 
  and $\Amc_{(T, L)}$ is consistent with \Tmc.

  $\Amf_\Tmc$ has at most $p(|\Tmc|)$ states, $p$ a polynomial.

\end{enumerate}
We can construct $\Amf_Q$ and $\Amf_\Tmc$  in time polynomial in their size.
\end{restatable}
The construction of the automata in
Proposition~\ref{prop:automataexists} uses forest decompositions of
the CQ $q$ as known for example from \cite{DBLP:conf/cade/Lutz08}. The
difference in automata size between $\mathcal{ELIHF}_\bot$ and
$\mathcal{ELHF}_\bot$ is due to the different number of tree-shaped
subqueries that can arise in these decompositions.

To decide $Q_1 \subseteq Q_2$ for OMQs $Q_i=(\Tmc_i,\Sigma,q_i)$, $i
\in \{1,2\}$, from $(\mathcal{ELIHF}_\bot,\text{BCQ})$, by
Proposition~\ref{prop:conttree} it suffices to decide whether
$L(\Amf_{Q_1}) \cap L(\Amf_{\Tmc_2})\subseteq L(\Amf_{Q_2})$. Since
this question can be polynomially reduced to a TWAPA emptiness check
and the latter can be executed in time single exponential in the
number of states, this yields the upper bounds for containment stated
in Theorems~\ref{thm:main} and~\ref{thm:mainzoom}.

To decide non-FO-rewritability of an OMQ $Q=(\Tmc,\Sigma,q)$ from
$(\mathcal{ELIHF}_\bot,\text{conBCQ})$, by Theorem~\ref{lem:witabox}
we need to decide whether there is a pseudo tree $\Sigma$-ABox \Amc of
outdegree at most $|\Tmc|$ and width at most $|q|$ that is consistent
with \Tmc and satisfies (i)~$\Amc \models Q$, (ii)~$\Amc|_{>0}
\not\models Q$, and (iii)~$\Amc|_{\leq k} \not\models Q$ where
$k=|q|+2^{4(|\Tmc| + 2^{|q|})^{2}}$. For consistency with \Tmc and
for~(i), we use the automaton $\Amf_Q$ from
Proposition~\ref{prop:automataexists}. To achieve (ii) and (iii), we
amend the tree alphabet $\Sigma_\varepsilon \cup \Sigma_n$ with
additional labels that implement a counter which counts up to $k$ and
annotate each node in the tree with its depth (up to~$k$). We then
complement $\Amf_Q$ (which for TWAPAs can be done in polynomial time),
relativize the resulting automaton to all but the first level of the
input ABox for~(ii) and to the first $k$ levels for~(iii), and finally
intersect all automata and check emptiness. This yields the upper
bounds for FO-rewritability stated in Theorems~\ref{thm:main}
and~\ref{thm:mainzoom}.

As remarked in the introduction, apart from FO-rewritability of an OMQ
$(\Tmc,\Sigma,q)$ we should also be interested in FO-rewritability of
ABox inconsistency relative to \Tmc and $\Sigma$. We close this
section with noting that an upper bound for this problem can be
obtained from Point~2 of Proposition~\ref{prop:automataexists}
since TWAPAs can be complemented in polynomial time.
A matching
lower bound can be found in \cite{ijcai-2013}.
\begin{theorem}
\label{thm:aboxconsforewr}
  In $\mathcal{ELIHF}_\bot$,  FO-rewritability of ABox inconsistency
  is \ExpTime-complete. 
\end{theorem}

\section{Rooted Queries and Lower Bounds}

We first consider the case of rooted queries and establish the upper bound
in Theorem~\ref{thm:mainrooted}. 
\begin{theorem}
\label{thm:rootedupper}
  FO-rewritability and containment in
  $(\mathcal{ELIHF}_\bot,\text{rCQ})$ are in \coNExpTime.
\end{theorem}
Because of space limitations, we confine ourselves to a brief sketch,
concentrating on FO-rewritability. By Point 1 of Theorem
\ref{lem:witabox}, deciding non-FO-rewritability of an OMQ
$Q=(\Tmc,\Sigma,q)$ from $(\mathcal{ELIHF}_\bot,\text{rCQ})$ comes
down to checking the existence of a pseudo tree $\Sigma$-ABox $\Amc$
that is consistent with $\Tmc$ and
such that $\Amc \models Q(\vec{a})$ and $\Amc|_{\leq k} \not\models
Q(\vec{a})$ for some tuple of individuals $\vec{a}$ from the core of
$\Amc$, for some suitable $k$. Recall that $\Amc \models Q(\vec{a})$
if and only if there is a homomorphism $h$ from $q$ to the pseudo
tree-shaped canonical model of \Tmc and \Amc that takes the answer
variables to $\vec{a}$.  Because $\vec{a}$ is from the core of $\Amc$
and $q$ is rooted, $h$ can map existential variables in $q$ only to
individuals from $\Amc|_{|q|}$ and to the anonymous elements in the
subtrees below them. To decide the existence of \Amc, we can thus
guess $\Amc|_{|q|}$ together with sets of concept assertions about
individuals in $\Amc|_{|q|}$ that can be inferred from \Amc and \Tmc,
and from $\Amc|_{\leq k}$ and \Tmc. We can then check whether there is
a homomorphism $h$ as described, without access to the full ABoxes
$\Amc$ and $\Amc|_{\leq k}$. It remains to ensure that the guessed
initial part $\Amc_{|q|}$ can be extended to \Amc such that the
entailed concept assertions are precisely those that were guessed, by
attaching tree-shaped ABoxes to individuals on level $|q|$. This can
be done by a mix of guessing and automata techniques.

We next establish the lower bounds stated in Theorems~\ref{thm:main}
and~\ref{thm:mainrooted}. For Theorem~\ref{thm:main}, we
only prove a lower bound for Point~1 as the one
in Point~2 follows from \cite{ijcai-2013}.
\begin{theorem}
\label{thm:lower}
   Containment and FO-rewritability are
  \begin{enumerate}
  \item \coNExpTime-hard in $(\ELI,\text{rCQ})$ and 
  \item 2\ExpTime-hard in $(\ELI,\text{CQ})$.
  \end{enumerate}
  The results for containment apply already when both OMQs
  share the same TBox.
\end{theorem}
Point~1 is proved by reduction of the problem of tiling a torus of
exponential size, and Point~2 is proved by reduction of the word
problem of exponentially space-bounded alternating Turing machines
(ATMs).  The proofs use queries similar to those introduced in
\cite{DBLP:conf/cade/Lutz08} to establish lower bounds on the
complexity of query answering in the expressive OMQ languages
$(\ALCI,\text{rCQ})$ and $(\ALCI,\text{CQ})$. A major difference to
the proofs in \cite{DBLP:conf/cade/Lutz08} is that we represent torus
tilings / ATM computations in the ABox that witnesses
non-containment or non-FO-rewritability, instead of in the `anonymous
part' of the model created by existential quantifiers. 
%
%



The proof of Point~2 of Theorem~\ref{thm:lower} can be modified to
yield new lower bounds for monadic Datalog 
containment. Recall that the rule body of a Datalog program is a
CQ. \emph{Tree-shapedness} of a CQ $q$ is defined in the same way as
for an ABox in Section~\ref{sect:charact}, that is, $q$ viewed as an
undirected graph must be a tree without multi-edges. 
%
\begin{restatable}{theorem}{THMdatalog}\label{thm:dlog}
  For monadic Datalog programs which contain no EDB relations of arity
  larger than two and no constants, containment
  \begin{enumerate}

  \item in a rooted CQ is \coNExpTime-hard;

  \item in a CQ is 2\ExpTime-hard, even when all rule bodies 
    are tree-shaped.

  \end{enumerate}
%
\end{restatable}
Point~1 closes an open problem from
\cite{DBLP:conf/pods/ChaudhuriV94}, where a \coNExpTime upper bound
for containment of a monadic Datalog program in a rooted UCQ was
proved and the lower bound was left open.  Point~2 further improves a
lower bound from \cite{DBLP:conf/icalp/BenediktBS12} which also does
not rely on EDB relations of arity larger than two, but requires
that rule bodies are not tree-shaped or constants are present (which,
in this case, correspond to nominals in the DL world).


\section{Conclusion}

A natural next step for future work is to use the techniques developed
here for devising practically efficient algorithms that construct
actual rewritings, which was very successful in the AQ case
\cite{IJCAI15}.

An interesting open theoretical question is the complexity of
FO-rewritability and containment for the OMQ languages considered in
this paper in the special case when the ABox signature contains all
concept and role names.

\smallskip
\noindent {\bf Acknowledgements.} Bienvenu was supported by ANR
project PAGODA (12-JS02-007-01), Hansen and Lutz by
ERC grant 647289, Wolter by EPSRC UK grant EP/M012646/1.

\clearpage
\bibliographystyle{named}


\clearpage
\appendix

{\noindent\LARGE\textbf{Appendix}}

\section{Proofs for Section~\ref{sect:prelim}}

\medskip
\noindent
\THMreduce*

\medskip

\noindent
\begin{proof}
The proof is similar to reductions provided in 
\citeA{journals/jar/HustadtMS07,conf/ijcai/Kazakov09}.
We sketch the proof for Horn-$\mathcal{SHIF}$. The proof for
\ELplus is similar and omitted.

Assume a $\mathcal{SHIF}$ TBox $\Tmc$ is given. The following rules are used to rewrite \Tmc into an
$\mathcal{ELIHF}_{\bot}$ TBox in normal form. It then only remains to eliminate the
transitivity assertions. We assume that the concept names introduced in the rules below are fresh 
(not in ${\sf sig}(\Tmc)\cup \Sigma$):
\begin{itemize}
\item If $L$ is of the form $L_{1}\sqcap L_{2}$ and $R$ is not a concept name,
then take a fresh concept name $A$ and replace $L \sqsubseteq R$ by $L \sqsubseteq A$ and $A \sqsubseteq R$.
If $R$ is a concept name,
and either $L_{1}$ or $L_{2}$ are not concept names, then take fresh concept names $A_{1},A_{2}$
and replace $L\sqsubseteq R$ by $L_{1}\sqsubseteq A_{1}$,
$L_{2}\sqsubseteq A_{2}$ and $A_{1}\sqcap A_{2}\sqsubseteq R$;
\item If $L$ is of the form $L_{1}\sqcup L_{2}$ and $R$ is a concept name, then
replace $L\sqsubseteq R$ by $L_{1}\sqsubseteq R$ and $L_{2}\sqsubseteq R$. Otherwise
take a fresh concept name $A$ and replace $L \sqsubseteq R$ by $L \sqsubseteq A$ and
$A \sqsubseteq R$;
\item If $L$ is of the form $\exists r.L'$ and $L'$ is not a concept name, 
then take a fresh concept name $A'$ and replace 
$L\sqsubseteq R$ by $L' \sqsubseteq A'$ and $\exists r.A' \sqsubseteq R$;
\item If $R$ is of the form $\neg A$, then replace $L\sqsubseteq R$ by $L \sqcap A\sqsubseteq \bot$;
\item If $R$ is of the form $R_{1} \sqcap R_{2}$ and $L$ is not a concept name, then take a fresh concept
name $A$ and replace $L \sqsubseteq R$ by $L \sqsubseteq A$ and $A \sqsubseteq R$. Otherwise take 
fresh concept names $A_{1},A_{2}$ and replace $L\sqsubseteq R$ by $L\sqsubseteq A_{1}$,
$L\sqsubseteq A_{2}$, $A_{1}\sqsubseteq R_{1}$, and $A_{2} \sqsubseteq R_{2}$;
\item If $R$ is of the form $\neg L' \sqcup R'$, then replace $L\sqsubseteq R$ by 
$L \sqcap L' \sqsubseteq R'$;
\item If $R$ is of the form $\exists r.R'$ and $R'$ is not a concept name, then take a fresh concept name 
$A'$ and replace $L\sqsubseteq R$ by $L \sqsubseteq \exists r.A'$ and $A' \sqsubseteq R'$;
\item  If $R$ is of the form $\forall r.R'$, then replace $L \sqsubseteq R$ by
$\exists r^{-}.L \sqsubseteq R$.
\end{itemize}
The resulting TBox $\Tmc'$ is a conservative extension of $\Tmc$; i.e., it has the following two properties:
\begin{itemize}
\item $\Tmc'\models \Tmc$;
\item every model \Imc of $\Tmc$ can be extended to a model of $\Tmc'$ by appropriately
interpreting the fresh concept names.
\end{itemize} 
Now we show how transitivity assertions can be eliminated from $\Tmc'$:
for any role $r$ with $\Tmc\models {\sf trans}(r)$ and concept name $B$ take a fresh concept name 
$X$ and add the CIs $\exists r.B \sqsubseteq X$, $\exists r. X \sqsubseteq X$, and $X \sqsubseteq \exists r.B$
to $\Tmc'$. Also remove the transitivity assertions from $\Tmc'$. The resulting TBox, $\Tmc''$, is an $\mathcal{ELIHF}_{\bot}$
TBox and has the following two properties (we call a role name $r$ \emph{simple relative to $\Tmc$} if there does
not exist a role $s$ with $\Tmc\models {\sf trans}(s)$ and $\Tmc\models s \sqsubseteq r$):
\begin{itemize}
\item every model of $\Tmc'$ can be extended to a model of $\Tmc''$ by appropriately
interpreting the fresh concept names of $\Tmc''$;
\item for every model $\Imc$ of $\Tmc''$ there exists a model $\Jmc$ 
of $\Tmc'$ which coincides with $\Imc$ regarding the interpretation of concept names 
and regarding the interpretation of role names $r$ that are simple relative to $\Tmc$. Moreover,
for role names $r$ that are not simple relative to $\Tmc$ we have $r^\Jmc \supseteq r^{\Imc}$.
\end{itemize}
It follows that $\Tmc''$ is as required since role names that are not simple relative to $\Tmc$ do not occur in 
any CQs in OMQs.
\end{proof}
We require the following standard characterization of FO-definability.
Let $\Imc$ and $\Jmc$ be interpretations and $\abf=a_{1},\ldots,a_{n}$ a sequence of individual names. 
Then $\Imc$ and $\Jmc$ are called \emph{$m$-equivalent for $\Sigma$ and $\abf$}, in symbols
$\Imc \equiv_{m,\Sigma,\vec{a}} \Jmc$, if $\Imc$ and $\Jmc$ satisfy the same first-order sentences
of quantifier rank $\leq m$ using predicates from $\Sigma$ and individual constants from $\vec{a}$ only.
The following characterization of FO-definability is well known and can be proved in a straightforward way.
\begin{lemma}
\label{lem:m-rewrite}
Let $Q=(\Tmc,\Sigma,q)$ be an OMQ. Then $Q$ is not FO-rewritable iff for all $m>0$ there are $\Sigma$-ABoxes $\Amc_{m}$ and 
$\Bmc_{m}$ that are consistent with $\Tmc$ and there is $\vec{a}\in 
\mn{Ind}(\Amc_{m})\cap \mn{Ind}(\Bmc_{m})$  
such that 
\begin{itemize}
\item $\Amc_{m},\Tmc \models q(\vec{a})$ and $\Bmc_{m},\Tmc\not\models q(\vec{a})$ and
\item $\Imc_{\Amc_{m}} \equiv_{m,\Sigma,\vec{a}} \Imc_{\Bmc_{m}}$.
\end{itemize}
\end{lemma}

We use Lemma~\ref{lem:m-rewrite} to prove Theorem~\ref{thm:toconnected}.

\medskip 
\noindent 
\THMtoconnected*

\smallskip 
\noindent 
\begin{proof}
Let $Q=(\Tmc,\Sigma,q)$ be an OMQ in $(\Lmc,\text{CQ})$. Assume $q(\vec{x})= \exists \vec{y}.\varphi(\vec{x},\vec{y})$.
The polynomial time algorithm is as follows:
\begin{enumerate}
\item Let $\vec{x}_{1},\ldots,\vec{x}_{k}$ and $\vec{y}_{1},\ldots,\vec{y}_{k}$ be mutually disjoint subsets of
$\vec{x}$ and $\vec{y}$, respectively, such that 
\begin{eqnarray*}
\Gamma  & = & \{q_{1}(\vec{x}_{1})= \exists \vec{y}_{1}\varphi_{1}(\vec{x}_{1},\vec{y}_{1}),
\ldots, \\
&   & \hspace*{0.3cm}q_{k}(\vec{x}_{k})= \exists \vec{y}_{k}\varphi_{k}(\vec{x}_{k},\vec{y}_{k})\}
\end{eqnarray*}
is the set of maximal connected subqueries of $q$. 
\item Obtain $\Gamma'$ from $\Gamma$ by removing Boolean CQs $q_{j}$ that are entailed by the remaining
CQs as follows: set $\Gamma_{0}=\Gamma$ and assume $\Gamma_{0},\ldots,\Gamma_{j}$ have been defined for some $j<k$.
Then set $\Gamma_{j+1}:= \Gamma_{j}\setminus \{q_{j+1}\}$
if $q_{j+1}$ is Boolean and 
$$
\Amc,\Tmc\models \bigwedge_{q_{i}\in \Gamma_{j}\setminus \{q_{j+1}\}}q_{i}(\vec{a}_{i})\quad \Rightarrow \quad \Amc,\Tmc\models q_{j}
$$
holds for all $\Sigma$-ABoxes $\Amc$ and all $\vec{a}_{i}$ in $\mn{Ind}(\Amc)$. Otherwise set $\Gamma_{j+1}:=\Gamma_{j}$.
Let $\Gamma':=\Gamma_{k}$.
Clearly, $\Gamma'$ can be computed using an oracle for containment in $(\Lmc,\text{CQ})$.
\item Check FO-rewritability of $(\Tmc,\Sigma,q_{i})$ for all $q_{i}\in \Gamma'$ using an oracle
for FO-rewritability in $(\Lmc,\text{conCQ})$.
\item Output `$Q$ is FO-rewritable' iff all $q_{i}\in \Gamma'$ are FO-rewritable.
\end{enumerate}  
The following claim establishes the correctness of this algorithm.

\medskip
\noindent
{\bf Claim}. $Q$ is FO-rewritable iff all $(\Tmc,\Sigma,q_{j})$ with $q_{j}\in \Gamma'$ are FO-rewritable.

\medskip
\noindent
The direction from right to left is trivial. 
Conversely, assume that some $(\Tmc, \Sigma,q_{j})$ with $q_{j}\in \Gamma'$
is not FO-rewritable. By Lemma~\ref{lem:m-rewrite} we find, for all $m>0$, 
$\Sigma$-ABoxes $\Amc_{m}$ and $\Bmc_{m}$ that are consistent relative to $\Tmc$ and 
$\vec{a}_{j}\in \mn{Ind}(\Amc_{m})\cap \mn{Ind}(\Bmc_{m})$ of the same length as $\vec{x}_{j}$ 
such that 
\begin{itemize}
\item $\Amc_{m},\Tmc \models q_{j}(\vec{a}_{j})$ and $\Bmc_{m},\Tmc\not\models q_{j}(\vec{a}_{j})$;
\item $\Imc_{\Amc_{m}} \equiv_{m,\Sigma,\vec{a}_{j}} \Imc_{\Bmc_{m}}$.
\end{itemize}
Consider the query 
$$
q'(\vec{x}')= \bigwedge_{q_{i}(\vec{x}_{i})\in \Gamma'\setminus \{q_{j}(\vec{x}_{j})\}}q_{i}(\vec{x}_{i}).
$$
Observe that $q(\vec{x}) = q(\vec{x}',\vec{x}_{j})$ and that $q(\vec{x})$ is equivalent to 
$q_{j}(\vec{x}_{j}) \wedge q'(\vec{x}')$.  
We distinguish two cases.

\medskip

(1) If $q_{j}$ is not Boolean, then take some $\Sigma$-ABox $\Amc$ that is 
consistent relative
to $\Tmc$ and with $\mn{Ind}(\Amc)\cap \mn{Ind}(\Amc_{m})=\emptyset$ and 
$\mn{Ind}(\Amc)\cap \mn{Ind}(\Bmc_{m})=\emptyset$ for all $m>0$ such that 
$\Amc,\Tmc\models q'(\vec{a}')$ for some $\vec{a}'$ in $\mn{Ind}(\Amc)$ of the same length as $\vec{x}'$.
We obtain for all $m>0$:
\begin{itemize}
\item $\Amc_n \cup \Amc,\Tmc\models q(\vec{a}',\vec{a}_{j})$ and $\Bmc_{n}\cup \Amc\not\models q(\vec{a}',\vec{a}_{j})$;
\item $\Imc_{\Amc_{n}\cup \Amc}\equiv_{n,\Sigma,\vec{a}',\vec{a}_{j}} \Imc_{\Bmc_{n}\cup \Amc}$.
\end{itemize}
It follows from Lemma~\ref{lem:m-rewrite} that $(\Tmc,\Sigma,q)$ is not FO-rewritable.

\medskip

(2) If $q_{j}$ is Boolean, then take some $\Sigma$-ABox $\Amc$ with
$\mn{Ind}(\Amc)\cap \mn{Ind}(\Amc_{m})=\emptyset$ for all $m>0$
such that $\Amc,\Tmc\models q'(\vec{a}')$ and $\Amc,\Tmc\not\not\models q_{j}$
for some $\vec{a}'$ in $\mn{Ind}(\Amc)$ of the same length as $\vec{x}'$ (which, since $q_{j}$
is Boolean, coincides with the length of $\vec{x}$). We obtain for all $m>0$:
\begin{itemize}
\item $\Amc_m \cup \Amc,\Tmc\models q(\vec{a}')$ and 
$\Bmc_{m}\cup \Amc\not\models q(\vec{a}')$;
\item $\Imc_{\Amc_{m}\cup \Amc}\equiv_{m,\Sigma,\vec{a}'}\Imc_{\Bmc_{m}\cup \Amc}$.
\end{itemize}
It follows again from Lemma~\ref{lem:m-rewrite} that 
$(\Tmc,\Sigma,q)$ is not FO-rewritable.
\end{proof}


\section{Proofs for Section~\ref{sect:charact}}

\subsection{Preliminary: Role intersections}
We extend the DLs \ELIHFbot and \ELHFbot with intersections of roles that
can occur in existential restrictions on the left hand side of concept
inclusions. This extension enables us to reduce entailment of tree-shaped
CQs to TBox reasoning.

An \emph{$\mathcal{ELI}^{\cap}$ concept} is an $\mathcal{ELI}$ concept that additionally admits 
\emph{role intersections}
$R= r_{1}\cap \cdots \cap r_{n}$ of roles $r_{1},\ldots,r_{n}$ in existential restrictions. 
We denote role intersections by $R,S,R'$ etc.
An \emph{$\mathcal{EL}^{\cap}$ concept} is an 
$\mathcal{EL}$ concept that additionally admits intersections of role names in existential restrictions.
An \emph{$\mathcal{ELIHF}^{\cap\text{-}\mn{lhs}}_{\bot}$ TBox} is an \ELIHFbot TBox in which $\mathcal{ELI}^{\cap}$
concepts can occur on the left hand side of concept inclusions. Similarly, an 
\emph{$\mathcal{ELHF}^{\cap\text{-}\mn{lhs}}_{\bot}$ TBox} is an \ELHFbot TBox in which $\mathcal{EL}^{\cap}$
concepts can occur on the left hand side of concept inclusions. 
The semantics of $\mathcal{ELIHF}^{\cap\text{-}\mn{lhs}}_{\bot}$ TBoxes is defined by extending the semantics
of \ELIHFbot in a straightforward manner, where we assume that $R^{\Imc}= r_{1}^{\Imc} \cap \cdots \cap r_{n}^{\Imc}$
for any interpretation $\Imc$ and role inclusion $R= r_{1}\cap \cdots \cap r_{n}$.

The definition of a normal form for TBoxes and Theorem~\ref{thm:reduce} can be easily extended 
from \ELIHFbot to $\mathcal{ELIHF}^{\cap\text{-}\mn{lhs}}_{\bot}$: say that an $\mathcal{ELIHF}^{\cap\text{-}\mn{lhs}}_{\bot}$ 
TBox $\Tmc$ is in \emph{normal form} if its concept inclusions take the form
\begin{align*} 
 A \sqsubseteq \bot \quad 
 A \sqsubseteq \exists r . B \quad
 \top  \sqsubseteq A \quad B_1 \sqcap B_2 \sqsubseteq A \quad \exists R . B 
 \sqsubseteq A
\end{align*}
with $A,B,B_1,B_2$ concept names, $r$ a role, and $R$ a role intersection. An analogue of Theorem~\ref{thm:reduce}
is formulated and proved in the obvious way. We leave this to the reader.

\subsection{Preliminary: Canonical models}\label{canmod-def}

We introduce the canonical model $\Imc_{\Amc,\Tmc}$ of an ABox $\Amc$ and TBox $\Tmc$ in
$\mathcal{ELIHF}^{\cap\text{-}\mn{lhs}}_{\bot}$. The main properties of $\Imc_{\Tmc,\Amc}$ are:
\begin{itemize}
\item $\Imc_{\Tmc,\Amc}$ is a model of $\Amc$ and $\Tmc$;
\item for every model $\Imc$ of $\Tmc$ there exists a homomorphism
from $\Imc_{\Imc,\Amc}$ to $\Imc$ that maps each $a\in {\sf Ind}(\Amc)$ to
itself.
\end{itemize}
$\Imc_{\Amc,\Tmc}$ is constructed using a standard chase procedure. We will also
introduce a variant of this procedure that constructs, given $\Amc$ and $\Tmc$,
the \emph{completion} ABox $\Amc^{c}_{\Tmc}$ of $\Amc$ which contains $\Amc$ and all assertions
$A(a)$ and $r(a,b)$ with $a,b\in {\sf Ind}(\Amc)$ that are entailed by $\Amc$ and $\Tmc$. 
In both cases we assume that $\Amc$ is consistent with $\Tmc$ and that $\Tmc$ is in normal form.

We start by defining the canonical model $\Imc_{\Amc,\Tmc}$ of $\Amc$ and $\Tmc$.
It is convenient to use ABox notation when constructing $\Imc_{\Amc,\Tmc}$ and so we will
construct a (possibly infinite) ABox $\Amc^{\text{can}}_{\Tmc}$ and define $\Imc_{\Amc,\Tmc}$ as the
interpretation corresponding to $\Amc^{\text{can}}_{\Tmc}$.

Thus assume that $\Amc$ and $\Tmc$ are given. The \emph{full completion sequence of $\Amc$ w.r.t.~$\Tmc$}
is the sequence of ABoxes $\Amc_0,\Amc_1,\dots$ defined by setting 
\begin{eqnarray*}
  \Amc_0 & = & \Amc \cup \\
         &   & \{ r(a,b) \mid s(a,b)\in \Amc, \Tmc\models s\sqsubseteq r\}\cup \\
         &   & \{ r(a,b) \mid s(b,a)\in \Amc, \Tmc\models s^{-} \sqsubseteq r\}
\end{eqnarray*}
and defining $\Amc_{i+1}$ to be $\Amc_i$ extended as follows (recall that we abbreviate $r(a,b)$ by $r^{-}(b,a)$
and that $r$ ranges over roles):
\begin{itemize}

\item[(i)] if $\exists R. B \sqsubseteq A \in \Tmc$ for $R=r_{1}\cap \cdots \cap r_{n}$ and 
$r_{1}(a,b),\ldots,r_{n}(a,b), B(b) \in \Amc_i$, then add $A(a)$ to $\Amc_{i}$;


\item[(ii)] if $\top \sqsubseteq A\in \Tmc$ and $a\in {\sf Ind}(\Amc_{i})$,
then add $A(a)$ to $\Amc_{i}$;

\item[(iii)] if $B_{1}\sqcap B_{2} \sqsubseteq A\in \Tmc$ and $B_{1}(a),B_{2}(a)\in \Amc_{i}$, 
then add $A(a)$ to $\Amc_{i}$;

\item[(iv)] if $A \sqsubseteq \exists r. B \in \Tmc$ and ${\sf func}(r)\in \Tmc$
and $A(a)\in \Amc_{i}$ and there exists $b$ with $r(a,b)\in \Amc_{i}$, 
then add $B(b)$ to $\Amc_{i}$;

\item[(v)] if $A \sqsubseteq \exists r. B \in \Tmc$ and ${\sf func}(r)\not\in \Tmc$
and $A(a)\in \Amc_{i}$, then take a fresh individual $b$ and add $r(a,b)$ and $B(b)$ to $\Amc_{i}$;

\item[(vi)] if $r\sqsubseteq s\in \Tmc$ and $r(a,b)\in \Amc_{i}$, then add $s(a,b)$ to $\Amc_{i}$.
\end{itemize}
Now let $\Amc_{\Tmc}^{\text{can}}=\bigcup_{i\geq 0}\Amc_{i}$ and let $\Imc_{\Amc,\Tmc}$ be the interpretation
corresponding to $\Amc_{\Tmc}^{\text{can}}$. It is straightforward to prove the following properties
of $\Imc_{\Amc,\Tmc}$.
\begin{lemma} \label{lem:canmodelproperties}
Assume $\Amc$ is consistent with $\Tmc$ and $\Tmc$ is in normal form.
Then
\begin{itemize}
\item $\Imc_{\Tmc,\Amc}$ is a model of $\Amc$ and $\Tmc$;
\item for every model $\Imc$ of $\Tmc$ there exists a homomorphism
from $\Imc_{\Imc,\Amc}$ to $\Imc$ that maps each $a\in {\sf Ind}(\Amc)$ to
itself.
\end{itemize}
\end{lemma}
The ABox $\Amc_{\Tmc}^{\text{can}}$ can contain additional individuals
and can even be infinite. For some purposes it is more convenient to work with the subset
$\Amc_{\Tmc}^{c}$ of $\Amc_{\Tmc}^{\text{can}}$ that only contains those assertions in 
$\Amc_{\Tmc}^{\text{can}}$ that use individual names from $\Amc$. 
$\Amc_{\Tmc}^{c}$ can be constructed using rules as well. For any individual name $a$
we set
$$
\Amc|_{a} = \{A(a) \mid A(a)\in \Amc\}
$$
Now consider the rules (i) to (iv) from above and replace the rules (v) and (vi) by the single rule
\begin{itemize}
\item[(vii)] if $\Amc_i|_a,\Tmc\models A(a)$, add $A(a)$ to $\Amc_{i}$.
\end{itemize}
Thus, the \emph{completion sequence of $\Amc$ w.r.t.~$\Tmc$}
is the sequence of ABoxes $\Amc_0,\Amc_1,\dots$, where $\Amc_{0}$ is as defined
above and $\Amc_{i+1}$ is obtained from $\Amc_{i}$ by applying the rules (i) to (iv) and (vii)
to $\Amc_{i}$. The proof of the following is straightforward.
\begin{lemma}
\label{lem:complabox}
For all assertions $A(a)$ and $r(a,b)$ with $a,b\in {\sf Ind}(\Amc)$:
\begin{itemize}
\item $\Amc,\Tmc \models A(a)$ iff $A(a) \in \Amc^c_{\Tmc}$;
\item $\Amc,\Tmc \models r(a,b)$ iff $r(a,b)\in \Amc_{0}$ iff $r(a,b) \in \Amc^c_{\Tmc}$. 
\end{itemize}
\end{lemma}

\subsection{ABox Unraveling and Proof of Proposition~\ref{prop:conttree}}
We show that if a CQ is entailed by an ABox $\Amc$ and TBox $\Tmc$, then it is entailed by an 
unraveling of $\Amc$ into a pseudo tree ABox $\Amc^{\ast}$ and the TBox 
$\Tmc$. 
The corresponding result has been proved for $\mathcal{ELIF}_{\bot}$ TBoxes in 
\citeA{DBLP:conf/kr/BaaderBLW10} and
can be extended to $\mathcal{ELIHF}_{\bot}$ TBoxes in a straightforward manner.
To formulate the result, we need a notion of homomorphisms between ABoxes.
\begin{definition}
  Let \Amc, \Bmc be ABoxes. A mapping $h: \Ind(\Amc) \rightarrow \Ind(\Bmc)$ 
  is a \emph{homomorphism} if
  \begin{itemize}
    \item $A(a) \in \Amc$ implies $A(h(a)) \in \Bmc$ for all $a \in 
    \Ind(\Amc)$;
    \item $r(a,b) \in \Amc$ implies $r(h(a), h(b)) \in \Bmc$ for all $a,b \in 
    \Ind(\Amc)$.
  \end{itemize}
\end{definition}
The following preservation property of homomorphisms w.r.t.~certain answers to CQs is well known.
\begin{lemma}\label{lem:hom-ABox}
  Let $Q=(\Tmc,\Sigma,q)$ be an OMQ from (\ELIHFbot, CQ), \Amc, \Bmc 
  ABoxes, and $h$ a homomorphism from \Amc to \Bmc such that every role that is functional
  in $\Bmc$ is functional in $\Amc$ as well. 
\begin{itemize}
\item If $\Bmc$ is consistent with $\Tmc$, then $\Amc$ is consistent with $\Tmc$;
\item if $\Amc \models Q(\vec{a})$, then $\Bmc \models Q(h(\vec{a}))$ for all $\vec{a} \subseteq 
  \Ind(\Amc)$.
\end{itemize}
\end{lemma}

\medskip

\begin{proposition}\label{forest-witness}
  Let $Q=(\Tmc,\Sigma,q)$ be an OMQ from
  $(\ELIHFbot,\text{CQ})$ and let
  \Amc be a $\Sigma$-ABox that is consistent with $\Tmc$ such that $\Amc\models Q(\vec{a})$. Then
  there is a pseudo tree $\Sigma$-ABox $\Amc^{\ast}$ that is consistent with 
  $\Tmc$, 
  of width at most $|q|$, of outdegree bounded by $|\Tmc|$ and such that $\vec{a}$ is in
  the core of $\Amc^*$ and the following conditions are satisfied:
  \begin{enumerate}
\item $\Amc^{\ast} \models Q(\vec{a})$;
\item there is a homomorphism from $\Amc^{\ast}$ to
  $\Amc$ that is the identity on $\vec{a}$;
\item if a role $r$ is functional in $\Amc$, then $r$ is functional in $\Amc^{\ast}$.
 \end{enumerate}
If $\Tmc$ is an \ELHFbot TBox, then there exists a pseudo ditree ABox $\Amc^{\ast}$ with
these properties.
\end{proposition}

\medskip
\PROPconttree*

\smallskip
\noindent
\begin{proof}
The direction from right to left is trivial. Now assume that $Q_{1}\not\subseteq Q_{2}$.
Then there exists a $\Sigma$-ABox $\Amc$ that is consistent with $\Tmc_{1}$
and $\Tmc_{2}$ and $\vec{a}$ in ${\sf Ind}(\Amc)$ such that
$\Amc \models Q_{1}(\vec{a})$ and $\Amc\not\models Q_{2}(\vec{a})$.
By Proposition~\ref{forest-witness} there exists a pseudo tree $\Sigma$-ABox 
$\Amc^{\ast}$
that is consistent with $\Tmc_{1}$, of width at most $|q_{1}|$ and of outdegree bounded by $|\Tmc_{1}|$ 
such that $\vec{a}$ is in the core of $\Amc^*$ with 
\begin{itemize}
\item $\Amc^{\ast} \models Q_{1}(\vec{a})$;
\item there is a homomorphism from $\Amc^{\ast}$ to $\Amc$ that is the identity on $\vec{a}$;
\item if a role $r$ is functional in $\Amc$, then $r$ is functional in $\Amc^{\ast}$
\end{itemize}
It follows from Lemma~\ref{lem:hom-ABox} that $\Amc^{\ast}$ is consistent with $\Tmc_{2}$
and that $\Amc^{\ast}\not\models Q_{2}(\vec{a})$, as required.
\end{proof}

\subsection{Preliminary: Tree-shaped queries}
We show how Boolean CQs can be rewritten into a set of tree-shaped CQs and then encoded into 
$\mathcal{ELIHF}^{\cap\text{-}\mn{lhs}}_{\bot}$ TBoxes in such a way that their entailment due 
to matches in tree-shaped parts of the canonical model is preserved.

For a CQ $q$ we denote by $\mn{var}(q)$ the set of variables in $q$.
A CQ $q$ is \emph{weakly tree-shaped} if the undirected graph with nodes $\mn{var}(q)$ and
edges $\{\{x,x'\} \mid r(x,x')\in q\}$ is acyclic and connected.
$q$ is called \emph{weakly ditree-shaped} if the directed graph with nodes $\mn{var}(q)$ 
and edges $\{(x,x') \mid (x,x')\in q\}$ is a tree. 

Given a weakly tree-shaped query $q$ we denote
by $C_{q}$ the corresponding $\mathcal{ELI}^{\cap}$ concept (the obvious $\mathcal{ELI}^{\cap}$ concept
for which for any interpretation $\Imc$ and any $d\in \Delta^{\Imc}$ we have $d\in C^{\Imc}$ iff $\Imc\models q(d)$).
If $q$ is a Boolean weakly tree-shaped query, we denote by $C_{q}$
the $\mathcal{ELI}^{\cap}$ concept corresponding to an arbitarily chosen query $q'$ that results from $q$ by 
regarding one of its variables as an answer variable (in what follows it will not matter which variable we choose).  
Note that if $q$ is a weakly ditree-shaped CQ then we can assume that $C_{q}$ is an $\mathcal{EL}^{\cap}$ concept.

Call an interpretation $\Imc$ \emph{weakly tree-shaped} if the undirected graph
with nodes $\Delta^{\Imc}$ and edges $\{ \{d,d'\} \mid (d,d')\in r^{\Imc}\}$ is acyclic
and connected. Call $\Imc$ \emph{weakly ditree-shaped} if the directed graph with nodes $\Delta^{\Imc}$
and edges $\{(d,d') \mid (d,d')\in r^{\Imc}\}$ is a tree. Observe that in the
canonical model $\Imc_{\Amc,\Tmc}$ the interpretation $\Imc_{a}$ attached to the individual
names $a\in {\sf Ind}(\Amc)$ are weakly tree-shaped. Moreover, if $\Tmc$ is an
$\mathcal{ELHF}^{\cap\text{-}\mn{lhs}}_{\bot}$ TBox, then they are weakly ditree-shaped.
It follows that in the canonical model $\Imc_{\Tmc,\Amc}$ of an \ELIHFbot TBox 
$\Tmc$ and pseudo tree
ABox $\Amc$ the only non weakly tree-shaped part is the core of $\Amc$. Moreover, if $\Tmc$
is a \ELHFbot TBox then the only non weakly ditree-shaped part of $\Imc_{\Amc,\Tmc}$ is again
the core of $\Amc$. The following result is straightforward.
\begin{lemma}
\label{lem:queriesasconcepts}  
For any Boolean CQ $q$ there are sets ${\sf tree}(q)$ and ${\sf dtree}(q)$
of $\mathcal{ELI}^{\cap}$-concepts and, respectively, $\mathcal{EL}^{\cap}$-concepts such
that
\begin{enumerate} 
\item $|{\sf tree}(q)| \leq 2^{|q|}$ and for any weakly tree-shaped interpretation $\Imc$, $\Imc\models q$ iff 
there exists $C\in {\sf tree}(q)$ such that $C^{\Imc}\not=\emptyset$;
\item $|{\sf dtree}(q)| \leq 1$ and for any weakly ditree-shaped interpretation $\Imc$, $\Imc\models q$ iff 
there exists $C\in {\sf dtree}(q)$ such that $C^{\Imc}\not=\emptyset$.
\end{enumerate}
\end{lemma}
We use simple $\mathcal{ELIHF}^{\cap\text{-}\mn{lhs}}_{\bot}$ TBoxes to encode 
entailment of weakly tree-shaped queries. For any set $\mathcal{Q}$ of $\mathcal{ELI}^{\cap}$ concepts
denote by $\Tmc_{\mathcal{Q}}$ the $\mathcal{ELIHF}^{\cap\text{-}\mn{lhs}}_{\bot}$ TBox that is obtained 
by computing the normal form of 
$$
\{ C \sqsubseteq A_{C}\mid C\in \mathcal{Q}\},
$$
where the $A_{C}$ are fresh concept names for each $C\in \mathcal{Q}$. 
A \emph{match} of a CQ $q=\exists \vec{x}\varphi(\vec{x},\vec{y})$ in an interpretation
$\Imc$ is a mapping $\pi$ from the variables $\vec{x}\cup \vec{y}$ of $q$ into $\Delta^{\Imc}$
such that $\Imc\models \varphi(\pi(\vec{x},\vec{y}))$.
\begin{lemma}
\label{lem:queriesasconcepts2}  
Let $Q=(\Tmc,\Sigma,q)$, where $\Tmc$ is an \ELIHFbot TBox and $q$ is Boolean CQ.
Let $\Amc$ be a pseudo tree $\Sigma$-ABox and $\Tmc' = \Tmc \cup \Tmc_{{\sf 
tree}(q)}$. 
If $q$ has a match in $\Imc_{\Amc,\Tmc}$ whose range does not intersect 
with the core of $\Amc$, then $\Amc,\Tmc'\models \exists x A_{C}(x)$ for some $C\in {\sf tree}(q)$.

Moreover, if $\Tmc$ is an \ELHFbot TBox and $\Amc$ a pseudo ditree $\Sigma$-ABox, then this still
holds if $\Tmc'$ is replaced by $\Tmc \cup \Tmc_{{\sf dtree}(q)}$ and ${\sf tree}(q)$ by ${\sf dtree}(q)$. 
\end{lemma}

\subsection{Proof of Theorem~\ref{lem:char}}

\noindent
\LEMchar*

\smallskip
\noindent
\begin{proof}
(1) $\Rightarrow$ (2). 
Assume $\varphi$ is an FO rewriting of $Q=(\Tmc,\Sigma,q)$ but (2)
does not hold. Let $\text{qr}(\varphi)$ denote the quantifier rank of $\varphi$.
Consider first the case in which $q$ is Boolean and take $k>2^{\text{qr}(\varphi)}$ and a $\Sigma$-ABox $\Amc$
that is consistent with $\Tmc$ such that 
\begin{itemize}
\item $\Amc, \Tmc \models q$;
\item $\Amc|_{\leq k}, \Tmc \not\models q$;
\item $\Amc|_{>0},\Tmc \not\models q$.
\end{itemize}
Let $\Amc'$ be the disjoint union of $\text{qr}(\varphi)$ many copies of 
$\Amc|_{>0}$ and $\Amc|_{\leq k}$, respectively, and let $\Amc''$ be the disjoint union of 
$\Amc$ and $\Amc'$. We have 
\begin{itemize}
\item $\Amc'',\Tmc\models q$ and $\Amc',\Tmc\not\models q$ (since $q$ is connected).
\end{itemize}
Hence $\Imc_{\Amc''}\models \varphi$ and $\Imc_{\Amc'}\not\models \varphi$.
On the other hand, one can easily prove using Ehrenfeucht-Fra\"iss\'e games 
that
\begin{itemize}
\item $\Imc_{\Amc'}\equiv_{\text{qr}(\varphi),\Sigma,()}\Imc_{\Amc''}$.
\end{itemize}
It follows that $\Imc_{\Amc'}\models \varphi$ and we have derived a contradiction.

\medskip

Now assume that $q$ is not Boolean. Take $k>2^{\text{qr}(\varphi)}$ and a $\Sigma$-ABox $\Amc$
that is consistent with $\Tmc$ and $\vec{a}$ in the core of $\Amc$ such that 
\begin{itemize}
\item $\Amc, \Tmc \models q(\vec{a})$;
\item $\Amc|_{\leq k}, \Tmc \not\models q(\vec{a})$.
\end{itemize}
Let $\Amc_{0}$ be the disjoint union of $\text{qr}(\varphi)$ many copies of 
$\Amc$ and $\Amc|_{\leq k}$, respectively.
Now let $\Amc'$ be the disjoint union $\Amc_{0}$ and $\Amc_{|\leq k}$ and let
$\Amc''$ be the disjoint union of $\Amc_{0}$ and $\Amc$.
We have 
\begin{itemize}
\item $\Amc'',\Tmc\models q(\vec{a})$ and $\Amc',\Tmc\not\models q(\vec{a})$ (since $q$ is connected).
\end{itemize}
Hence $\Imc_{\Amc''}\models \varphi(\vec{a})$ and $\Imc_{\Amc'}\not\models \varphi(\vec{a})$.
On the other hand, one can again easily prove using Ehrenfeucht-Fra\"iss\'e 
games that
\begin{itemize}
\item $\Imc_{\Amc'}\equiv_{\text{qr}(\varphi),\Sigma,\vec{a}}\Imc_{\Amc''}$.
\end{itemize}
It follows that $\Imc_{\Amc'}\models \varphi(\vec{a})$ and we have derived a contradiction.

\medskip

(2) $\Rightarrow$ (1). Let $k_{0}$ be such that (2) holds.
  Consider the set $\Gamma$ of pairs $(\Amc,\vec{c})$ of pseudo tree $\Sigma$-ABoxes \Amc of width at most $|q|$, 
  outdegree
  at most $|\Tmc|$, and depth at most $k_{0}$ that are with to $\Tmc$ and such that 
  $\vec{c}$ is in the core of $\Amc$ and $\Amc,\Tmc\models q(\vec{c})$.

  Now regard each $(\Amc,\vec{c}) \in \Gamma$ as a CQ 
  $q_{\Amc,\vec{c}}(\vec{x})$, where each individual name in \Amc is viewed as a 
  variable, and $\vec{c}$ corresponds to the answer variables $\vec{x}$. 
  We will show that 
  \[ \varphi(\vec{x}) = \bigvee_{\Amc,\vec{c} \in \Gamma} q_{\Amc,\vec{c}}(\vec{x}) \]
  is an FO-rewriting of $Q$.
  
  \medskip\noindent
  Assume $\Amc$ is a $\Sigma$ ABox that is consistent with $\Tmc$ and 
  $\Imc_\Amc \models \varphi(\vec{a})$. Then $\Imc_\Amc \models q_{\Bmc,\vec{c}}(\vec{a})$ for some 
  $(\Bmc,\vec{c}) \in \Gamma$ and so there is a homomorphism $h$ from \Bmc to \Amc 
  mapping $\vec{c}$ to $\vec{a}$. By definition of $\Gamma$, it holds that 
  $\Bmc, \Tmc \models q(\vec{c})$, and therefore $\Amc, \Tmc \models 
  q(\vec{a})$.
   
  \medskip\noindent
  Assume that $\Amc$ is a $\Sigma$ ABox that is consistent with $\Tmc$
  and $\Amc,\Tmc\models q(\vec{a})$. 
  By Proposition~\ref{forest-witness}, there is a pseudo tree $\Sigma$-ABox
  $\Amc^{\ast}$ of width at most $|q|$ and outdegree at most $|\Tmc|$
  that is consistent with $\Tmc$ such that
\begin{itemize}
\item $\Amc^{\ast},\Tmc\models q(\vec{a})$;
\item There is a homomorphism $h$ from $\Amc^{\ast}$ to $\Amc$ that is the identity on $\vec{a}$.
\end{itemize}
Assume first that $q$ is non Boolean. 
By (2) we have $\Amc^{\ast}|_{\leq k_{0}},\Tmc\models q(\vec{a})$. Thus $(\Amc^{\ast}|_{\leq k_{0}},\vec{a})\in \Gamma$.
The homomorphism $h$ (restricted to $\Amc^{\ast}|_{\leq k_{0}}$) shows that $\Imc_{\Amc}\models\varphi(\vec{a})$. 

Now assume that $q$ is Boolean. Take a minimal subset $\Amc'$ of $\Amc^{\ast}$ such that $\Amc',\Tmc\models q$.
$\Amc'$ is a pseudo tree $\Sigma$ ABox with some core $\Amc_{0}'$. By minimality, $\Amc'|_{>0},\Tmc\not\models q$.
Thus, by (2) and minimality we have $\Amc'|_{\leq k_{0}} = \Amc'$. Thus $(\Amc',())\in \Gamma$.
The homomorphism $h$ (restricted to $\Amc'$) shows that $\Imc_{\Amc}\models\varphi$.   

The proof that for $\mathcal{ELFH}_{\bot}$ TBoxes it is sufficient to consider pseudo 
ditree $\Sigma$-ABoxes is similar and uses the fact that in Proposition~\ref{forest-witness}
pseudo tree ABoxes can be replaced by pseudo ditree $\Sigma$-ABoxes.
\end{proof}

\subsection{Proof of Theorem~\ref{lem:witabox}}

\noindent
\LEMwitabox*

\medskip
\noindent
For the pumping argument, we require some preparation. For an ABox $\Amc$
and TBox $\Tmc$ we employ the completion sequence $\Amc_{0},\Amc_{1},\ldots$ of $\Amc$ w.r.t.~$\Tmc$
and the completion $\Amc^{c}_{\Tmc}$ defined in the section on canonical models.
For an ABox $\Amc$ and individual $a$, we set 
$$
\Amc|_{a}   =  \{ A(a) \mid A(a)\in \Amc, A\in \NC \}.
$$
For a set $\Xmc$ of concepts and an individual $u$, we set
$\Xmc(u)= \{C(u) \mid C\in \Xmc\}$.  Let $\Amc$ be a
pseudo tree $\Sigma$-ABox and $u \in \Ind(\Amc_j)$ for some tree of $\Amc$.
Define
$$
{\sf AT}_{\Amc}^{\vdash}(u) := \{A \in \NC \mid A(u) \in 
\Amc^{c}_{\Tmc}\} 
$$
Let $\Amc_{u}^{\downarrow}$ denote the subtree of $\Amc_{j}$ rooted at $u$, and 
let $\Amc_{u}^{\uparrow}$ be the ABox obtained from $\Amc$ by dropping 
$\Amc_{u}^{\downarrow}$ from $\Amc$ except for $u$ itself.
Define the \emph{transfer sequence $\Xmc_{0},\Xmc_{1}, \ldots$ 
of $(\Amc,u)$ \wrt $\Tmc$} 
by induction as follows: 
\begin{itemize}
  \item $\Xmc_{0}= {\sf AT}_{\Amc^{0}}^{\vdash}(u)$, where $\Amc^{0}= \Amc_{u}^{\uparrow}$;
  \item $\Xmc_{1} = {\sf AT}_{\Amc^{1}}^{\vdash}(u)$, where $\Amc^{1} = 
  \Amc_{u}^{\downarrow} \cup 
  \Xmc_{0}(u)$;
  \item $\Xmc_{2i+2}= {\sf AT}_{\Amc^{2i+2}}^{\vdash}(u)$, where 
  $\Amc^{2i+2}= \Amc^{2i} \cup 
  \Xmc_{2i+1}(u)$, for $i\geq 0$;
  \item $\Xmc_{2i+1}= {\sf AT}_{\Amc^{2i+1}}^{\vdash}(u)$, where 
  $\Amc^{2i+1}= \Amc^{2i-1} \cup 
  \Xmc_{2i}(u)$, for $i\geq 1$.
\end{itemize}
The sequence of ABoxes $\Amc^{0},\Amc^{1}\ldots$ defined above is called the 
\emph{ABox transfer sequence for $(\Amc,u)$ \wrt $\Tmc$}. 

\begin{lemma}\label{lem:transfer1}
  Let $n= |{\sf sig}(\Tmc)|+1$. Then $\Xmc_{n}= \Xmc_{m}$ 
  for all $m>n$ and $(\Amc^{n-1})^c_\Tmc \cup (\Amc^{n})^c_\Tmc = \Amc_{\Tmc}^{c}$. 
\end{lemma}

\begin{proof}
  By definition, $\Xmc_{m}\subseteq \Xmc_{m+1}$, for all $m>0$. 
  Moreover, if $\Xmc_{m+1}= \Xmc_{m}$ for some $m>0$ then, by 
  Lemma~\ref{lem:complabox}, 
  \begin{itemize}
  \item all $\Amc^{(m+1)+2i}$, $i\geq 0$, coincide;
  \item all $\Amc^{(m+2)+2i}$, $i\geq 0$, coincide.
  \end{itemize}
  It follows that $\Xmc_{m'}=\Xmc_{m}$ for all $m'>m$.
\end{proof}

We say that $(\Amc,a)$ and $(\Bmc,b)$ \emph{coincide locally w.r.t.~$\Tmc$}
(in symbols $(\Amc,a)\sim_{\Tmc}(\Bmc,b)$): 
\begin{itemize}
\item $\{A \in \mn{sig}(\Tmc) \mid A(a) \in \Amc\} = \{ B \in \mn{sig}(\Tmc) \mid B(b)\in \Bmc\}$;
\item for every role $r$ with ${\sf func}(r)\in \Tmc$: there exists $a'$ with
$r(a,a')\in \Amc_{a}^{\uparrow}$ iff there exists $b'$ with $r(b,b')\in \Bmc_{b}^{\uparrow}$;
\item for every role $r$ with ${\sf func}(r)\in \Tmc$: there exists $a'$ with
$r(a,a')\in \Amc_{a}^{\downarrow}$ iff there exists $b'$ with $r(b,b')\in \Bmc_{b}^{\downarrow}$;
\end{itemize} 

\begin{lemma}\label{lem:transfer}
Let $\Amc$ and $\Bmc$ be pseudo tree $\Sigma$ ABoxes with $a\in {\sf 
Ind}(\mn{trees}(\Amc))$ and $b\in {\sf Ind}(\mn{trees}(\Bmc))$ such that 
\begin{itemize}
  \item $(\Amc,a)$ and $(\Bmc,b)$ coincide locally w.r.t.~$\Tmc$;
  \item the transfer sequence of $(\Amc,a)$ w.r.t.~$\Tmc$ coincides with the
  transfer sequence of $(\Bmc,b)$ w.r.t.~$\Tmc$ and is given by 
  $\mathcal{X}_{0},\ldots$. 
\end{itemize}
Denote by $\Cmc$ the ABox obtained from $\Amc$ by replacing the subtree 
$\Amc_{a}^{\downarrow}$ by $\Bmc_{b}^{\downarrow}$.
Then
\begin{itemize}
\item $\mathcal{X}_{0},\ldots$ is also the transfer sequence of $(\Cmc,b)$ 
w.r.t.~$\Tmc$.
\item Given the ABox transfer sequences $\Amc^{0},\ldots$ and 
$\Bmc^{0},\ldots$ of $(\Amc,a)$ and
$(\Bmc,b)$ \wrt \Tmc\!\!\!, respectively, the ABox transfer sequence 
$\Cmc^{0},\ldots$ of $(\Cmc,b)$ 
\wrt \Tmc is given by setting $\Cmc^{2i}= \Amc^{2i}$ and 
$\Cmc^{2i+1}=\Bmc^{2i+1}$, for $i\geq 0$.
\end{itemize}
\end{lemma}  

\begin{proof}
Straightforward using Lemma~\ref{lem:complabox}.
\end{proof}

\medskip\noindent
Let $Q=(\Tmc,\Sigma,q)$ be an OMQ from
$(\mathcal{ELIHF}_{\bot},\text{conCQ})$ and $k \geq 0$. 
A pair $\mathcal{A}$, $\vec{a}$ with $\Amc$ a pseudo tree $\Sigma$-ABox and
$\vec{a}$ a tuple in the core of $\Amc$ is a \emph{$k$-entailment witness} for $Q$ if
\begin{enumerate}
\item $\Amc$ is consistent with $\Tmc$;
\item $\Amc,\Tmc \models q(\vec{a})$;
\item and  
\begin{itemize}
\item $q$ is not Boolean and $\Amc|_{\leq k},\Tmc\not\models q(\vec{a})$ or
\item $q$ is Boolean, $\Amc|_{\leq k},\Tmc\not\models q$ and $\Amc|_{>0},\Tmc \not\models q$.
\end{itemize}
\end{enumerate}
If $q$ is Boolean then we say that $\mathcal{A}$ is a $k$-entailment witness for $Q$
if $\Amc$, $()$ is a $k$-entailment witness for $Q$.

The following Lemma implies Part~1 of 
Theorem~\ref{lem:witabox} for queries that are not Boolean.

\begin{lemma}\label{lem:part1}
Let $Q=(\Tmc,\Sigma,q)$ be an OMQ from
$(\mathcal{ELIHF}_{\bot},\text{conCQ})$. 
If $q$ is not Boolean,
then $Q$ is not FO-rewritable iff there exists a $k_{0}$-entailment witness for $Q$
of outdegree bounded by $|\Tmc|$ for $k_{0}= |q|+2^{3 m^2}$ where $m=|\Tmc|$.
\end{lemma}

\begin{proof}
The direction $(\Rightarrow)$ follows from Theorem~\ref{lem:char}.
Conversely, assume that there is a $k_{0}$-entailment witness $\Amc$, $\vec{a}$
for $Q=(\Tmc,\Sigma, q)$. We show that for every $k>k_{0}$ there exists a 
$k$-entailment witness for $Q$.
Then non FO-rewritability of $Q$ follows from Theorem~\ref{lem:char}.

Assume $\Amc$, $\vec{a}$ is a $k$-entailment witness for $Q$ for some $k\geq k_{0}$. 
It is sufficient to construct a pseudo tree $\Sigma$-ABox $\Bmc$ which, together with $\vec{a}$
is a $k'$-entailment witness for $Q$ for some $k'>k$. We may assume \Wlog that $\Amc$ is minimal in
the sense that, for every individual $a$ from the trees of $\Amc$ we 
have $\Amc^{-a},\Tmc\not\models q(\vec{a})$, where $\Amc^{-a}$ is 
obtained from $\Amc$ by dropping the subtree rooted at $a$ (including $a$).

Let $w$ be a leaf node in \Amc of maximal distance from the core of $\Amc$ 
and $\rho$ be the root of the tree $\Amc_{i}$ of $\Amc$ containing $w$.
Then the distance of $w$ from $\rho$ is at least $k+1$. Since by Lemma~\ref{lem:transfer1} the number of transfer
sequences \wrt \Tmc does not exceed $2^{|\Tmc|^2}\!\!$, on
the path from $\rho$ to $w$ there must be at least two individuals
$u_1$ and $u_2$ with distance at least $|q|$ from $\rho$ such that
\begin{itemize}
\item[(a)] $(\Amc, u_{1})$ and $(\Amc,u_{2})$ coincide locally w.r.t.~$\Tmc$;
\item[(b)] the transfer sequences of $(\Amc,u_{1})$ and $(\Amc,u_{2})$ 
w.r.t.~$\Tmc$ coincide;
\item[(c)] the transfer sequences of $(\Amc^{-w}|,u_{1})$ and $(\Amc^{-w}|,u_{2})$ 
w.r.t.~$\Tmc$ coincide.
\end{itemize}
We may assume that $u_{1}$ is between $\rho$ and $u_{2}$.
Let $\Bmc$ be the ABox obtained from $\Amc$ by replacing 
$\Amc_{u_{2}}^{\downarrow}$ by $\Amc_{u_{1}}^{\downarrow}$ in $\Amc$. By 
renaming nodes in $\Amc_{u_{1}}^{\downarrow}$, we can assume that the root of 
the subtree $\Amc_{u_{1}}^{\downarrow}$ of $\Bmc$ is denoted by $u_{2}$.

We show that $\Bmc$, $\vec{a}$ is a $k+1$-entailment witness for $\Tmc$, $\Sigma$, and $q$.
To this end we show:
\begin{enumerate}
\item $\Bmc$ is consistent relative to $\Tmc$;
\item $\Bmc,\Tmc\models q(\vec{a})$;
\item $\Bmc|_{\leq k+1},\Tmc\not\models q(\vec{a})$.
\end{enumerate}
First observe that by (a) $(\Amc,u_{1})$ and $(\Amc,u_{2})$ coincide locally w.r.t.~$\Tmc$. Thus,
since $\Amc$ is consistent relative to $\Tmc$, $\Imc_{\Bmc}$ satisfies the functionality constraints of $\Tmc$.
Also, it follows from (b) and Lemma~\ref{lem:complabox} and Lemma~\ref{lem:transfer}
that we obtain a canonical model $\Imc_{\Bmc,\Tmc}$ of $\Bmc$ and $\Tmc$ by 
replacing the subtree-interpretation rooted at $u_{2}$ in $\Imc_{\Amc,\Tmc}$ with the
subtree-interpretation rooted at $u_{1}$ in $\Imc_{\Amc,\Tmc}$. Thus, $\Bmc$ is consistent relative to $\Tmc$.  

It follows from $\Amc,\Tmc\models q(\vec{a})$, the condition that $q$ is connected,
and the fact that $\Amc$ coincides with $\Bmc$ for individuals with distance 
$\leq |q|$ from the
core of $\Amc$ that $\Bmc,\Tmc\models q(\vec{a})$.

It follows from (c), Lemma~\ref{lem:complabox}, and Lemma~\ref{lem:transfer}
that we obtain a canonical model $\Imc_{\Bmc^{-w},\Tmc}$ of $\Bmc^{-w}$ and $\Tmc$ by 
replacing the interpretation $\Imc_{u_{2}}$ rooted at $u_{2}$ in $\Imc_{\Amc^{-w},\Tmc}$ with the
interpretation $\Imc_{u_{1}}$ rooted at $u_{1}$ in $\Imc_{\Amc^{-w},\Tmc}$.

Now recall that $\Amc^{-w},\Tmc\not\models q(\vec{a})$. It follows from the condition that $q$ is connected
and the fact that $\Amc^{-w}$ coincides with $\Bmc^{-w}$ 
for individuals with distance $\leq |q|$ from the core of $\Amc$ that $\Bmc^{-w},\Tmc\not\models q(\vec{a})$.

Clearly $\Bmc^{-w}\supseteq \Bmc|_{\leq k+1}$. Thus, $\Bmc|_{\leq k+1},\Tmc\not\models q(\vec{a})$, as required.
\end{proof}

We now consider the case in which $q$ is Boolean. Let $\Amc$ be a 
pseudo tree $\Sigma$-ABox. In contrast to the
non Boolean case, $q$ can have matches in the canonical model $\Imc_{\Amc,\Tmc}$ that
do not hit the core of $\Amc$. Thus, to ensure that after
pumping $\Amc$, no additional matches of $q$ are introduced we have
to consider transfer sequences that are invariant under possible
matches of $q$. 

Given a CQ $q$, we thus consider the additional TBox $\Tmc_{{\sf tree}(q)}$ defined 
above and consider transfer sequence relative to $\Tmc\cup \Tmc_{{\sf tree}(q)}$
rather than $\Tmc$ only. Observe that this has the desired effect as the canonical
model $\Imc_{\Amc,\Tmc}$ is weakly tree-shaped except for the individuals in its core.

The following Lemma implies Part~2 of Theorem~\ref{lem:witabox}.

\begin{lemma}\label{lem:part1a}
Let $Q=(\Tmc,\Sigma,q)$ be an OMQ from
$(\mathcal{ELIHF}_{\bot},\text{conCQ})$ 
If $q$ is Boolean, then $Q$ is not FO-rewritable iff there exists a $k_{0}$-entailment witness for $Q$
of outdegree bounded by $|\Tmc|$ for $k_{0}= |q|+2^{4 m^2}$ where $m=|\Tmc|+2^{|q|}$.
\end{lemma}
\begin{proof}
We modify the proof of Lemma~\ref{lem:part1}. The direction
$(\Rightarrow)$ follows again from Theorem~\ref{lem:char}.

Conversely, assume that there is a $k_{0}$-entailment witness for $Q$.
We show that for every $k>k_{0}$ there exists a 
$k$-entailment witness for $Q$.

Assume $\Amc$ is a $k$-entailment witness for $Q$ for some $k\geq k_{0}$. 
It is sufficient to construct a pseudo tree $\Sigma$-ABox $\Bmc$ that is consistent
relative to $\Tmc$ and is a $k'$-entailment witness for $Q$
for some $k'>k$. We may assume \Wlog that $\Amc$ is minimal in
the sense that, for every individual $a$ in any tree of $\Amc$ 
we have $\Amc^{-a},\Tmc\not\models q$.

Let $w$ be a leaf node in \Amc of maximal distance from the core of $\Amc$ 
and $\rho$ be the root of its tree. Then the distance of $w$ from $\rho$  is 
at least $k+1$. Since by Lemma~\ref{lem:transfer1} the number of transfer
sequences w.r.t.~$\Tmc\cup \Tmc_{{\sf tree}(q)}$ does not exceed $2^{(|\Tmc|+|\Tmc_{{\sf tree}(q)}|)^{2}}\!\!$, on
the path from $\rho$ to $w$ there must be at least two individuals
$u_1$ and $u_2$ with distance at least $|q|$ from $\rho$ such that
\begin{itemize}
\item[(a)] $(\Amc, u_{1})$ and $(\Amc,u_{2})$ coincide locally;
\item[(b)] the transfer sequences of $(\Amc,u_{1})$ and $(\Amc,u_{2})$ 
w.r.t.~$\Tmc\cup \Tmc_{{\sf tree}(q)}$ coincide;
\item[(c)] the transfer sequences of $(\Amc^{-w},u_{1})$ and $(\Amc^{-w},u_{2})$ 
w.r.t.~$\Tmc\cup \Tmc_{{\sf tree}(q)}$ coincide.
\item[(d)] the transfer sequences of $(\Amc|_{>0},u_{1})$ and $(\Amc|_{>0},u_{2})$ 
w.r.t.~$\Tmc\cup \Tmc_{{\sf tree}(q)}$ coincide.
\end{itemize}
We may assume that $u_{1}$ is between $\rho$ and $u_{2}$.
Let $\Bmc$ be the ABox obtained from $\Amc$ by replacing 
$\Amc_{u_{2}}^{\downarrow}$ by $\Amc_{u_{1}}^{\downarrow}$ in $\Amc$. By 
renaming nodes in $\Amc_{u_{1}}^{\downarrow}$, we can assume that the root of 
the subtree $\Amc_{u_{1}}^{\downarrow}$ of $\Bmc$ is denoted by $u_{2}$.

We show that
\begin{itemize}
\item $\Bmc$ is consistent with $\Tmc$;
\item $\Bmc,\Tmc\models q$;
\item $\Bmc|_{> 0},\Tmc\not\models q$;
\item $\Bmc|_{\leq k+1},\Tmc\not\models q$.
\end{itemize}

The argument for (a) is the same as in the proof of Lemma~\ref{lem:part1} and omitted.

In what follows we apply the observation that for any $\Sigma$-ABox $\Amc$ one obtains a canonical 
model of $\Amc$ and $\Tmc$ from a 
canonical model of $\Amc$ and $\Tmc\cup \Tmc_{{\sf tree}(q)}$ by taking the reduct to 
the symbols in $\Sigma\cup\Tmc$. In fact, every canonical model of $\Amc$ and $\Tmc$ can be obtained in this way.

It follows from Lemma~\ref{lem:complabox} and Lemma~\ref{lem:transfer}
and conditions (b), (c), and (d), respectively, that we obtain a canonical model 
\begin{itemize}
\item[(b')] $\Imc_{\Bmc,\Tmc\cup \Tmc^{q}}$ of $\Bmc$ and $\Tmc\cup \Tmc_{{\sf tree}(q)}$ by 
replacing the subtree rooted at $u_{2}$ in $\Imc_{\Amc,\Tmc\cup \Tmc_{{\sf tree}(q)}}$ with the
subtree rooted at $u_{1}$ in $\Imc_{\Amc,\Tmc\cup \Tmc_{{\sf tree}(q)}}$;  
\item[(c')] $\Imc_{\Bmc^{-w},\Tmc\cup \Tmc^{q}}$ of $\Bmc^{-w}$ and $\Tmc\cup \Tmc_{{\sf tree}(q)}$ by 
replacing the subtree rooted at $u_{2}$ in $\Imc_{\Amc^{-w},\Tmc\cup \Tmc_{{\sf tree}(q)}}$ with the
subtree rooted at $u_{1}$ in $\Imc_{\Amc^{-w},\Tmc\cup \Tmc_{{\sf tree}(q)}}$;
\item[(d')] $\Imc_{\Bmc|_{>0},\Tmc\cup \Tmc_{{\sf tree}(q)}}$ of $\Bmc|_{>0}$ and $\Tmc\cup \Tmc_{{\sf tree}(q)}$ by 
replacing the subtree rooted at $u_{2}$ in $\Imc_{\Amc|_{>0},\Tmc\cup \Tmc_{{\sf tree}(q)}}$ with the
subtree rooted at $u_{1}$ in $\Imc_{\Amc|_{>0},\Tmc\cup \Tmc_{{\sf tree}(q)}}$.
\end{itemize}
To show that $\Bmc,\Tmc\models q$, we distinguish two cases: if $q$ has a match in $\Imc_{\Amc,\Tmc}$ that intersects
with the core of $\Amc$, then $q$ has such a match as well in $\Imc_{\Bmc,\Tmc}$ since 
$\Amc$ coincides with $\Bmc$ for individuals with distance $\leq |q|$ from the core of $\Amc$.
If $q$ does not have such match, then $\Amc,\Tmc\cup \Tmc_{{\sf tree}(q)}\models \exists x\;A_{C}(x)$ for 
some $C\in {\sf trees}(q)$.
But then $A_{C}^{\Imc_{\Amc,\Tmc\cup \Tmc_{{\sf tree}(q)}}}\not=\emptyset$ and so 
$A_{C}^{\Imc_{\Bmc,\Tmc\cup \Tmc_{{\sf tree}(q)}}}\not=\emptyset$.
Thus $\Bmc,\Tmc\models q$.

$\Bmc|_{>0},\Tmc\not\models q$ follows from (d') and the fact that $\Amc|_{>0},\Tmc\not\models q$.

To show $\Bmc|_{\leq k+1},\Tmc\not\models q$ it is sufficient to show that 
$\Bmc^{-w},\Tmc\not\models q$. But this follows from (c') and the fact that
$\Amc^{-w},\Tmc\not\models q$.
\end{proof}

It remains to prove the claim of Theorem~\ref{lem:witabox} for
$\mathcal{ELHF}_{\bot}$ TBoxes and Boolean queries. In this case, since one can work
with pseudo ditree $\Sigma$-ABoxes rather than arbitrary pseudo tree $\Sigma$-ABoxes 
one can employ the linear size TBox $\Tmc_{{\sf dtree}(q)}$ rather than
the possibly exponential size TBox $\Tmc_{{\sf tree}(q)}$. The remaining part
of the proof is exactly the same as before and so one obtains:

\begin{lemma}\label{lem:part3}
Let $Q=(\Tmc,\Sigma,q)$ be an OMQ from
$(\mathcal{ELHF}_{\bot},\text{conCQ})$. 
If $q$ is Boolean, then $Q$ is not FO-rewritable iff there exists a $k_{0}$-entailment witness for $Q$
of outdegree bounded by $|\Tmc|$ for $k_{0}= |q|+2^{4 m^{2}}$ where $m=|\Tmc|+|q|$.
\end{lemma}

\section{Proofs for Section~\ref{sect:automata}}

\subsection{Preliminary: Tree Automata}

\medskip
We introduce two-way alternating parity automata on finite trees
(TWAPAs). 
%
%
We assume w.l.o.g.\ that all
 nodes in an $m$-ary tree are from $\{1,\dots,m\}^*$.
%
For any set $X$,
let $\Bmc^+(X)$ denote the set of all positive Boolean formulas over
$X$, i.e., formulas built using conjunction and disjunction over the
elements of $X$ used as propositional variables, and where the special
formulas $\mn{true}$ and $\mn{false}$ are allowed as well. 
An \emph{infinite path} $P$ of a tree $T$ is a
prefix-closed set $P \subseteq T$ such that for every $i \geq 0$,
there is a unique $x \in P$ with $|x|=i$.

\begin{definition}[TWAPA]
  A \emph{two-way alternating parity automaton 
    (TWAPA) on finite $m$-ary trees} is a tuple
  $\Amf=(S,\Gamma,\delta,s_0,c)$ where $S$ is a finite set of
  \emph{states}, $\Gamma$ is a finite alphabet, $\delta: S \times
  \Gamma \rightarrow \Bmc^+(\mn{tran}(\Amf))$ is the \emph{transition
    function} with $\mn{tran}(\Amf) = \{ \dia{i} s, \ [i] s
  \mid -1 \leq
  i \leq m \text{ and } s \in S \}$ the set of
  \emph{transitions} of \Amf, $s_0 \in S$ is the \emph{initial state},
  and $c:S \rightarrow \Nbbm$ is the \emph{parity condition} that 
  assigns to each state a \emph{priority}.
\end{definition}
Intuitively, a transition $\dia{i} s$ with $i>0$ means that
a copy of the automaton in state $s$ is sent to the $i$-th successor
of the current node, which is then required to exist. Similarly,
$\dia0 s$ means that the automaton stays at the current
node and switches to state $s$, and $\dia{-1} s$ indicates
moving to the predecessor of the current node, which is then required
to exist. Transitions $[i] s$ mean that a copy of the automaton in
state $s$ is sent to the relevant successor if that successor exists
(which is not required).\footnote{In our automata constructions, we
  will explicitly use only transitions of the form $\dia{i} s$. The dual 
  transitions are needed for closure of TWAPAs under
  complement.}
\begin{definition}[Run, Acceptance]
  A \emph{run} of a TWAPA $\Amf = (S,\Gamma,\delta,s_0,c)$ on a finite
  $\Gamma$-labeled tree $(T,L)$ is a $T \times S$-labeled tree
  $(T_r,r)$ such that the following conditions are satisfied:
  \begin{enumerate}

  \item $r(\varepsilon) = ( \varepsilon, s_0)$;
    

  \item if $y \in T_r$, $r(y)=(x,s)$, and $\delta(s,L(x))=\vp$, then
    there is a (possibly empty) set $S \subseteq \mn{tran}(\Amf)$ such
    that $S$ (viewed as a propositional valuation) satisfies $\vp$ as
    well as the following conditions:
    \begin{enumerate}

    \item if $\dia{i} s' \in S$, then $x \cdot i$ is defined and 
      there is a node $y \cdot j \in T_r$ such that $r(y \cdot j)=(x 
      \cdot i,s')$;

    \item if $[i]s' \in S$ and $x \cdot i$ is defined and in $T$, then
      there is a
      node $y \cdot j \in T_r$ such that $r(y \cdot j)=(x \cdot
      i,s')$.

    \end{enumerate}

  \end{enumerate}
  We say that $(T_r,r)$ is \emph{accepting} if on all infinite paths
  $\varepsilon = y_1 y_2 \cdots$ of $T_r$, the maximum priority that
  appears infinitely often is even.  A finite $\Gamma$-labeled tree
  $(T,L)$ is \emph{accepted} by \Amf if there is an accepting run of
  \Amf on $(T,L)$. We use $L(\Amf)$ to denote the set of all finite
  $\Gamma$-labeled tree accepted by \Amf.
\end{definition}
It is known (and easy to see) that TWAPAs are closed under
complementation and intersection, and that these constructions
involve only a polynomial blowup.
%
It is also known
that 
their emptiness problem can be solved in time single exponential in
the number of states and polynomial in all other components of the
automaton. 
In what follows, we shall generally only explicitly
analyze the number of states of a TWAPA, but only implicitly take care
that all other components are of the allowed size for the complexity
result that we aim to obtain.
%
%
%

\LEMbooleancontainment*
\begin{proof}
(1.) Consider an OMQ $Q=(\Tmc,\Sigma,q(\vec{x}))$, where $\vec{x}=x_{1},\ldots,x_{n}$.
Take fresh concept names $A_{1},\ldots,A_{n}$ and let $\Sigma'=\Sigma \cup \{A_{1},\ldots,A_{n}\}$.
Denote by $q'(\vec{x})$ the result of adding the conjuncts 
$A_{j}(x_{j})$ to $q(\vec{x})$ for $j\in \{1,\ldots,n\}$.
One can show that $Q'=(\Tmc,\Sigma',\exists \vec{x}q'(\vec{x}))$ is FO-rewritable iff
$Q$ is FO-rewritable. In fact, if $\varphi(\vec{x})$ is an FO-rewriting of $Q$, 
then obtain $\psi(\vec{x})$ from $\varphi(\vec{x})$ by adding the conjuncts $A_{j}(x)$ for 
$j\in \{1,\ldots,n\}$ to $\varphi(\vec{x})$. Then $\exists \vec{x}\psi(\vec{x})$ is an FO-rewriting
of $Q'$. 

(2.) Let $Q_{i}=(\Tmc_{i},\Sigma,q_{i}(\vec{x}))$ be OMQs for $i=1,2$.
We form the Boolean OMQs $Q_{i}'$ in the same way as above.
Then $Q_{1}\subseteq Q_{2}$ iff $Q_{1}'\subseteq Q_{2}'$, as required.
\end{proof}   

\medskip 

\subsection{Preliminary: Forest Decompositions}
\label{app:forestdecomp}

Before proving Proposition~\ref{prop:autocentral}, we carefully
analyse query matches in canonical models of pseudo tree ABoxes.  We
start with the case where the TBox is formulated in
$\mathcal{ELIHF}_\bot$ and afterwards consider $\mathcal{ELHF}_\bot$.

Let $q$ be a connected CQ. We use $\mn{id}(q)$ to denote the set of
all queries that can be obtained from $q$ by identifying variables. A
\emph{forest decomposition} of $q$ is a tuple
$F=(q_{\mn{core}}, q_1,x_1, \ldots, q_k,x_k, \mu)$ where
$(q_{\mn{core}}, q_1, \ldots, q_k)$ is a partition of (the atoms of) a
query from $\mn{id}(q)$, $x_1,\dots,x_k$ are variables from
$q_{\mn{core}}$, and $\mu$ is a mapping from $\Var(q_{\mn{core}})$ to
$\Ind_{\mn{core}}$ such that the following conditions are satisfied
for $1 \leq i,j \leq k$;
\begin{enumerate}
  \item $q_{\mn{core}}$ is non-empty;

  \item $q_i$ is weakly tree-shaped with root $x_i$;

  \item $\mn{Var}(q_i) \cap \mn{Var}(q_{\mn{core}})=\{ x_i \}$;

  \item $\Var(q_i) \cap \Var(q_j) \subseteq 
  \Var(q_{\mn{core}})$ if $i \neq j$;

  \item $q_i$ contains no atom $A(x_i)$;

  \item $x_i$ has a single successor in $q_i$.

\end{enumerate}
With $\mn{fdec}(q)$, we denote the set of all forest decompositions of
$q$. The following lemma shows how certain matches of $q$ in 
the
canonical models of pseudo tree ABoxes give rise to forest
decompositions of $q$.
\begin{lemma}\label{forest-decomp-lemma}
  Let \Tmc be an $\mathcal{ELIHF}_\bot$ TBox, \Amc a pseudo tree ABox,
  and $q$ a Boolean connected CQ.  Then the following are equivalent:
  \begin{enumerate}
    
  \item there is a match $\pi$ of $q$ in $\Imc_{\Tmc,\Amc}$ such that
    at least one individual from the core of \Amc is in the range of $\pi$;

  \item there is a forest decomposition $F=(q_{\mn{core}},
    q_1,x_1, \ldots,$ $q_k,x_k,\mu)$ of $q$ such that
  \begin{itemize}
    \item $\mu$ is a match for $q_{\mn{core}}$ in $\Imc_{\Tmc,\Amc}$ whose
      range consists solely of core individuals from \Amc;
    \item for $1 \leq i \leq k$, there is a match $\pi_i$ for $q_i$ in 
    $\Imc_{\Tmc,\Amc}$ such that $\pi_i(x_i)=\mu(x_i)$. 
  \end{itemize}
  \end{enumerate}
\end{lemma}
%
Lemma~\ref{forest-decomp-lemma} is a minor variation of
similar lemmas proved e.g.\ in \cite{DBLP:conf/cade/Lutz08}; proof details are
omitted.

It can be verified that there is a polynomial $p$ such that for every
connected CQ $q$, the number of forest decompositions of $q$ is
bounded by $2^{p(|q|)}$.

\medskip

Now for the case of $\mathcal{ELHF}_\bot$. We say that a CQ $q'$ is
obtained from a CQ $q$ \emph{by fork elimination} if $q'$ is obtained
from $q$ by selecting two atoms $r(x,z), s(y,z)$ and identifying the
variables $x$ and $y$. We call $q'$ a \emph{fork rewriting} of $q$ if
$q'$ can be obtained from $q$ by repeated (but not necessarily
exhaustive) fork elimination. We say that $q'$ is a \emph{maximal 
fork rewriting} of $q$ if it is a fork rewriting and no further fork
elimination is possible. A \emph{directed forest decomposition}
is defined like a forest decomposition except that
\begin{enumerate}

\item $(q_{\mn{core}}, q_1, \ldots, q_k)$ is a partition of (the atoms
  of) a fork rewriting of $q$, instead of a query from $\mn{id}(q)$;

\item the queries $q_1,\dots,q_k$ are required to be weakly
  ditree-shaped.

\end{enumerate}
%
We now
establish a strengthened version of Lemma~\ref{forest-decomp-lemma}
for $\mathcal{ELHF}_\bot$ TBoxes.
\begin{lemma}
\label{EL-forest-decomp-lemma}
When $\Tmc$ is an $\mathcal{ELHF}_\bot$ TBox, then the equivalence in
Lemma~\ref{forest-decomp-lemma} is true when forest decompositions are
replaced with 
directed forest decompositions.
\end{lemma}

\subsection{Preliminary: Derivation Trees}

We characterize entailment of AQs in terms of derivation trees. Fix an
$\mathcal{ELIHF}^{\cap\text{-}\mn{lhs}}_{\bot}$ TBox \Tmc in normal
  form and an ABox \Amc. 
A
\emph{derivation tree} for an assertion $A_0(a_0)$ in \Amc with $A_0
\in \NC \cup \{ \bot \}$ is a finite $\mn{Ind}(\Amc) \times (\NC \cup
\{\bot \})$-labeled tree $(T,V)$ that satisfies the following conditions:
\begin{enumerate}

\item $V(\varepsilon)=(a_0,A_0)$;




\item if $V(x)=(a,A)$ and neither $A(a) \notin \Amc$ nor $\top
  \sqsubseteq A \in \Tmc$, then one of the following holds:
  \begin{itemize}

  \item $x$ has successors $y_1,\dots,y_k$, $k \geq 1$ with
    $V(y_i)=(a,A_i)$ for $1 \leq i \leq k$ and $\Tmc \models A_1
    \sqcap \cdots \sqcap A_k \sqsubseteq A$;

  \item $x$ has a single successor $y$ with $V(y)=(b,B)$ and there
    is an $\exists R . B \sqsubseteq A \in \Tmc$ and an $R'(a,b) \in
    \Amc$ such that
    $\Tmc \models R' \sqsubseteq R$;

  \item $x$ has a single successor $y$ with $V(y)=(b,B)$ and there
    is a $B \sqsubseteq \exists r . A \in \Tmc$ such that $r(b,a) \in \Amc$ and
    $\mn{func}(r) \in \Tmc$. 

  \end{itemize}

\end{enumerate}
%
Note that the first item of Point~2 above requires $\Tmc \models
A_1\sqcap \dotsb\sqcap A_n \sqsubseteq A$ instead of $A_1 \sqcap A_2
\sqsubseteq A \in \Tmc$ to `shortcut' anonymous parts of the canonical
model. In fact, the derivation of $A$ from $A_1 \sqcap \dots \sqcap
A_n$ by $\Tmc$ can involve the introduction of anonymous elements.

We call a TBox \Tmc \emph{satisfiable} if it has a model.
The main property of derivation trees is the following.
\begin{lemma}
\label{lem:derivationtrees}
~\\[-4mm]
  \begin{enumerate}

  \item $\Amc,\Tmc \models A(a)$ iff \Amc is inconsistent with \Tmc or there is
    a derivation tree for $A(a)$ in \Amc, for all assertions $A(a)$
    with $A \in \NC$ and $a \in \mn{Ind}(\Amc)$;

  \item \Amc is inconsistent with \Tmc iff \Tmc is unsatisfiable,
    there is a derivation tree for $\bot(a)$ in \Amc for some $a \in 
    \mn{Ind}(\Amc)$, or \Amc violates a functionality assertion in \Tmc.
  \end{enumerate}
\end{lemma}
The proof is a straightforward extension of an analogous result for
$\ELI_\bot$ in \cite{ijcai-2013}. Details are omitted.

\subsection{Proof of Proposition~\ref{prop:autocentral}}
\label{app:mainauto}

We next give the main automaton construction underlying our upper
complexity bounds, stated as Proposition~\ref{prop:autocentral} in
the main paper. For convenience, we repeat the proposition here.

\smallskip 
\noindent
\PROPautocentral*

\smallskip We start with proving Point~1 of
Proposition~\ref{prop:autocentral}.  Let $Q=(\Tmc,\Sigma,q)$ be an OMQ
from $(\mathcal{ELIHF}_\bot,\text{BCQ})$. The automaton $\Amf_Q$ is
the intersection of two automata $\Amf_{Q,1}$ and $\Amf_{Q,2}$ and the
automaton $\Amf_\Tmc$ from Point~2 of
Proposition~\ref{prop:autocentral}. All of them run on $(|\Tmc| \cdot
|q|)$-ary $\Sigma_\varepsilon \cup \Sigma_N$-labeled trees. The first
automaton $\Amf_{Q,1}$ accepts the input tree $(T,L)$ iff it is proper.
This automaton is very simple to construct and its number of states is 
polynomial in $|\Tmc|$ and independent of $q$; we
omit details. The second automaton $\Amf_{Q,2}$ accepts $(T,L)$ iff
$\Amc_{(T,L)} \models Q$, provided that $\Amc_{(T,L)}$ is consistent
with
$\Tmc$. 

Before constructing $\Amf_{Q,2}$, we first extend the TBox \Tmc as
follows.  Let $q_1,\dots,q_\ell$ be the maximal connected components
of the BCQ $q$ from $Q$. We use \Qmc to denote the set of queries
that contains
\begin{itemize}

\item all queries from $\mn{id}(q_1) \cup \cdots \cup \mn{id}(q_\ell)$  
  which are weakly tree-shaped;

\item the queries $p_1,\dots,p_k$ from any forest decomposition
  $(p_{\mn{core}}, p_1,x_1, \ldots, p_k,x_k, \mu)$ in
  $\mn{fdec}(q_1) \cup \cdots \cup \mn{fdec}(q_\ell)$.
\end{itemize}
Each query in $q' \in \Qmc$ can be viewed as an
$\ELI^\cap$-concept~$C_{q'}$. We add to \Tmc the inclusion
$C_{q'} \sqsubseteq A_{q'}$ for each $q' \in \Qmc$, and convert to normal 
form. Call the resulting TBox $\Tmc^+$. The 
following lemma will guide the construction of the automaton $\Amf_{Q,2}$. 
%
\begin{lemma}
\label{lem:threecases}
Let \Amc be a pseudo tree ABox that is consistent with \Tmc. Then $\Amc,\Tmc 
\models q$
iff for all maximal connected components $q_j$ of $q$, one 
of the following properties is satisfied:
    \begin{enumerate}
  
  \item there is a forest decomposition $F=(q_{\mn{core}},
    q_1,x_1, \ldots,$ $q_k,x_k,\mu)$ of $q_j$ such that
  \begin{itemize}
    \item $\mu$ is a match for $q_{\mn{core}}$ in $\Imc_{\Tmc,\Amc}$
      whose
      range consists solely of core individuals from \Amc;
    \item for $1 \leq i \leq k$, there is a match $\pi$ for $q_i$ in 
    $\Imc_{\Tmc,\Amc}$ such that $\pi(x_i)=\mu(x_i)$;
  \end{itemize}

\item there is a query $q' \in \mn{id}(q_j)$ that is weakly tree-shaped
  and an $a \in \mn{Ind}(\Amc)$ that is not from the core part of
  \Amc and satisfies $\Amc,\Tmc^+ \models A_{q'}(a)$;

\item there is an $a \in \mn{Ind}(\Amc)$ and a set $S$ of
  concept names from $\Tmc$ such that $a \in A^{\Imc_{\Amc,\Tmc}}$ 
  for all $A \in S$ and $\Amc_S,\Tmc \models q_j$, where $\Amc_S
  = \{ A(a) \mid A \in S \}$. 

  \end{enumerate}

\end{lemma}
\begin{proof}
  (sketch)
  $\Amc,\Tmc \models q$ is witnessed by a match for every maximal connected 
  component $q_j$ of $q$ into the canonical model $\Imc_{\Amc,\Tmc}$ of \Amc 
  and \Tmc. We distinguish three kinds of such matches: 
  (i)~Matches that involve a core individual of \Amc: As $q_j$ is connected, 
  by Lemma~\ref{forest-decomp-lemma} we directly obtain Point~1.
  (ii)~Matches that involve a tree individual $a$ of \Amc but no core 
  individual define a weakly tree-shaped query $q'$ that results from 
  identifying variables in $q_j$. The root of $q'$ is mapped to $a$, and as
  $C_{q'} \sqsubseteq A_{q'} \in \Tmc^+$, it follows that $\Amc,\Tmc^+ \models 
  A_{q'}(a)$.
  (iii)~Matches that do not involve any ABox individuals: Consider the ABox 
  individual $a$ that is the root of the anonymous tree $q_j$ is mapped to. As 
  $\Tmc^+$ is in normal form, it is easy to prove using canonical models that 
  there is a set $S$ of concept names from $\Tmc$ such that $a \in 
  A^{\Imc_{\Amc,\Tmc}}$ for all $A \in S$ and $\Amc_S,\Tmc \models q_j$.
\end{proof}
We use $\mn{CN}(\Tmc^+)$ to denote the set of all concept names
in~$\Tmc^+$ (and likewise for $\mn{CN}(\Tmc)$) and $\mn{rol}(\Tmc^+)$
to denote the set of all role intersections in~$\Tmc^+$.
Define the TWAPA $\Amf_{Q,2} = (S,\Sigma,\delta,s_0,c)$ by setting
\begin{align*}
  S = \;& \{ s_0, s^j_{\mn{core}}, s^j_{\mn{tree}}, s^j_{\mn{anon}} \mid 1 
  \leq j \leq
  \ell \} \; \uplus \\
   & \{ s_{A,a},\, s_A,\, s_{A,R,a} ,\, s_{A,R},\, s_{A,R,\uparrow}\mid \\ 
   & \quad a \in 
   \Ind_{\mn{core}}, \ R \in \mn{rol}(\Tmc^+), \ A \in \mn{CN}(\Tmc^+) \}
\end{align*}
and $c(s)=1$ for all $s \in S$ (\ie, exactly the finite runs are
accepting). We introduce the transition function $\delta$ in several
steps and provide some explanation along the way. Start by setting
for all $\rho\in \Sigma_\varepsilon$ %
\begin{itemize}
  \item $\delta(s_0,\rho) = \displaystyle\bigwedge_{j=1..\ell}
    \dia0 s^j_{\mn{core}} \vee \dia0 s^j_{\mn{tree}} \vee 
    \dia0 s^j_{\mn{anon}}$,
\end{itemize}
which distinguishes the three cases in Lemma~\ref{lem:threecases}, for
each connected component of $q$.  Next put for all
$\rho\in \Sigma_\varepsilon$, $\nu\in \Sigma_N$, and
$1 \leq j \leq \ell$:
\begin{itemize}

  \item $\displaystyle \delta(s^j_{\mn{core}},\rho) = 
  \!\!\!\!\!\!\!\!
\bigvee_{ {(p_{\mn{core}}, p_1,x_1, \ldots, p_k,x_k, \mu) 
  \in \mn{fdec}(q_j)} \atop{\text{ with } \mu \text{ a homomorphism from } 
  p_{\mn{core}} \text{ to } \rho} } \; \bigwedge_{i \in 1..k} 
  \dia0 s_{A_{p_i},\mu(x_i)}$;

  \item $\displaystyle \delta(s^j_{\mn{tree}},\rho) = 
\bigvee_{i \in 
  1..m} \dia{i}s^j_{\mn{tree}}$;

  \item $\displaystyle \delta(s^j_{\mn{tree}},\nu) = 
\bigvee_{q' \in \mn{id}(q_j)}  \dia0 s_{A_{q'(x)}}
\vee 
\bigvee_{i \in 1..m} 
  \dia{i}s^j_{\mn{tree}}$;


  \item $\displaystyle \delta(s^j_{\mn{anon}},\rho) = 
    \bigvee_{a \in \mn{Ind}(\rho)} \bigvee_{S \subseteq \mn{CN}(\Tmc)
      \atop{\Amc_S,\Tmc\models q_j}}
    \Big (
    \bigwedge_{A \in S}
    \dia0  s^j_{A,a}
\Big )\;
\vee$\\[1mm]
\hspace*{20.65mm}$
\displaystyle
\bigvee_{i \in 
  1..m} \dia{i}s^j_{\mn{anon}}$;

  \item $\displaystyle \delta(s^j_{\mn{anon}},\nu) = 
    \bigvee_{S \subseteq \mn{CN}(\Tmc) \atop{\Amc_S,\Tmc\models q_j}}
    \Big (
    \bigwedge_{A \in S}
    \dia0  s^j_{A}
    \Big )
    \vee
    \bigvee_{i \in 1..m} \dia{i}s^j_{\mn{anon}} $.
  
\end{itemize}
The first line selects a forest decomposition as in Point~1 of
Lemma~\ref{lem:threecases}; lines two and three select an ABox
individual in a tree component of the pseudo tree ABox $\Amc_{(T,L)}$
as in Point~2 of that lemma; and lines four and five select an
individual from $\Amc_{(T,L)}$ as in Point~3, which can be either in
the core part or in a tree part. It remains to implement the proof
obligations expressed by states of the form $s_{A,a}$ and $s_A$.  The
former indicates that the concept name $A$ is made true in the
canonical model by the core individual $a$ and the latter that $A$ is
made true in the canonical model by the tree individual that
corresponds to the current point of the input tree. We make sure that
these obligations are satisfied by checking the existence of
corresponding derivation trees according to
Lemma~\ref{lem:derivationtrees}. We start with doing this for the core
part of pseudo tree ABoxes. Set for all $\rho\in \Sigma_\varepsilon$
and all $\nu\in \Sigma_N$:
\begin{itemize}

\item $\delta(s_{A,a},\rho) = \mn{true}$ if $A(a)\in \rho$ or $a \in
  \Ind(\rho)$ and $\top \sqsubseteq A \in \Tmc^+$;
  

  \item $\displaystyle \delta(s_{A,a},\rho) = \\[1mm]
  \hspace*{0.75mm}\bigvee_{\Tmc^+ \models A_1\sqcap \dotsb\sqcap A_n \sqsubseteq 
  A} \big(\dia0 s_{A_1,a} \wedge \dotsb \wedge \dia0 s_{A_n,a}\big) \vee 
  \\[1mm]
  \hspace*{4.7mm} \bigvee_{\exists R.B\sqsubseteq A\in \Tmc^+} \Big( 
  \bigvee_{R'(a,b) \in \rho \text{ with } \Tmc^+ \models R' \sqsubseteq R} \; 
  \dia0 s_{B,b} \vee \\[1mm]
  \hspace*{37mm} \bigvee_{i\in1..m}\dia{i}s_{B,R,a} \Big) \vee \\[1mm]
  \hspace*{3mm} \bigvee_{B\sqsubseteq \exists r . A\in \Tmc^+, \
    \mn{func}(r) \in \Tmc^+} \Big( 
  \bigvee_{r(b,a) \in \rho} \; 
  \dia0 s_{B,b} \vee \\[1mm]
  \hspace*{39.7mm} \bigvee_{i\in1..m}\dia{i}s_{B,r^-,a} \Big) $ \\[1mm]
  if $a \in \Ind(\rho)$ and $A(a) 
  \notin \rho$;
  
  \item $\displaystyle \delta(s_{A,R,a}, \nu) = \dia0 s_A$ if $a \in \nu$
    and there is an $R' \in \nu$ with $\Tmc^+ \models R' \sqsubseteq R$.

\end{itemize}
It should be obvious how the above transitions verify the existence of
a derivation tree. Note that the tree must be finite since runs are
required to be finite. We now deal with proof obligations in the trees
of pseudo tree ABoxes.  Set for all $\rho\in \Sigma_\varepsilon$ and
all $\nu\in \Sigma_N$:
\begin{itemize}

  \item $\displaystyle \delta(s_A,\nu) =$ true for all $s_A \in S$
    with $A\in \nu$ or 
    \mbox{$\top \sqsubseteq A \in \Tmc^+$};

  \item $\delta(s_A,\nu)= \displaystyle \!\bigvee_{\Tmc^+ \models A_1
  \sqcap \cdots \sqcap A_n \sqsubseteq A} \big( \dia0 s_{A_1} \wedge \cdots
  \wedge \dia0 s_{A_n} \big) \; \vee$
  \\[1mm]
  \hspace*{12.6mm}$\displaystyle \bigvee_{\exists R.B \sqsubseteq A
  \in \Tmc^+} \Big(\dia0 s_{B,R^-,\uparrow} \vee 
  \bigvee_{i \in 1..m} \dia{i} s_{B,R} \Big) \; \vee$
  \\[1mm]
  \hspace*{5mm}$\displaystyle \bigvee_{B \sqsubseteq \exists r . A
  \in \Tmc^+,\ \mn{func}(r) \in \Tmc^+} \Big(\dia0 s_{B,r,\uparrow} \vee 
  \bigvee_{i \in 1..m} \dia{i} s_{B,r^-} \Big)$ \\[2mm]
  for all $A \in \mn{CN}(\Tmc^+)$ with
$A \notin \nu$;
  
\item $\delta(s_{A,R,\uparrow},\nu) = \dia{-1}s_A$ if there is an $R' \in \nu$ 
with
  $\Tmc^+ \models R'\sqsubseteq R$ and $\nu \cap \Ind_{\mn{core}} =
  \emptyset$;
  
  \item $\delta(s_{A,R,\uparrow},\nu) = \dia{-1}s_{A,a}$ if there is an $R' 
  \in \nu$ with $\Tmc^+ \models R' \sqsubseteq R$
  and $a \in \nu$;
    
  \item $\delta(s_{A,R},\nu)= \dia0 s_A$ if there is an 
  $R' \in \nu$ with $\Tmc^+ \models R' \sqsubseteq R$.

%
\end{itemize}
We define $\delta(s,\alpha) = \mn{false}$ for all $s \in S, \alpha \in
\Sigma_\varepsilon \cup \Sigma_N$ not covered above. Using the
intuitions above and Lemmas~\ref{forest-decomp-lemma}
and~\ref{lem:derivationtrees}, one can prove the following.
\begin{lemma}
\label{lem:autoQcorr}
  Let $(T,L)$ be a $\Sigma_\varepsilon \cup \Sigma_N$-labeled tree
  that is proper 
  and such that $\Amc_{(T,L)}$ is
  consistent with \Tmc. Then $\Amf_{Q,2}$ accepts $(T,L)$ iff
  $\Amc_{(T,L)} \models Q$.
\end{lemma}
By construction and Lemma~\ref{lem:autoQcorr}, the overall TWAPA
$\Amf_Q$ accepts the tree language described in
Proposition~\ref{prop:autocentral} and it remains to analyse its size.

\smallskip 

First assume that \Tmc is formulated in $\mathcal{ELIHF}_\bot$. Using
the bound on the number of forest decompositions from
Section~\ref{app:forestdecomp}, it can be verified that the size of
the extended TBox~$\Tmc^+$ is at most $2^{p(|q|+\mn{log}(|\Tmc|))}$
for some polynomial $p$ and thus the same is true for the number of
states of the TWAPA $\Amf_{Q,2}$. Since the number of states of
$\Amf_{Q,1}$ is polynomial in the size of \Tmc and independent of $q$
and by the size bounds for $\Amf_\Tmc$ given in Point~2 of
Proposition~\ref{prop:autocentral} (and since intersection blows up
TWAPAs only polynomially), the bound of at most
$2^{p(|q|+\mn{log}(|\Tmc|))}$ states also applies to the overall TWAPA
$\Amf_Q$. 

Now for case when \Tmc is formulated in $\mathcal{ELHF}_\bot$.  By
Lemma~\ref{EL-forest-decomp-lemma}, we can then replace forest
decompositions with 
directed forest decompositions in
the construction of $\Amf_Q$. Moreover, Point~2 of
Lemma~\ref{lem:threecases}, can be replaced with
\begin{enumerate}

\item[2$'$.] the maximal fork rewriting $q^\fsf_j$ of $q_j$ is
  weakly ditree-shaped and there is an $a \in \mn{Ind}(\Amc)$ that is
  not from the core part of \Amc and satisfies $\Amc,\Tmc^+ \models
  A_{q^\fsf_j}(a)$

\end{enumerate}
Consequently, the set \Qmc used in the construction of $\Tmc^+$ now
only needs to contain
\begin{itemize}

\item the queries $q^\fsf_1,\dots,q^\fsf_\ell$;

\item the queries $p_1,\dots,p_k$ from any 
directed
  forest decomposition $(p_{\mn{core}}, p_1,x_1, \ldots, p_k,x_k,
  \mu)$ in $\mn{fdec}(q_1) \cup \cdots \cup \mn{fdec}(q_\ell)$.
\end{itemize}
It is then a consequence of the following lemma that the size of the
extended TBox $\Tmc^+$ is now bounded by $p(|q|+|\Tmc|)$ for some
polynomial $p$. The same arguments as before then allow us to carry
over that bound to the number of states in~$\Amf_Q$. The following is
a reformulation of Lemma~4 in \citeA{Lutz-DL-07}. Note that
this result crucially relies on Condition~6 from the definition of forest decompositions.
%
\begin{lemma}
  Let $q$ be a connected BCQ and let \Qmc be the set of all queries that
  occur as a tree component in a 
directed forest
  decomposition of $q$.  Then the cardinality of \Qmc is polynomial in~$q$.
\end{lemma}
We now sketch the construction of the TWAPA $\Amf_\Tmc$ from Point~2
of Proposition~\ref{prop:autocentral}. We first build a TWAPA $\Amf$
which accepts a proper 
$\Sigma_\varepsilon \cup \Sigma_N$-labeled tree $(T,L)$ iff
$\Amc_{(T,L)}$ is inconsistent with \Tmc, and then obtain $\Amf_\Tmc$
from $\Amf$ by complementing and intersecting with $\Amf_{Q,1}$. 

We can assume that \Tmc is satisfiable because in all considered
cases,
\begin{enumerate}

\item TBox satisfiability is not harder than the reasoning problem
  (containment or FO rewritability) that  we are interested in;

\item the result of containment or FO rewritability is trivial if at
  least one of the involved TBoxes is unsatisfiable.

\end{enumerate}
By Lemma~\ref{lem:derivationtrees}, we thus have to construct $\Amf$
such that it accepts the input tree $(T,L)$ iff a
functionality assertion from \Tmc is violated by $\Amc_{T,L}$ or
there is a derivation tree for $\bot(a)$ for some $a \in
\mn{Ind}(\Amc_{(T,L)})$. The former is straightforward
and the latter can be done in almost exactly the same way as in the
automaton $\Amf_{Q,2}$ above. To start, we put the following
transitions for all $\rho\in \Sigma_\varepsilon$ and all $\nu\in
\Sigma_N$, where $s_0$ is the initial state:
\begin{itemize}

\item $ \delta(s_0,\rho)= \displaystyle\bigvee_{a \in \mn{Ind}(\rho)}
  \dia0  s_{\bot,a} \vee
  \bigvee_{i \in 1..m} \dia{i} s_0$;

\item $ \delta(s_0,\nu)= \displaystyle s_{\bot} \vee
  \bigvee_{i \in 1..m} \dia{i} s_0$.

\end{itemize}
It thus remains to deal with the proof obligations $s_{\bot,a}$ and
$s_{\bot}$, which is done as in $\Amf_{Q,2}$ except that we use the
original TBox \Tmc in place of the extended TBox $\Tmc^+$ (and treat
$\bot$ like a concept name from $\mn{CN}(\Tmc)$).  This finishes the
construction of the automaton $\Amf_\Tmc$. The size of at most
$p(|\Tmc|)$ states is easily verified. Note that the construction is
essentially independent of $q$ (except for condition~(iv) of
properness) because the transitions of $\Amf_{Q,2}$ that refer to
forest decompositions are replaced by the transitions given above.

\subsection{Deciding Containment}

We prove the upper bounds for containment stated in
Theorem~\ref{thm:main}. Actually, they follow directly from
Proposition~\ref{prop:autocentral}, the fact that $Q_1 \subseteq Q_2$
(where $Q_i=(\Tmc_i,\Sigma,q_i)$) iff $L(\Amf_{Q_1}) \cap L(\Amf_{\Tmc_2})
\subseteq L(\Amf_{Q_2})$ iff $L(\Amf_{Q_1}) \cap L(\Amf_{\Tmc_2}) \cap
\overline{L(\Amf_{Q_2})}$ is empty, the closure of TWAPAs under
(polynomial) intersection and complement, and the complexity of TWAPA
emptiness.

\subsection{Deciding FO rewritability}

We prove the upper bounds for FO rewritability stated in
Theorems~\ref{thm:main} and~\ref{thm:mainzoom}. The proof is based on
the characterization from Points~2 (for $\mathcal{ELIHF}_\bot$) and
Point~1 (for $\mathcal{ELHF}_\bot$) of Theorem~\ref{lem:witabox} and
uses the automata from Proposition~\ref{prop:autocentral}, in a
slightly adapted form. 

\smallskip 

Let $Q=(\Tmc,\Sigma,q)$ be an OMQ from
($\mathcal{ELIHF}_\bot$, conBCQ) and $k_0 =
2^{4(|\Tmc|+2^{|q|})^2}$ the bound from Point~2 of
Theorem~\ref{lem:witabox}. By that theorem, $Q$ is not FO-rewritable
iff there is a pseudo tree ABox \Amc of width at most $|q|$ and
outdegree at most $|\Tmc|$ that satisfies the following conditions:
\begin{enumerate}


\item $\Amc$ is consistent with \Tmc;

\item $\Amc\models Q$;

\item $\Amc|_{>0} \not\models Q$;

\item $\Amc|_{\leq k_0}\not\models Q$.

\end{enumerate}
We aim to build a TWAPA \Amf that accepts representations of such
ABoxes; it then remains to decide emptiness. To deal with the
`truncated' ABoxes $\Amc|_{\leq k_0}$, we need to endow
$\Sigma_\varepsilon \cup \Sigma_N$-labeled trees with a counting
component. More precisely, we now use $\Sigma_\varepsilon \cup
(\Sigma_N \times [k_0])$-labeled trees, where
$[k_0]=\{1,\dots,k_0+1\}$. All notions for $\Sigma_\varepsilon \cup
\Sigma_N$-labeled trees such as properness, the associated ABox 
carry over to the extended alphabet. Additionally,
we say that a $\Sigma_\varepsilon \cup (\Sigma_N \times
[k_0])$-labeled tree $(T,L)$ is \emph{counting} if for every node $x
\in T$ on level $i >0$, $L(T)=(\alpha,j)$ implies $j=\mn{min}(i,k_0+1)$.

The desired TWAPA \Amf is the intersection of five TWAPAs
$\Amf_0,\dots, \Amf_4$. While $\Amf_0$ makes sure that the input tree
$(T,L)$ is proper and counting, each of the 
automata
$\Amf_1,\dots,\Amf_4$ makes sure that the ABox $\Amc_{(T,L)}$
satisfies the corresponding condition from the above list. In fact, we
have already seen in Section~\ref{app:mainauto} how to build TWAPAs
for Conditions~1 and~2; they are easily adapted to the new input
format and simply ignore the additional counting component of input
trees. Moreover, the automaton $\Amf_Q$ which ensures Condition~2
is easily modified to ensure Conditions~3 and~4 provided that the 
input tree is counting; the modified automaton simply ignores those
parts of the input that are `truncated away'. 

It thus remains to verify that we can build the automaton~$\Amf_0$.
Properness was already dealt with in
Section~\ref{app:mainauto}. We additionally need to verify that the
input tree is counting. This can be done with $O(\mn{log}(k_0))$
states: we send a copy to the automaton to every tree node, for every
$i$-th bit, $i \in \{1,\dots,\lceil\mn{log}(k_0)\rceil\}$. Based on the node
label, we determine the value $t \in \{0,1\}$ of the $i$-th bit at all
successors and send a copy of the automaton in state ``$\mn{bit}\;i{=}t$''
to all successors nodes, where that value is verified. 

It can be verified that the constructed overall TWAPA $\Amf$ has
${2^{p(|q_1| + \mn{log}(|\Tmc_1|)) }}$ states, $p$ a polynomial.  From
the complexity of TWAPA emptiness, we thus get Point~1 in
Theorem~\ref{thm:mainzoom}.  If $\Tmc$ is formulated in
$\mathcal{ELIHF}_\bot$, then \Amf has only $p(|q_1| + |\Tmc_1|)$
states due to the improved bounds for this logic in
Theorem~\ref{lem:witabox} and Proposition~\ref{prop:autocentral}.
Consequently, we also obtain the upper bound in Point~2 of
Theorem~\ref{thm:main}.
 
\section{Rooted Queries}

We now establish the co\NExpTime\ upper bounds for rooted queries (Theorem \ref{thm:rootedupper}).
We first give the proof for containment, and 
afterwards we explain how the construction can be modified to handle FO-rewritability.

\subsection{Overview of upper bound for containment }


For convenience, we repeat the result we aim to prove. 

\begin{theorem}
  Containment for OMQs in $(\mathcal{ELIHF}_\bot,\text{rCQ})$ is in co\NExpTime.  
\end{theorem}

Let $\Tmc_1$ and $\Tmc_2$ be $\mathcal{ELIHF}_\bot$ TBoxes, let $\Sigma$ be an ABox signature, and let $q_1$ and $q_2$
be rooted CQs. 
We recall that by Proposition \ref{prop:conttree}, $(\Tmc_1,\Sigma, q_{1}) \not \subseteq (\Tmc_2,\Sigma, q_{2})$
iff there is a pseudo
tree $\Sigma$-ABox \Amc of outdegree at most $|\Tmc_1|$ and width at
most~$|q_{1}|$ that is consistent with both $\Tmc_1$ and $\Tmc_2$ and a
tuple \abf from the core of \Amc such that $\Tmc_1, \Amc \models q_1(\vec{a})$ and $\Tmc_2, \Amc \not \models q_2(\vec{a})$.
To test for the existence of such a witness ABox and tuple, we proceed as follows. \smallskip

\noindent\textbf{Step 1} Guess the following:
\begin{itemize}
\item pseudo tree $\Sigma$-ABox $\Amc_\mn{init}$ whose core is bounded by $|q_1|$, whose outdegree is bounded by $|\Tmc_1|$, 
whose depth is bounded by $m_q= \mathsf{max}(|q_1|, |q_2|)$ (with $\Umc_q$ the set of individuals that are at distance 
exactly $m_q$ from the core)
\item tuple $\vec{a}$ of individuals from the core of $\Amc$, of the same arity as $q_1$ and $q_2$ 
\item ABoxes $\cabox_1$ and $\cabox_2$ such that $\Amc \subseteq \cabox_{i}$ and
$\cabox_i \setminus \Amc \subseteq \{B(a) \mid a \in \mn{Ind}(\Amc), B \in \NC\}$, for $i \in \{1,2\}$
\item two `global' candidate transfer sequences $\mathcal{Y}_0^1, \ldots, \mathcal{Y}^1_{N_1}$ and $\mathcal{Y}^2_0, \ldots, \mathcal{Y}^2_{N_2}$ 
for $(\Umc_q, \Tmc_1)$ and $(\Umc_q, \Tmc_2)$ respectively (precise definition given later)
\end{itemize}

\noindent Intuitively, the guessed ABox $\Amc_\mn{init}$ is the initial portion of a witness for non-containment,
obtained by restricting the witness ABox to individuals within distance $m_q=\mathsf{max}(|q_1|, |q_2|)$ of the core, 
and $\vec{a}$ is a tuple witnessing the non-containment. 
The ABox $\cabox_{i}$ ($i \in \{1, 2\}$) enriches $\Amc_\mn{init}$ with the concept 
assertions over $\mn{Ind}(\Amc)$ that are entailed from the full witness ABox and the TBox $\Tmc_i$.
To keep track of the interactions between the guessed part $\Amc_\mn{init}$ and the missing trees, 
we generalize the notion of transfer sequence to 
sets of individuals (rather than a single individual).
The guessed sequence $\mathcal{Y}^{i}_0, \ldots, \mathcal{Y}^{i}_{N_i}$ ($i \in \{1,2\}$) 
corresponds to the transfer sequence of the full witness ABox with respect to the individuals in $\Umc_q$ (occuring at depth $m_q$)
and the TBox~$\Tmc_i$. 

\medskip

\noindent\textbf{Step 2} Verify that:
\begin{itemize}
\item for $i \in \{1,2\}$, $\Bmc_i$ is consistent with $\Tmc_i$
\item $\cabox_1, \Tmc_1 \models q_1(\vec{a})$ and $\cabox_2, \Tmc_2 \not \models q_2(\vec{a})$
\item for $i \in \{1,2\}$, the candidate transfer sequence $\mathcal{Y}^{i}_0, \ldots, \mathcal{Y}^{i}_{N_i}$ is compatible with $(\Amc, \cabox_i)$ 
\end{itemize}
and return no if one of these conditions fails to hold. 
\smallskip

\noindent The second point corresponds to checking that, with respect to the full witness, we have $q_1(\vec{a})$ but not $q_2(\vec{a})$. 
Indeed, since $q_1$ and $q_2$ are rooted, 
we know that query matches only involve individuals that are within distance 
$m_q$
of the core. 
Since $\cabox_i$ contains all concept assertions for these individuals that are entailed w.r.t.\ the full witness ABox,
it can be used in place of the witness.
The compatibility checks in the third item (which will be made precise further) 
will be used to ensure that $\mathcal{Y}^{i}_0, \ldots, \mathcal{Y}^{i}_{N_i}$ 
is the transfer sequence of the full witness ABox w.r.t.\ $\Umc_q$ and $\Tmc_i$. 

\medskip

\noindent\textbf{Step 3} For each 
individual $u \in \Umc_q$, 
construct a tree automaton that checks whether there is 
a tree-shaped ABox $\Amc_u$ rooted at $u$ 
that does not contain any concept assertion $A(u)$ and is such that 
for both $i \in \{1,2\}$, we have:
\begin{itemize}
\item 
 $\mathcal{Y}^{i}_0, \ldots, \mathcal{Y}^{i}_{N_i}$ is compatible with $\Amc_u$ at $u$ w.r.t.\ $\Tmc_i$
\item $\Amc_u \cup \mathcal{Y}^{i}_{N_i}$ is consistent with $\Tmc_i$ 
\item if $\mathsf{func}(r) \in \Tmc_i$ and $r(u,u') \in \Amc$, then $\Amc_u$ does not contain any assertion of the form $r(u,u'')$
\end{itemize}
Return yes if all of these automata are non-empty, 
otherwise return no. \smallskip

\noindent The final step checks that it is possible to construct, for every individual $u$ at depth $m_q$,
a tree-shaped ABox $\Amc_u$, such that the ABox $\Amc_\mn{init} \cup \bigcup_{u \in \Umc_q} \Amc_u$ that is 
obtained by attaching all of these trees to $\Amc_\mn{init}$  yields the full witness ABox (by renaming individuals,
we can assume that $\mn{Ind}(\Amc_\mn{init}) \cap \mn{Ind}(\Amc_u)=\{u\}$ and $\mn{Ind}(\Amc_u) \cap \mn{Ind}(\Amc_{v})=\emptyset$
for $u,v \in $ with $u \neq v$). 
For this to be the case, we need to ensure that 
the tree-shaped ABoxes allow us to infer exactly those concept assertions 
present in the candidate transfer sequence (this is the purpose of the compatibility condition, formalized further). 
We must also ensure that after adding the entailed assertions $\mathcal{Y}^{i}_{N_i}$ to ABox $\Amc_u$,
the resulting ABox is consistent with both TBoxes and 
that no violations of functionality assertions 
are introduced when attaching $\Amc_u$ to $\Amc_\mn{init}$. 

\medskip

In what follows, we provide more details on Steps 2 and 3 of this procedure. 

\subsection{Query entailment checks in Step 2}
We briefly explain how to perform the query entailment checks in Step 2. 
We focus on the first entailment check $\cabox_1, \Tmc_1 \models q_1(\vec{a})$, but the same construction can be used to decide 
whether $\cabox_2, \Tmc_2 \not \models q_2(\vec{a})$. 
The idea is as follows: to decide whether $\cabox_1, \Tmc_1 \models q_1(\vec{a})$, we 
will compute the restriction $\Imc_{\cabox_1, \Tmc}^{q_1}$ of the canonical model $\Imc_{\cabox_1, \Tmc}$
to the ABox individuals in $\Bmc$ and the new domain elements that are 
within distance $|q_1|$ of one of these individuals. Then it suffices to iterate over all (exponentially many) 
mappings $\pi$  from the variables of $q_1$ into $\Delta^{\Imc_{\Bmc, \Tmc}^{q_1}}$
and to check if one of these mappings is a match for the query. 

We sketch how to construct the interpretation $\Imc^{q_1}_{\cabox_1,\Tmc}$ in exponential time. 
For convenience, we adopt the ABox representation of interpretations. 
We first include all concept and role assertions that are entailed from $\cabox_1,\Tmc$; such assertions can be computed in 
exponential time (e.g., by applying the modified closure rules from Appendix \ref{canmod-def}). 
Next, for each individual $a \in \mn{Ind}(\cabox_1)$, we let $C_a$ be the conjunction of concepts $A$
such that $A(a) \in \cabox_1$.  To determine which successors we need to connect to $a$, 
we compute all axioms of the form $C_a \sqsubseteq \exists R. D$, 
where $R$ is a conjunction of roles from $\mn{sig}(\Tmc)$ and $D$ is a conjunction of concept names from $\mn{sig}(\Tmc)$. 
We keep only the `strongest' such axioms, i.e.\ those for which there does not exist an entailed axioms $C_a \sqsubseteq \exists R'. D'$ 
where $R'$ (resp.\ $D'$) contains a superset of role (resp.\ concept) names, and at least one of these superset relationships is strict. 
It is not hard to show that there can be at most $|\Tmc|$ strongest axioms, one for each existential restriction on the right-hand side 
of an inclusion in $\Tmc$. If $C_a \sqsubseteq \exists R. D$ is a strongest entailed axiom, and 
there is no $b \in \mn{Ind}(\Amc)$ such that $\Amc, \Tmc \models R(a,b)$
and $\Amc, \Tmc \models D(b)$, then we pick a fresh individual $c$ and add the following assertions:
$
\{r(a,c) \mid r \in R\} \cup \{A(c) \mid A \in D\} 
$. 
For each of the newly introduced individuals, we proceed in exactly the same manner to construct its successors, 
stopping when an individual has no successors, or when the individual occurs at distance $|q|$ from one of the original individuals. 
Since the number of successors of an individual is bounded by $|\Tmc|$, and we stop producing successors at depth $|q|$, we only introduce exponentially 
many individuals. Moreover, to decide which successors to add (and which concepts and roles they should satisfy), 
we perform at most exponentially many entailment checks, and every
such check can be performed in exponential time. 

\subsection{Transfer sequences for frontier individuals}
We next formally introduce the generalized notion of transfer sequence, as well as the candidate transfer sequences that we guess in Step 1. 

Consider an arbitrary pseudo tree ABox $\Amc$. 
We call a set of individuals $\{u_1, \ldots, u_\ell\} \subseteq \mn{Ind}(\Amc)$ a \emph{valid frontier for $\Amc$}
if there do not exist $u_i \neq u_j$ such that $u_i$ is a descendant of $u_j$
in one of the tree-shaped ABoxes of $\Amc$. 
If $\Umc=\{u_1, \ldots, u_\ell\}$ is a valid frontier for $\Amc$, then we use $\Amc_{\Umc}^{\uparrow}$ 
to denote the ABox obtained from $\Amc$ by dropping the subtrees
$\Amc_{u_1}^{\downarrow}$, $\ldots $, $\Amc_{u_\ell}^{\downarrow}$ from $\Amc$,  
excepting the individuals $u_1, \ldots, u_\ell$.  Slightly abusing notation, 
we will extend the notation ${\sf AT}_{\Amc}^{\vdash}$ 
to sets of individuals as follows:  
$$
{\sf AT}_{\Amc, \Tmc}^{\vdash}(\Umc) := \{A(u) \mid u \in \Umc, A(u) \in 
\Amc^{c}_{\Tmc}\}. 
$$
(Note that we add $\Tmc$ to the subscript to make clear which TBox was used to complete $\Amc$.)
%

If $\Umc=\{u_1, \ldots, u_\ell\}$ is a valid frontier for $\Amc$,
then the \emph{transfer sequence $\Xmc_{0},\Xmc_{1}, \ldots$ 
of $(\Amc,\Umc)$ \wrt $\Tmc$} 
is defined inductively as follows: 
\begin{itemize}
  \item $\Xmc_{0}=  {\sf AT}_{\Amc^{0}}^{\vdash}(\Umc)$, where $\Amc^{0}= \Amc_{\Umc}^{\uparrow}$;
  \item $\Xmc_{1} = {\sf AT}_{\Amc^{1}}^{\vdash}(\Umc)$, where $\Amc^{1} = 
  \bigcup_{u \in \Umc} \Amc_{u}^{\downarrow} \cup 
  \Xmc_{0}$;
  \item for $i\geq 0$, $\Xmc_{2i+2}= {\sf AT}_{\Amc^{2i+2}}^{\vdash}(\Umc)$, where 
  $\Amc^{2i+2}= \Amc^{2i} \cup 
  \Xmc_{2i+1}$ (equivalently: $\Amc^{2i+2}= \Amc_{\Umc}^{\uparrow} \cup 
  \Xmc_{2i+1}$);
  \item for $i\geq 1$, $\Xmc_{2i+1}= {\sf AT}_{\Amc^{2i+1}}^{\vdash}(\Umc)$, where 
  $\Amc^{2i+1}= \Amc^{2i-1} \cup 
  \Xmc_{2i}$ (equivalently: $\Amc^{2i+1}= \bigcup_{u \in \Umc} \Amc_{u}^{\downarrow} \cup 
  \Xmc_{2i}$).
\end{itemize}

\smallskip

An analogue of Lemma \ref{lem:transfer1} can be shown:

\begin{lemma}\label{gen-ts-lemma}
Let $N= (|\Umc| \cdot |\mn{sig}(\Tmc)|) +1$. Then $\Xmc_{N}= \Xmc_{N'}$ 
 for all $N'>N$ and $(\Amc^{N-1})^c_\Tmc \cup (\Amc^{N})^c_\Tmc = \Amc_{\Tmc}^{c}$. 
\end{lemma}

By \emph{candidate transfer sequence} for $(\Umc, \Tmc)$ 
we mean a sequence $\Xmc_{0}, \Xmc_{1}, \ldots, \Xmc_{N}$ 
such that $N= (|\Umc| \cdot |\mn{sig}(\Tmc)|) +1 $ and for every $j \geq 0$,
$\Xmc_{j} \subseteq \{A(u_i) \mid A \in \NC \cap \mathsf{sig}(\Tmc), u_i \in \Umc\}$ 
and $\Xmc_{j} \subseteq \Xmc_{j+1}$. In our procedure, we consider 
candidate transfer sequences for $(\Umc_q, \Tmc_1)$ and $(\Umc_q, \Tmc_2)$,
which will terminate by the indices
$N_1 = (|\Umc_q| \cdot |\mn{sig}(\Tmc_1)|) +1$ and $N_2 =(|\Umc_q| \cdot |\mn{sig}(\Tmc_2)|) +1$, respectively. 
Observe that $|\Umc_q| \leq |\Tmc_i|^{m_q}$, so $N_i$ is polynomial in $|\Tmc_i|$ and exponential in 
$\mathsf{max}(|q_1|, |q_2|)$. 

\subsection{Compatibility of candidate transfer sequences} 
Let $\Amc$ be a pseudo tree ABox and $\Bmc \supseteq \Amc$ be an ABox with $\Bmc \setminus \Amc \subseteq \{B(a) \mid a \in \mn{Ind}(\Amc), B \in \NC\}$.
Further let $\Umc 
\subseteq \mn{Ind}(\Amc)$ be a subset of the leaves  of $\Amc$ (i.e.\ individuals occurring in one of the trees  of $\Amc$ but without any successors),
and let
$\Xmc=\Xmc_{0}, \Xmc_{1}, \ldots, \Xmc_{N}$ be a candidate transfer sequence for $(\Umc, \Tmc)$.
We say that $\Xmc$ 
is \emph{compatible with ($\Amc$, $\Bmc$) 
w.r.t.\ $(\Umc, \Tmc)$} iff: 
\begin{enumerate}
\item $\Xmc_{0}=  {\sf AT}_{\Amc^0, \Tmc}^{\vdash}(\Umc)$, where $\Dmc^{0}= \Amc$; 
\item for every  $i\geq 0$ with $2i+2 < N$:\\
$\Xmc_{2i+2}=  {\sf AT}_{\Amc^{2i+2},\Tmc}^{\vdash}(\Umc)$ where $\Dmc^{2i+2}= 
\Amc \cup\Xmc_{2i+1}$;
\item $\Xmc_{N}= \{A(u) \in \Bmc \mid u \in \Umc\}$; 
\item for every $B(a)$ with $a \in \mn{Ind}(\Amc)$ and $B \in \NC$:\newline $B(a) \in \Bmc$ iff $\Tmc, \Amc \cup \Xmc_{N} \models B(a)$ 
\end{enumerate}
Checking compatibility of a candidate transfer sequence w.r.t. a pair of ABoxes 
can be decided in \ExpTime. Indeed, 
all four conditions involve computing the closure of an exponential-sized ABox w.r.t.\ $\Tmc$, which can be done in exponential time. 
Indeed, there are only exponentially many concept assertions that can 
be added, so only exponentially many rule applications are needed to reach the closure,
and finding the next rule to apply involves an exponential number of (\ExpTime) entailment checks.

Next let $u \in \Umc$, and let $\Gmc$ be a tree-shaped ABox with root~$u$. 
We say that $\Xmc$ is \emph{compatible with 
$\Gmc$ at $u$ w.r.t.\ $\Tmc$} if and only if:
\begin{itemize}
\item $\Xmc_{1}(u) = {\sf AT}_{\Dmc^{1}, \Tmc}^{\vdash}(u)$ where $\Dmc^{1} =  \Gmc \cup 
  \Xmc_{0}(u)$
\item for every $i \geq 1$ with $2i+1 \leq N$: $\Xmc_{2i+1}(u)= {\sf AT}_{\Dmc^{2i+1},\Tmc}^{\vdash}(u)$, where 
  $\Dmc^{2i+1}= \Gmc \cup  \Xmc_{2i}(u)$
\end{itemize}
where, slightly abusing notation, we use the notation $\Xmc_i(u)$ to mean the set $\{A(u) \mid A(u) \in \Xmc_i\}$.

\smallskip



\subsection{Automata construction} 
Let $\Amc$ be the guessed pseudo tree ABox, 
let $\Ymc^1$ 
and $\Ymc^2$ 
be the guessed candidate transfer sequences for $(\Umc_q, \Tmc_1)$ and $(\Umc_q, \Tmc_2)$ respectively,
and let $\Umc_q=\{u_1, \ldots, u_{\ell_q}\}$. 

To implement Step 3 of the procedure, we need to construct, for every $1 \leq j \leq \ell_q$, a TWAPA $\Amf_j$
that accepts encodings of tree-shaped ABoxes $\Gmc_j$ with root node $u_j$ 
satisfying the conditions of Proposition \ref{prop-cont-nexp}. 
%
The desired automaton $\Amf_j$ can be obtained by intersecting 
the following automaton:
\begin{itemize}
\item $\Amf_\mn{cons}$ that ensures that the encoded ABox is consistent with 
the TBoxes $\Tmc_1$ and $\Tmc_2$.
\item $\Amf_\mn{funct}$ that ensures that the encoded ABox, when added to the ABox $\Amc$, 
does not violate the functionality assertions in $\Tmc_1$ and $\Tmc_2$.
Specifically, we need to ensure that if $\mathsf{func}(r) \in \Tmc_k$ ($k \in \{1,2\}$)
and $r(u_j,u') \in \Amc$, then the encoded ABox (whose root individual is $u_j$) does not contain any assertion of the form $r(u_j,u'')$.
\item for every $1 \leq 2i+1 \leq N$ and $k \in \{1,2\}$, an automaton $\Amf^{\Ymc^k}_{2i+1}$
that determines whether $\Ymc_{2i+1}^k(u_j)$ is precisely the set of concept assertions about $u_j$
that are entailed from  
$\Tmc$, the encoded ABox, and the assertions in $\Ymc_{2i}^k(u_j)$.
\end{itemize}
Note that the automaton $\Amf^{\Ymc^k}_{2i+1}$ in the third item 
can be constructed by intersecting automata that check 
whether a given concept assertion $A(u_j) \in \Ymc_{2i+1}^k(u_j)$ is entailed 
with those checking that each concept assertion $A(u_j) \not \in \Ymc_{2i+1}^k(u_j)$ 
(with $A \in \NC \cap \mathsf{sig}(\Tmc)$)
is not entailed. Moreover, the automata checking whether a concept 
is not entailed at the root $u_j$ can be obtained by complementing 
the automaton that accepts trees in which the concept is entailed at the root. 

Importantly, because the sets $\Ymc_{i}(u_j)$ 
increase monotonically
and only contain concept assertions about the individual $u_j$, there are only 
\emph{polynomially many} different elements in the set 
$\{ (\Ymc_{2i+1}(u_j), \Ymc_{2i}(u_j))\mid 1 \leq 2i+1 \leq N\}$.
It follows that $\Amf_j$ can be obtained by \emph{intersecting a polynomial 
number of automata}. Moreover, it is not hard to see that each of the
component automata can be \emph{constructed in polynomial time}. 
Since there are (at most) single exponentially many elements in $\Umc_{q}$,
and emptiness of TWAPAs can be tested in single-exponential time, it follows that Step 3 can be
performed in \ExpTime.

\subsection{Correctness of the procedure}

We have already given the main lines of the argument in the overview, so here we concentrate on
the following proposition, which is the key step to establishing correctness. 

\begin{proposition}\label{prop-cont-nexp}
Let $\Amc$ be a pseudo tree ABox, 
let $\Bmc \supseteq \Amc$ be an ABox with $\Bmc \setminus \Amc \subseteq \{B(a) \mid a \in \mn{Ind}(\Amc), B \in \NC\}$ that is consistent with $\Tmc$,
and let $\Umc= \{u_1, \ldots, u_\ell\} \subseteq \mn{Ind}(\Amc)$ be a subset of the leaves in $\Amc$. 
Suppose that
\begin{enumerate} 
\item the candidate transfer sequence $\Xmc=\Xmc_0, \ldots, \Xmc_N$  
 is compatible with $(\Amc, \Bmc)$ w.r.t.\ $(\Umc, \Tmc)$, and 
\item 
there exist tree-shaped ABoxes $\Gmc_1, \ldots, \Gmc_\ell$ such that 
$\mn{Ind}(\Gmc_j) \cap \mn{Ind}(\Gmc_j')=\emptyset$ for every $j \neq j'$
and for every $1 \leq j \leq \ell$:
\begin{itemize}
\item $\mn{Ind}(\Amc) \cap \mn{Ind}(\Gmc_j)=\{u_j\}$,
\item $\Gmc_j$ does not contain any concept assertion $A(u_j)$,
\item $\Gmc_j \cup \Xmc_N$ is consistent with $\Tmc$,
\item $\Amc \cup \Gmc_j$ does not violate any functionality assertion in $\Tmc$,
\item $\Gmc_j$ is compatible with $\Xmc$ at $u_j$ w.r.t.\ $\Tmc$. 
\end{itemize}
\end{enumerate}
Let $\Amc^* = \Amc \cup \bigcup_{1 \leq j \leq \ell} \Gmc_j$. 
Then: 
\begin{itemize}
\item $\Amc^*$ is consistent with $\Tmc$, 
\item $\Xmc$ is the transfer sequence for $(\Amc^*, \{u_1, \ldots, u_\ell\})$, and 
\item $\Tmc, \Amc^* \models A(a)$ iff $A(a) \in \Bmc$ (for $a \in \mn{Ind}(\Amc)$, $A \in \NC$) 
\end{itemize}
\end{proposition}
%
%
\begin{proof}
Let $\Amc, \Bmc, \Gmc_1, \ldots, \Gmc_\ell, \Umc, \Tmc, \Xmc$ be as in the statement. 

To show consistency of $\Amc^*$ with $\Tmc$,  first note that $\Amc^*$ does not violate 
any functionality assertions in $\Tmc$, since each of the ABoxes $\Amc, \Gmc_1, \ldots, \Gmc_\ell$
is consistent with $\Tmc$, there is no individual shared by two different 
$\Gmc_j$, 
and by assumption, for every $1 \leq j \leq \ell$, the ABox $\Amc \cup \Gmc_j$ does not violate any functionality assertion in $\Tmc$. 

Let $\Amc^+ = \Amc \cup \Xmc_N$, and  let $\Gmc_j^+ = \Gmc_j \cup \Xmc_N(u_j)$ for $1 \leq j \leq \ell$. 
We know that each of the ABoxes $\Amc^+, \Gmc_1^+, \ldots, \Gmc_\ell^+$ is consistent with $\Tmc$, and hence possesses a canonical model. 
We let $\Imc_{\Amc^+,\Tmc}$ and $\Imc_{\Gmc_j^+,\Tmc}$ be the canonical models for $\Amc^+, \Tmc$ (resp.\ $\Gmc_j^+, \Tmc$), as defined in Appendix \ref{canmod-def}. 
Without loss of generality, we may assume that $\Delta^{\Imc_{\Gmc_j,\Tmc}} \cap \Delta^{\Imc_{\Gmc_k,\Tmc}}=\emptyset$
for every $j\neq k$, and that $\Delta^{\Imc_{\Amc,\Tmc}} \cap \Delta^{\Imc_{\Gmc_j,\Tmc}}=\emptyset$ for every $1 \leq j \leq \ell$. 
We recall that these interpretations can be seen as adding to the original ABox all entailed assertions about the ABox individuals, and additionally attaching weakly tree-shaped interpretations to each of the ABox individuals in order to witness the existential restrictions on the right-hand side of TBox axioms. 
For every $u_j \in \Umc$, we let $\Imc_j^\Amc$ (resp.\ $\Imc_j^\Gmc$) be the weakly tree-shaped interpretation that is attached to the individual $u_j$ in $\Imc_{\Amc,\Tmc}$
(resp.\ $\Imc_{\Gmc_k,\Tmc}$). 
Define the interpretation $\Jmc$ as the union of the interpretations $\Amc^+, \Gmc_1^+, \ldots, \Gmc_\ell^+$: 
\begin{itemize}
\item $\Delta^\Jmc = \Delta^{\Imc_{\Amc,\Tmc}} \cup \bigcup_{1 \leq j \leq \ell}  \Delta^{\Imc_{\Gmc_j,\Tmc}}$
\item for every $A \in \NC$: $A^\Jmc = A^{\Imc_{\Amc,\Tmc}} \cup \bigcup_{1 \leq j \leq \ell}  A^{\Imc_{\Gmc_j,\Tmc}}$
\item for every $r \in \NR$: $r^\Jmc = r^{\Imc_{\Amc,\Tmc}} \cup \bigcup_{1 \leq j \leq \ell}  r^{\Imc_{\Gmc_j,\Tmc}}$
\end{itemize}
To satisfy the functionality assertions, we proceed as follows. For every $u_j \in \Umc$ and functional role $R$:
\begin{itemize}
\item If $u_j \in \exists R^{\Imc_{\Amc^+,\Tmc}}$ and there is no $b$ such that $\Amc^+, \Tmc \models R(u_j,b)$, then let $e$ be  the unique element in $\Delta^{\Imc_{\Amc,\Tmc}} $ such that $(u_j, e) \in R^{\Imc_{\Amc,\Tmc}}$. This element must belong to the weakly tree-shaped subinterpretation $\Imc_j^\Amc$. Remove the element $e$ and all of its descendants in $\Imc_j^\Amc$ from $\Delta^\Jmc$. 
\item If $u_j \in \exists R^{\Imc_{\Gmc_j^+,\Tmc}}$ and there is no $b$ such that $\Gmc_j^+, \Tmc \models R(u_j,b)$, then let $e$ be the unique element in $\Delta^{\Imc_{\Amc,\Tmc}} $ such that $(u_j, e) \in R^{\Imc_{\Gmc_j^+,\Tmc}}$. This element must belong to the weakly tree-shaped subinterpretation $\Imc_j^\Gmc$. Remove the element $e$ and all of its descendants in $\Imc_j^\Gmc$ from $\Delta^\Jmc$. 
\end{itemize}
Call the resulting interpretation $\Jmc^-$. We claim that $\Jmc^-$ is a model of $\Amc^*$ and $\Tmc$. First observe that $\Jmc^-$ makes true all ABox assertions in $\Amc^*$, since $\Jmc$ satisfies this property and the modifications that were made to $\Jmc$ only involve elements that did not occur in the original ABoxes.
Because of our modifications, we have resolved all of the violations of functionality axioms that were introduced when combining the interpretations.
It can also be easily seen that axioms of the forms $ A \sqsubseteq \bot$, $ \top  \sqsubseteq A$,  $ B_1 \sqcap B_2 \sqsubseteq A$, and $\exists r . B \sqsubseteq A$
are all satisfied in $\Jmc^-$, since they were satisfied in each of the interpretations $\Imc_{\Amc^+,\Tmc}$, $\Imc_{\Gmc_j^+,\Tmc}$, $\ldots$, $\Imc_{\Gmc_j^+,\Tmc}$. 
Finally, if $e \in A^{\Jmc^-}$ and $A \sqsubseteq \exists r . B \in \Tmc$, then either we have the same witnessing $r$-successor $e'$ as was used in the component interpretation containing the element $e$, or $e \in \Umc$, and we were only allowed to remove $e'$ (and the whole tree-shaped interpretation rooted at $e'$) because in the ABox $\Amc^*$, 
there was an ABox individual that acted as the witnessing $r$-successor. (and which is present in $\Jmc^-$) Thus, $\Jmc^-$ is a model of $\Amc^*, \Tmc$, so $\Amc^*$ is consistent with $\Tmc$. 

\smallskip

We next prove by induction that $\Xmc=\Xmc_0, \ldots, \Xmc_n$ is the transfer sequence for $(\Amc^*, \Umc)$. 
We start by considering the first set in the sequence ($\Xmc_0$). We know that $\Xmc_{0}=  {\sf AT}_{\Amc, \Tmc}^{\vdash}(\Umc)$ since 
$\Xmc$ is compatible with $(\Amc, \Bmc)$ w.r.t.\ $(\Umc, \Tmc)$. 
We then use the fact that, for every $1 \leq j \leq \ell$, the ABox $\Gmc_j$ is such that $\mn{Ind}(\Amc) \cap \mn{Ind}(\Gmc_j)=\{u_j\}$
and does not contain any assertion $A(u_j)$ to infer that $(\Amc^*)^\uparrow_\Umc= \Amc$.
Thus, $\Xmc_0$ is the first element in the transfer sequence for $(\Amc^*, \Umc)$. 

For the second element $\Xmc_1$, we first note that, for every $0 \leq i \leq n$, $\Xmc_i= \bigcup_{1 \leq j \leq \ell} \Xmc_i(u_j)$. 
Further note that for every $1 \leq j \leq \ell$, by the compatibility of $\Xmc$ with $\Gmc_j$ at $u_j$, we have 
$\Xmc_{1}(u_j) = {\sf AT}_{\Dmc^{1}_j, \Tmc}^{\vdash}(u_j)$ where $\Dmc^{1}_j =  \Gmc_j \cup \Xmc_{0}(u_j)$. 
Let $\Dmc^1 = \bigcup_{1 \leq j \leq \ell} \Dmc^1_j$. Since $\mn{Ind}(\Dmc^{1}_j)=\mn{Ind}(\Gmc_j)$ and we know that $\mn{Ind}(\Gmc_j) \cap \mn{Ind}(\Gmc_j')=\emptyset$ for every $j \neq j'$, it follows that $ {\sf AT}_{\Dmc^{1}_j, \Tmc}^{\vdash}(u_j)= {\sf AT}_{\Dmc^{1}, \Tmc}^{\vdash}(u_j)$, and hence that 
$\Xmc_1= {\sf AT}_{\Dmc^{1}, \Tmc}^{\vdash}(\Umc)$. Finally, we note that since $(\Amc^*)^\downarrow_{u_j}=\Gmc_j$, 
we have that $(\Amc^*)^\downarrow_\Umc = \bigcup_{1 \leq j \leq \ell} \Gmc_j$, and thus $\Dmc^1= (\Amc^*)^\downarrow_\Umc \cup \Xmc_0$, 
which shows that $\Xmc_1$ is as desired. 

Next consider an index $1 < 2i+2 \leq N$. As we know that $\Xmc$ is compatible with $(\Amc, \Bmc)$ w.r.t.\ $(\Umc, \Tmc)$, we can infer that
$\Xmc_{2i+2}=  {\sf AT}_{\Dmc^{2i+2},\Tmc}^{\vdash}(\Umc)$ where $\Dmc^{2i+2}= \Amc  \cup\Xmc_{2i+1}$.
Since $(\Amc^*)^\uparrow_\Umc= \Amc$, it follows that $\Xmc_{2i+2}$ is the correct $2i+2$th element in the transfer sequence for 
$(\Amc^*, \Umc)$. 

Finally consider an index $0 < 2i+1 \leq N$. We know that for every $1 \leq j \leq \ell$, the ABox 
$\Gmc_j$ is compatible with $\Xmc$ at $u_j$ w.r.t.\ $\Tmc$, so we have $\Xmc_{2i+1}(u)= {\sf AT}_{\Dmc^{2i+1}_j,\Tmc}^{\vdash}(u)$, where 
  $\Dmc^{2i+1}_j= \Gmc_j \cup  \Xmc_{2i}(u)$. 
Let $\Dmc^{2i+1} = \bigcup_{1 \leq j \leq \ell} \Dmc^{2i+1}_j$. Using the same arguments as for $\Xmc_1$, we can show that 
$\Xmc_{2i+1}= {\sf AT}_{\Dmc^{2i+1}, \Tmc}^{\vdash}(\Umc)$ and that $\Dmc^{2i+1}= (\Amc^*)^\downarrow_\Umc \cup \Xmc_{2i}$, as required. 
  
\smallskip

Now we show the third point. For the right-to-left direction, we note that since $\Xmc$ 
is compatible with ($\Amc$, $\Bmc$) 
w.r.t.\ $(\Umc, \Tmc)$,  we have $\Tmc, \Amc \cup \Xmc_{N} \models B(a)$ for every $B(a) \in \Bmc$. 
Since $\Xmc$ is the transfer sequence for $\Amc^*$  w.r.t.\ $\Tmc$, 
we must have $\Amc^*, \Tmc \models \Xmc_N$, and by definition, $\Amc^*$ contains $\Amc$. 
It follows that $\Tmc, \Amc^* \models B(a)$ for every $B(a) \in \Bmc$. 

For the left-to-right direction, suppose that $\Tmc, \Amc^* \models B(a)$, where 
$a \in \mn{Ind}(\Amc)$ and $B \in \NC$. Then $B(a) \in (\Amc^*)^c_\Tmc$. 
As $\Xmc$ is the transfer sequence for $(\Amc^*, \Umc)$, it follows from Lemma \ref{gen-ts-lemma} that 
we have $(\Amc^*)^c_\Tmc = (\Dmc_{N-1})^c_\Tmc \cup (\Dmc_{N})^c_\Tmc$
where $\Dmc_{N-1}= (\Amc^*)^\downarrow_\Umc \cup \Xmc_{N-1}$ and $\Dmc_{N}= (\Amc^*)^\uparrow_\Umc \cup \Xmc_{N}$.
First suppose that  $B(a) \in (\Dmc_{N})^c_\Tmc$. 
Since $(\Amc^*)^\uparrow_\Umc = \Amc$, we have $\Dmc_{N}= \Amc \cup \Xmc_{N-1}$, so $\Tmc, \Amc \cup \Xmc_{N-1} \models B(a)$.
As $\Xmc_{N-1} \subseteq \Xmc_N$, we also have $\Tmc, \Amc \cup \Xmc_{N-1} \models B(a)$, 
which implies $B(a) \in \Bmc$, due to the fourth condition of compatibility of $\Xmc$ with ($\Amc$, $\Bmc$) 
w.r.t.\ $(\Umc, \Tmc)$. 
Now consider the case in which $B(a) \in (\Dmc_{N-1})^c_\Tmc$, but $B(a) \not \in (\Dmc_{N})^c_\Tmc$.
Since $\Dmc_{N-1}= (\Amc^*)^\downarrow_\Umc \cup \Xmc_{N-2}$, we must have $a \in \Umc$,
and from $B(a) \in (\Dmc_{N-1})^c_\Tmc$, we obtain $B(a) \in \Xmc_{N-1}$.
It follows that $B(a) \in \Dmc^N$, which contradicts our assumption that $B(a) \not \in (\Dmc_{N})^c_\Tmc$.
\end{proof}

\subsection{Upper bound for FO-rewritability}
We aim to prove the following:

\begin{theorem}
FO-rewritability in ($\mathcal{ELIHF}_{\bot}$, rCQ) is in co\textsc{NExpTime}. 
\end{theorem}

By Lemma \ref{lem:part1}, we know that $(\Tmc,\Sigma, q(\vec{x}))$ is \emph{not} FO-rewritable
iff there exists a $k_{0}$-entailment witness for \Tmc, $\Sigma$, and 
$q(\vec{x})$ of outdegree bounded by $|\Tmc|$ for $k_{0}= |q|+2^{3 m^2}$ where $m=|\Tmc|$.
Thus, it suffices to provide an \NExpTime\ procedure for deciding whether such a witness exists. 

The procedure will be quite similar to the \NExpTime\ procedure for testing non-containment
of rooted queries. In what follows, we outline the main differences. 

In Step 1, the  guessed ABox $\Amc_\mn{init}$ corresponds to initial portion of the $k_{0}$-entailment witness (up to depth $|q|$), 
the tuple $\vec{a}$ is the answer tuple associated with the witness, 
and we take $\Umc_q$ to be the set of individuals that occur in $\Amc_\mn{init}$ at depth $|q|$. 
In place of the ABoxes $\Bmc_1$ and $\Bmc_2$,
we guess two ABoxes $\cabox$ and $\ckabox$, with the former being used for the concept assertions 
involving the individuals in $\Amc_\mn{init}$ that hold in the full entailment witness (i.e. once we have added back the missing trees), 
and the latter containing only those assertions that can be obtained using the entailment witness cut off at depth $k_0$. 
We also guess two candidate transfer sequences $\mathcal{Y}=\mathcal{Y}_0, \ldots, \mathcal{Y}_{N}$ and $\Zmc=\mathcal{Z}_0, \ldots, \mathcal{Z}_{N}$ (with $N= (|\Umc| \cdot |\mn{sig}(\Tmc)|) +1 $),
both with respect to $(\Umc_q, \Tmc)$. The first sequence $\Ymc$ is intended to track concept entailments w.r.t.\ the full entailment witness, and the second is for 
the ABox obtained by restricting the entailment witness to those individuals that occur at depth $k_0$ or less. 

In Step 2, we test whether $\Tmc, \cabox \models q(\vec{a})$ and $\Tmc, \ckabox \not \models q(\vec{a})$. 
We also verify that the first candidate transfer sequence $\Ymc$ is compatible with $(\Amc_\mn{init}, \cabox)$ and the 
second candidate transfer sequence $\Zmc$ is compatible with $(\Amc_\mn{init}, \ckabox)$. 

In Step 3, for each $u_j \in U_q$, we build an automaton $\Amf_j$ that accepts (encodings of) pseudo tree ABoxes $\Gmc_j$ such that:
\begin{itemize}
\item $\mn{Ind}(\Amc_\mn{init}) \cap \mn{Ind}(\Gmc_j)=\{u_j\}$,
\item $\Gmc_j$ is consistent with $\Tmc$,
\item $\Amc_\mn{init} \cup \Gmc_j$ does not violate any functionality assertion in $\Tmc$,
\item $\Gmc_j$ is compatible with $\Ymc$ at $u_j$ w.r.t.\ $\Tmc$,
\item $\Gmc_j|_{\leq k_0 - |q|}$ is compatible with $\Zmc$ at $u_j$ w.r.t.\ $\Tmc$. 
\end{itemize}
Note that in the last item, we cut off $\Gmc_j$ at depth $k_0 - |q|$ so that when we attach it to $\Amc_\mn{init}$ (in which $u_j$ occurs at depth $|q|$), 
we obtain an ABox having depth $k_0$. 

Using similar arguments as for containment, we can show that the modified procedure 
runs in \NExpTime\ and it returns yes just in the case that $(\Tmc,\Sigma, q(\vec{x}))$ is not FO-rewritable. 

\section{Lower bounds}

\subsection{coNExpTime lower bounds for rooted CQs}
\label{app:conexplower}

%

An \emph{(exponential torus) tiling problem} $P$ is a triple
$(T,H,V)$, where $T = \{0,\dots,k\}$ is a finite set of \emph{tile
  types} and $H,V \subseteq T \times T$ represent the \emph{horizontal
  and vertical matching conditions}. An \emph{initial condition} for
$P$ takes the form $c = (c_0,\dots,c_{n-1}) \in T^n$.  A mapping
$\tau: \{0,\dots,2^{n}-1\} \times \{0,\dots,2^{n}-1\} \to T$ is a
\emph{solution} for $P$ given $c$ if for all $x,y < 2^{n}$, the
following holds (where $\oplus_i$ denotes addition modulo $i$):
  \begin{itemize}
  \item if $\tau(x,y) = t_1$ and $\tau(x \oplus_{2^{n}} 1,y) =
    t_2$, then $(t_1,t_2) \in H$
  \item if $\tau(x,y) = t_1$ and $\tau(x,y \oplus_{2^{n}} 1) =
    t_2$, then $(t_1,t_2) \in V$
  \item $\tau(i,0) = c_{i}$ for all $i < n$.
  \end{itemize}
It is well-known that there exists a tiling problem $P=(T,H,V)$ such
that, given an initial condition $c$, it is \NExpTime-complete to
decide whether there exists a solution for $P$ given $c$. For the
following constructions, we fix such a $P$. 

\smallskip

%
\begin{lemma}
\label{lem:nexpbasic} 
Given an input $c$ for $P$ of length $n$, one can construct in
polynomial time an $\ELI$ TBox~$\Tmc_c$, a rooted CQ $q_c(x)$,
and an ABox signature $\Sigma_c$ such that, for a selected concept
name $A^* \notin \Sigma_c$,
\begin{enumerate}

\item  $P$ has a solution given $c$ iff there is a $\Sigma_c$-ABox \Amc and
  an $a \in \mn{Ind}(\Amc)$ such that $\Amc,\Tmc_c \models A^*(a)$
  and $\Amc,\Tmc_c \not\models q_c(a)$;



\item there is an \ELI-concept $C_{q_c}$ such that $d \in C_{q_c}^\Imc$
  implies $\Imc \models q_c(d)$ for all interpretations \Imc and $d
  \in \Delta^\Imc$;

\item $q_c$ is FO-rewritable relative to $\Tmc_c$ and $\Sigma_c$. 

\end{enumerate}
\end{lemma}

We will now prove the containment and FO-rewritability lower bounds, assuming the previous lemma. 
The proof of the lemma is given in the following subsection. 
\begin{theorem}
  Containment in $(\ELI,\text{rCQ})$ is \coNExpTime-hard. 
\end{theorem}
\begin{proof}
  Let $c$ be an input to $P$, and let $\Tmc_c$, $q_c(x)$, $\Sigma_c$,
  and $A^*$ be as in Lemma~\ref{lem:nexpbasic}. By Condition~1 
  of Lemma~\ref{lem:nexpbasic}, $(\Tmc_c,\Sigma_c,A^*) \not\subseteq
  (\Tmc_c,\Sigma_c,q_c)$ over $\Sigma_c$-ABoxes iff $P$ has a solution
  given $c$.
\end{proof}

\begin{theorem}
  FO-rewritability in $(\ELI,\text{rCQ})$ is \coNExpTime-hard.
\end{theorem}
\begin{proof}
  Let $c$ be an input to $P$, and let $\Tmc_c$, $q_c(x)$, $\Sigma_c$, and
  $A^*$ be as in Lemma~\ref{lem:nexpbasic}. We obtain a TBox \Tmc by
  extending $\Tmc_c$ with the following:
$$
\begin{array}{rcll}
\exists r . A &\sqsubseteq& A \\[1mm]
A \sqcap A^*& \sqsubseteq& C_{q_c} 
\end{array}
$$
where $A$ and $r$ do not occur in $\Tmc_c$ and $q_c$, $A^* \notin
\Sigma_c$ is the concept name from Lemma~\ref{lem:nexpbasic} and
$C_{q_c}$ the concept from Point~2 of that lemma. Set $\Sigma =
\Sigma_c \cup \{ A,r \}$.
It remains to prove the following. 
\\[2mm]
{\bf Claim}. $P$ has a solution given $c$ iff $q_c$ is not
FO-rewritable relative to \Tmc and $\Sigma$.
\\[2mm]
First assume that $P$ has a solution given $c$. By Point~1 of
Lemma~\ref{lem:nexpbasic}, there is a $\Sigma_c$-ABox \Amc and an $a_0
\in \mn{Ind}(\Amc)$ such that $\Amc,\Tmc_c \models A^*(a_0)$ and
$\Amc,\Tmc_c \not\models q_c(a_0)$. Since every \ELI TBox is
unraveling tolerant \citeA{lutz-2012} and by compactness, we can assume
w.l.o.g.\ that \Amc is tree-shaped with root $a_0$. Let $\ell$ be the
depth of~\Amc. For each $k > \ell$, let $\Amc_k$ be the ABox obtained
by extending \Amc with
$$
   r(a_0,a_{1}),\dots, r(a_{k-1},a_{k}), A(a_k)
$$
where $a_k,\dots,a_1$ do not occur in \Amc. Note that $\Amc_k$ is
tree-shaped and of depth at least $k$. Since $\Amc,\Tmc_c \models
A^*(a_0)$, it follows from Point 2 of Lemma~\ref{lem:nexpbasic} that $\Amc_k,\Tmc \models q_c(a_0)$. Now consider the
ABox $\Amc_k|_{\leq k-1}$. We aim to show that
$\Amc_k|_{\leq k-1},\Tmc \not\models q_c(a_0)$ and then to apply
Theorem~~\ref{lem:char} to show that $q_c$ is not FO-rewritable
relative to \Tmc and $\Sigma$. Note that $\Amc_k|_{\leq k-1}$ does not
contain $A$. On such ABoxes, $\Tmc$ can be replaced with $\Tmc_c$
since the left-hand sides of the concept inclusions in \Tmc will never
apply. It thus suffices to show that $\Amc_k|_{k-1},\Tmc_c \not\models
q_c(a_0)$. This follows from $\Amc,\Tmc_c \not\models q_c(a_0)$
and the fact that $r$ (the only symbol in assertions from
$(\Amc_k|_{k-1}) \setminus \Amc$) occurs neither in $\Tmc_c$ nor in~$q_c$.

\smallskip

Now assume that $P$ has no solution given $c$. Let $\widehat q_c(x)$
be an FO-rewriting of $q_c(x)$ relative to $\Tmc_c$ and $\Sigma_c$.
We argue that
%
%
$\widehat q_c(x)$ is also an FO-rewriting of $q_c(x)$ relative to \Tmc
and $\Sigma$.  

First assume $\Amc \models \widehat q_c(a)$ for some
$\Sigma$-ABox~\Amc. Since $\widehat q_c(x)$ uses only symbols from
$\Sigma_c$, this means that $\Amc' \models \widehat q_c(a)$ where
$\Amc'$ is the reduct of $\Amc$ to symbols in $\Sigma_c$. Thus
$\Amc',\Tmc_c \models q_c(a)$, implying $\Amc,\Tmc \models q_c(a)$.

Conversely, assume that $\Amc,\Tmc \models q_c(a)$ for some
$\Sigma$-ABox \Amc and $a \in \mn{Ind}(\Amc)$. Using
  canonical models and the construction of \Tmc, one can show that
  this implies (i)~$\Amc,\Tmc_c \models q_c(a)$ or
  (ii)~$\Amc,\Tmc_c \models A^*(a)$. In Case~(i), we get
$\Amc',\Tmc_c \models q_c(a)$, where $\Amc'$ is the $\Sigma_c$-reduct
of $\Amc$. 
Thus $\Amc' \models \widehat q_c(a)$, which implies
$\Amc \models \widehat q_c(a)$. In Case~(ii), Point~1 of
Lemma~\ref{lem:nexpbasic} yields $\Amc,\Tmc_c \models q_c(a)$ and thus
we can proceed as in Case~(i).
%
%
%
\end{proof}

%
%
%
%

\subsection{Proof of Lemma~\ref{lem:nexpbasic}}
\label{app:2explower}

Let $c=(c_0,\dots,c_{n-1})$ be an input for $P$.
 We show how to
construct the TBox $\Tmc_c$, query $q_c$, and ABox signature
$\Sigma_c$ that satisfy Points~1 and~2 of
Lemma~\ref{lem:nexpbasic}. We will first use a UCQ for $q_c$ and later
show how to improve to a CQ.  The general idea is that $\Tmc_c$
verifies in a bottom-up way the existence of (a homomorphic image of)
what we call a \emph{torus tree} in the ABox. A torus tree represents
the $2^n \times 2^n$-torus along with a tiling that respects the
tiling conditions in $P$ and initial condition~c, except that the
representation might be defective in that there can be different
elements which represent the same grid node but are labeled with
different tile types. If a torus tree is found, then $\Tmc_c$
ensures that $A^*$ is derived at the root of the tree.  The query
$q_c$ will be constructed to become true at the root if and only if
the torus tree has a defect. It can then be verified that
\begin{itemize}

\item[($*$)] there is a solution for $P$ given $c$ if and only if
  there is an ABox \Amc and an $a \in \mn{Ind}(\Amc)$ with
  $\Amc,\Tmc_c \models A^*(a)$ and $\Amc,\Tmc_c \not\models q_c(a)$

\end{itemize}
where intuitively \Amc is a defect-free torus tree with root $a$.

Torus trees are of depth $2n+2$ and all tree edges are labeled with
the role composition $r^-;r$, where $r$ is the only role name used in
the reduction. For readability, we use $S$ to abbreviate $r^-;r$. For
example, $\exists S . C$ stands for $\exists r^- . \exists r .C$. Note
that $S$ behaves like a reflexive-symmetric role.  The ABox signature
$\Sigma_c$ consists of the following symbols:
\begin{enumerate}

\item concept names $A_0,\dots,A_{2n-1}$ and
  $\overline{A}_0,\dots,\overline{A}_{2n-1}$ that serve as bits in the
  binary representation of a number between 0 and $2^{2n}-1$;

\item concept names $T_0,\dots,T_k$ which represent tile types;

\item concept names $H$, $R$, $U$ which stand for ``here'',
``right'', ``up'';

\item concept names $L_0,\dots,L_{2n}$ to identify the levels of torus 
  trees and concept names $F$ and $G$ to identify certain other nodes;

\item the role name $r$ used in the composition $S$.

\end{enumerate}
We refer to numbers between $0$ and $2^{2n}-1$ as a \emph{grid
  position}: in its binary representation, bits $0$ to $n-1$ represent
the horizontal position in the grid and bits $n$ to $2n-1$ the
vertical position.

The next step is to define the TBox $\Tmc_c$. We first give a few more
details about torus trees, illustrated in Figure~\ref{fig:mod}.  There
is binary branching on levels 0 to $2n-1$ and, intuitively, nodes on
levels 0 to $2n$ form the torus tree proper while nodes on levels
$2n+1$ and $2n+2$ form gadgets appended to the tree nodes on level
$2n$.  Such a gadget is highlighted in Figure~\ref{fig:mod}. All nodes
on level $2n+2$ are labeled with the concept name $G$, all nodes on
level $2n+1$ with the concept name~$F$, and all nodes on levels
$i=0..2n$ with the concept name $L_i$.
\begin{figure}[t!]
  \begin{center}
    \framebox[1\columnwidth]{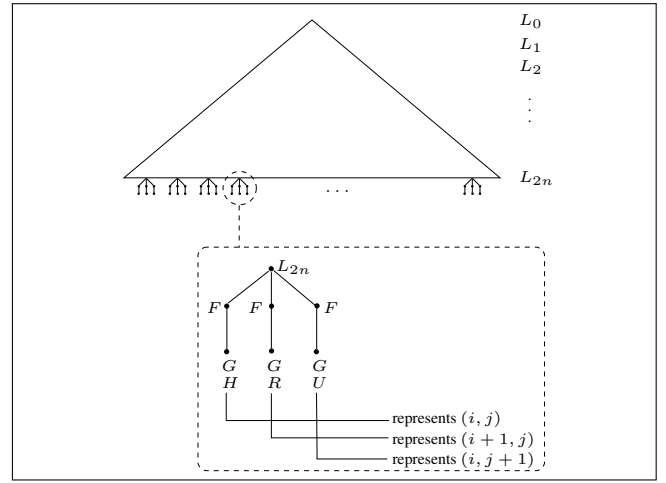}
    \caption{Structure of torus trees.}
    \label{fig:mod}
  \end{center}
\end{figure}
Moreover, every $L_{2n}$-node is associated with a grid position via
the concept names $A_i,\overline{A}_i$ (not shown in the figure). The
$G$-node leaf below it that is labeled $H$ is associated with the same
position as its $L_{2n}$-node ancestor. In contrast, the $R$-leaf is
associated with the neighboring position to the right and the $U$-leaf
with the neighboring position to the top (all via
$A_i,\overline{A}_i$). Every $G$-node is labeled with a tile $T_i$
(not shown in the figure) such that the tiles of $H$- and $R$-nodes in
the same gadget satisfy the horizontal matching condition, and
likewise for the $H$- and $U$-node and the vertical matching
condition. For technical reasons related to the query construction,
the $F$-node is labeled complementarily
regarding the concept names
$A_i,\overline{A}_i$ compared to its $G$-node successor.  Note that,
so far, we have only required that the matching conditions are
satisfied locally in each gadget. To ensure that a torus tree
represents a solution, we will enforce later using the query $q_c$
that whenever two $G$-nodes represent the same position, then they
are labeled with the same tile. 
%




 
We now construct the TBox $\Tmc_c$. The last level of torus trees must
be identified by the concept name $G$. Proper verification of that
level is indicated by the concept name $G\mn{ok}$ (which is not in
$\Sigma_c$):
$$
\begin{array}{c}
  A_i  \sqsubseteq  \mn{ok}_i \qquad
  \overline{A}_i  \sqsubseteq  \mn{ok}_i  \qquad
  T_j  \sqsubseteq  T\mn{ok} \\[1mm]
  \mn{ok}_0 \sqcap \cdots \sqcap \mn{ok}_{2n-1} \sqcap T\mn{ok} \sqcap G \sqsubseteq G\mn{ok}
\end{array}
$$
where $i$ ranges over $0..2n-1$. 
Note that, to receive a $G\mn{ok}$ label, a $G$-node must be labeled
with at least one of $A_i$ and $\overline{A}_i$ for each $i$, and by
at least one concept name of the form $T_j$.
%
%
We next verify $F$-nodes:
$$
\begin{array}{r@{\;}c@{\;}l}
  A_i \sqcap \exists S . (G\mn{ok} \sqcap \overline{A}_i) & \sqsubseteq & \mn{ok}'_i \\[1mm]
  \overline{A}_i \sqcap \exists S . (G\mn{ok} \sqcap A_i) & \sqsubseteq & \mn{ok}'_i \\[1mm]
  \mn{ok}'_0 \sqcap \cdots \sqcap \mn{ok}'_{2n-1} \sqcap F
&\sqsubseteq& F\mn{ok}  
\end{array}
$$
where $i$ ranges over $0..2n-1$. 
Note that we
have not yet guaranteed that $G$-nodes make true at most one of $A_i$
and $\overline{A}_i$ for each $i$. 
Moreover, the first two lines may speak about different
$S$-successors. It is thus not clear that they achieve the intended
complementary labeling.  We fix these problems by adding the following
inclusion:
%
$$
\begin{array}{r@{\;}c@{\;}l}
  \exists S^{2n+1}.(\exists S.(G \sqcap A_i) \sqcap \exists S.(G \sqcap \overline{A_i}) )
  & \sqsubseteq& C_{q_c} 
\end{array}
$$
where $i$ ranges over $0..2n-1$, 
$\exists S^\ell . C$ denotes $\ell$-fold quantification
$\exists S . \cdots \exists S . C$, and $C_{q_c}$ is an \ELI-concept
to be defined later that will satisfy Point~2 of
Lemma~\ref{lem:nexpbasic}, that is, make the query $q_c$ true at the
root of the torus tree.  To understand this, assume for
example that there is a $G$-node labeled with both $A_0$ and
$\overline{A}_0$. Then $C_{q_c}$ will be made true at the root of the
torus tree and thus the ABox is ruled out as a witness in ($*$) above. 

\smallskip



We now verify the existence of level $2n$ of the tree, identified by
the concept name $L_{2n}$. Each $L_{2n}$-node needs to have three
$S$-successors, all of them $F$-nodes, labeled with $H, R, U$,
respectively.  Moreover, it must be labeled with $A_i,\overline{A}_i$
to represent the same grid position as the $H$-leaf below:
$$
\begin{array}{r@{\;}c@{\;}l}
  A_i \sqcap \exists S . (F\mn{ok} \ \sqcap \exists S . (G\mn{ok} \sqcap H \sqcap
  A_i)) & \sqsubseteq & \mn{ok}''_i \\[1mm]
  \overline{A}_i \sqcap \exists S . (F\mn{ok} \sqcap \exists S . (G\mn{ok} \sqcap H \sqcap \overline{A}_i )) & \sqsubseteq & \mn{ok}''_i \\[1mm]
  \exists S . (F\mn{ok} \sqcap \exists S . (G\mn{ok} \sqcap R)) & \sqsubseteq & R\mn{ok} \\[1mm]
  \exists S . (F\mn{ok} \sqcap \exists S . (G\mn{ok} \sqcap U)) & \sqsubseteq & U\mn{ok} \\[1mm]
   \exists S^{2n}.(\exists S^2.(G \sqcap X \sqcap 
   A_i)  \sqcap 
\exists S^2.(G
  \sqcap X \sqcap \overline{A}_i))
  & \sqsubseteq& C_{q_c} 
\end{array}
$$
where $i$ ranges over $0..2n-1$  and
$X$ ranges over $H,R,U$.  The last inclusion makes
such that all $H$-leaves below a $L_{2n}$-node have
the same labeling regarding $A_i,\overline{A}_i$, and 
likewise for all $R$-leaves and all $U$-leaves. We next verify
that the  the grid positions of the $H,R,U$-nodes
below a level $2n$-node relate in the intended way. We start
with copying up the grid positions from the $R$-leaf and the
$U$-leaf, for convenience:
$$
\begin{array}{r@{\;}c@{\;}l}
   L_{2n} \sqcap \exists S^2 . (G \sqcap X \sqcap A_i) & \sqsubseteq & A^X_i \\[1mm]
   L_{2n} \sqcap \exists S^2 . ((G \sqcap X \sqcap \overline{A}_i )) & \sqsubseteq & \overline{A}^X_i 
\end{array}
$$
where $i$ ranges over $0..2n-1$  and
$X$ ranges over $R,U$.
%
%
The following inclusions are then used to verify that the horizontal
component of the $R$-node is incremented compared to the $H$-node:
$$
\begin{array}{r@{\;}c@{\;}l}
   A_0 \sqcap \cdots \sqcap A_{i-1} \sqcap \overline{A}_i \sqcap 
   A^R_i &\sqsubseteq& \mn{ok}_{HRi}
\\[1mm]
   A_0 \sqcap \cdots \sqcap A_{i-1} \sqcap A_i \sqcap 
   \overline{A}^R_i &\sqsubseteq& \mn{ok}_{HRi}
\\[1mm]
  \overline{A}_j \sqcap \overline{A}_i \sqcap \overline{A}^R_i 
  &\sqsubseteq& \mn{ok}_{HRi} 
\\[1mm]
  \overline{A}_j \sqcap A_i \sqcap A^R_i 
  &\sqsubseteq& \mn{ok}_{HRi} 
\end{array}
$$
where $i$ ranges over $0..n$ and $j$ over $0..i$. We can use similar inclusions
setting concept names $\mn{ok}_{HRn},\dots,\mn{ok}_{HR2n-1}$ when the
vertical component of the $R$-node is identical to that of the
$H$-node, concept names $\mn{ok}_{HU0},\dots,\mn{ok}_{HUn-1}$ when the
horizontal component of the $U$-node is identical to that of the
$H$-node, and concept names $\mn{ok}_{HUn},\dots,\mn{ok}_{HU2n-1}$
when the vertical component of the $U$-node is incremented compared to
the $H$-node.  To make $L_{2n}$ true, which identifies level $2n$-nodes, we require that all checks
succeeded:
$$
\begin{array}{r@{\,}c@{\,}l}
    \mn{ok}_{HR0} \sqcap \cdots \sqcap     \mn{ok}_{HRm-1} \, \sqcap 
    \\[1mm]
    \mn{ok}_{HU0} \sqcap \cdots \sqcap     \mn{ok}_{HUm-1} \, \sqcap
    \\[1mm]
    \mn{ok}''_0 \sqcap \cdots \sqcap \mn{ok}''_{2n-1} \, \sqcap
    \\[1mm]
    R\mn{ok}
    \sqcap U\mn{ok} \sqcap
    L_{2n} &\sqsubseteq& L_{2n}\mn{ok}.
\end{array}
$$
To locally ensure the tiling conditions at $L_{2n}$-nodes, we put for all
$(i,j) \notin H$ and all $(i,\ell) \notin V$:
$$
\begin{array}{r@{\,}c@{\,}l}
  \exists S^{2n}. (\exists S^2 . (H \sqcap T_i)  \sqcap \exists S^2 
  . (R \sqcap T_j)) &\sqsubseteq& C_{q_c} \\[1mm]
  \exists S^{2n}. (\exists S^2 . (H \sqcap T_i)  \sqcap \exists S^2 
  . (U \sqcap T_\ell)) &\sqsubseteq& C_{q_c}.
\end{array}
$$
We next verify the existence of levels $2n-1$ to 0 of the tree. To
make sure that the required successors are present on all levels, we
branch on the concept names $A_i$, $\overline{A}_i$ at level $i$ and
for all $j< i$, keep our choice of $A_j$, $\overline{A}_j$:
%
%
$$
%
\begin{array}{r@{\;}c@{\;}l}
   \exists S . (L_{i+1}\mn{ok} \sqcap A_{i}) \sqcap 
   \exists S . (L_{i+1}\mn{ok} \sqcap \overline{A}_{i}) & \sqsubseteq & \mn{succ}_i \\[1mm]
   A_j \sqcap \exists S . (L_{i+1}\mn{ok} \sqcap A_j) &\sqsubseteq& \mn{ok}_{i,j} \\[1mm]
   \overline{A}_j \sqcap \exists S . (L_{i+1}\mn{ok} \sqcap \overline{A}_j) &\sqsubseteq&
   \mn{ok}_{i,j} \\[1mm]
   \mn{succ}_i \sqcap \mn{ok}_{i,0} \sqcap \cdots \sqcap
   \mn{ok}_{i,i-1} \sqcap L_i
   &\sqsubseteq& L_{i}\mn{ok} \\[1mm]
  \exists S^i . (\exists S.(L_{i+1} \sqcap A_j) \sqcap \exists S.(L_{i+1} \sqcap \overline{A_j})) 
  & \sqsubseteq& C_{q_c}
\end{array}
$$
%
%
where $i$ ranges over $0..2n-1$ and $j$ over $0..i-1$.  
%
The initial condition is verified at the $G$-nodes. Put
$$
  \exists S^{2n+2} . (\overline{A}_0 \sqcap \cdots \sqcap \overline{A}_{2n-1} 
  \sqcap T_i) \sqsubseteq C_{q_c}
$$
for all $i \in \{0,\dots,k\}$ with $i \neq c_0$, and similarly
for the grid positions $(1,0),\dots,(n-1,0)$.

\smallskip 

We next define the query $q_c$ to ensure that all $G$-nodes that are
associated with the same grid position are labeled with the same tile
type. A bit more verbosely, we have to guarantee that
\begin{itemize}

\item[($*$)] if $a$ and $b$ are $G$-nodes labeled identically
  regarding the concept names $A_i,\overline{A}_i$, then there are
  no distinct tile types $k,j$ such that $a$ is labeled with $T_k$ and
  $b$ with $T_j$.

\end{itemize}
The UCQ $q_c^{\vee}$ contains one CQ for each choice of tile types $k,j$.
Fix concrete such $k,j$.  We construct the required CQ $q$ from
component queries $p_0,\dots,p_{n-1}$, which all take the form of the
query show on the left-hand side of Figure~\ref{fig:q1}. 
\begin{figure}[t!]
  \begin{center}
    \framebox[1\columnwidth]{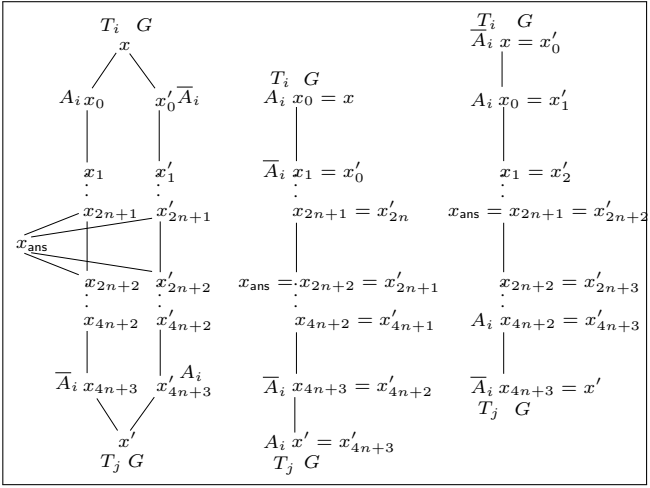}
    \caption{The query $p_i$ (left) and two of its 
      identifications (middle and right).}
    \label{fig:q1}
  \end{center}
\end{figure}
Note that all edges are $S$-edges, the only difference between the
component queries is which concept names $A_i$ and $\overline{A}_i$
are used, and $x_{\mn{ans}}$ is the only answer variable. We assemble
$p_0,\dots,p_{n-1}$ into the desired query $q_c^{\vee}$ by taking variable
disjoint copies of $p_0,\dots,p_{n-1}$ and then identifying (i)~the
$x$-variables of all components and (ii)~the $x'$-variables of all
components.

To see why $q^\vee_c$ achieves~($*$), first note that the variables $x$ and
$x'$ must be mapped to leaves of the torus tree because of their
$G$-label. Call these leaves $a$ and~$a'$. Since $x_0$ and
$x'_0$ are connected to $x$ in the query, both must then be mapped
either to $a$ or to its predecessor; likewise, $x_{4n+3}$ and
$x'_{4n+3}$ must be mapped either to $a'$ or to its predecessor.
Because of the labeling of $a$ and $a'$ and the predecessors in the
torus tree with $A_i$ and $\overline{A}_i$, we are actually even
more constrained: 
exactly one of $x_0$ and $x'_0$ must be mapped to
$a$, and exactly one of $x_{4n+3}$ and $x'_{4n+3}$ to~$a'$. If $x_0$
is mapped to $a$, then $x_{\mn{ans}}$ must be identified with
$x_{2n+2}$ because as an answer variable it has to be mapped to the
root of the tree and the only other option (identifying $x_{\mn{ans}}$
with $x_{2n+1}$) would thus require a path of length $2n+1$ between $a$
and the root. Also for path length reasons, this means that $x_{4n+3}$
must be mapped to the predecessor of $a'$, thus $x'_{4n+3}$ is mapped
to $a'$. Analogously, we can show that mapping $x'_0$ to $a$ requires
mapping $x_{4n+3}$ to $a'$. These two options give rise to the two
variable identifications in each query $p_i$ shown in
Figure~\ref{fig:q1}. Note that the first case implies that $a$ and
$a'$ are both labeled with $A_i$ while they are both labeled with
$\overline{A}_i$ in the second case. In summary, $a$ and $a'$ must
thus agree on all concept names $A_i$, $\overline{A}_i$. Since $a$
must satisfy $T_i$ and $a'$ must satisfy $T_j$ due to the labeling of
$x$ and $x'$, we have achieved ($*$).

We now show how to replace the UCQ $q_c^\vee$ with a single CQ $q_c$. This
requires the following changes:
\begin{enumerate}

\item the $F$-nodes in configuration trees receive additional labels:
  when a $G$-node is labeled with $T_i$, then its predecessor $F$-node
  is labeled with $T_j$ for all $j \neq i$;

\item the query construction is modified.

\end{enumerate}
Point~1 is important for the CQ to be constructed to work correctly
and can be achieved in a straightforward way by modifying $\Tmc_c$,
details are omitted. We thus concentrate on Point~2. The desired CQ
$q_c$ is again constructed from component queries. We use $n$ components
as shown in Figure~\ref{fig:q1}, except that the $T_i$- and
$T_j$-labels are dropped. We further add the component shown in
Figure~\ref{fig:q2} where again $x$ and $x'$ are the variables shared
with the other components, and where we assume for simplicity
that $T=\{0,1,2\}$; the generalization to an unrestricted number of
tile types is straightforward, see \citeA{Lutz-DL-07}. 
\begin{figure}[t!]
  \begin{center}
    \framebox[1\columnwidth]{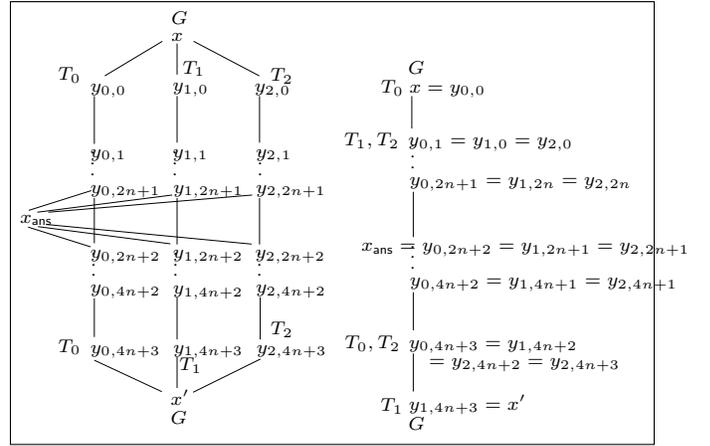}
    \caption{The query $q_\mn{tile}$ (left) and one of its  
     identifications (right).}
    \label{fig:q2}
  \end{center}
\end{figure} 
The additional component can be understood essentially in the same way
as the previous query components.
\begin{lemma}
  $\Tmc_c$, $q_c$, and $\Sigma_c$ satisfy Points~1 and~2 from
  Lemma~\ref{lem:nexpbasic} when choosing $A^* = L_0\mn{ok}$.
\end{lemma}
\begin{proof}(sketch) We show the following:
  \begin{enumerate}

  \item If $P$ has a solution given $c$, then there is a
    $\Sigma_c$-ABox \Amc and an $a \in \mn{Ind}(\Amc)$ such that
    $\Amc,\Tmc_c \models A^*(a)$ and $\Amc,\Tmc_c \not\models q_c(a)$.

  \item If $P$ has no solution given $c$, then for any $\Sigma_c$-ABox 
    \Amc and $a \in \mn{Ind}(\Amc)$, $\Amc,\Tmc_c \models A^*(a)$
    implies $\Amc,\Tmc_c \models q_c(a)$. 

\item There is an \ELI-concept $C_{q_c}$ such that $d \in C^\Imc$
  implies $\Imc \models q_c(d)$. 

\item $q_c$ is FO-rewritable relative to $\Tmc_c$ and $\Sigma_c$. 

\end{enumerate}
(1) Take as \Amc a torus tree that encodes a solution for $P$
given~$c$ (viewed as an ABox) and let $a$ be the root of the tree. The
verification of torus trees by $\Tmc_c$ yields $\Amc,\Tmc_c \models
L_0\mn{ok}(a)$. Since the torus tree is not defective, we have
$\Amc,\Tmc_c \not\models q_c(a)$.

\smallskip
\noindent  
 (2)    Since the verification of (homomorphic images of) torus trees by
    $\Tmc_c$ is sound, $\Amc,\Tmc_c \models L_0\mn{ok}(a)$ implies that \Amc
    contains a homomorphic image of a torus tree whose root is
    identified by $a$. Since there is no solution for $P$ given $c$,
    that tree must be defective.  Consequently, $\Amc,\Tmc_c \models
    q_c(a)$.



\smallskip
\noindent 
(3)  Set
  $$
  \begin{array}{r@{\,}c@{\,}l}
  G_1 &=& G \sqcap \overline{A}_0 \sqcap \cdots \sqcap \overline{A}_k \sqcap T_0 \\[1mm]
  G_2 &=& G \sqcap \overline{A}_0 \sqcap \cdots \sqcap \overline{A}_k \sqcap T_1 \\[1mm]
  F_1 &=& A_0 \sqcap \cdots \sqcap A_n \sqcap T_1 \sqcap \cdots \sqcap 
  T_{k} \\[1mm]
  F_2 &=& A_0 \sqcap \cdots \sqcap A_n \sqcap T_0 \sqcap T_2 \sqcap \cdots \sqcap 
  T_{k} \\[1mm]
  C_{q_c} &=& \exists S^{2n+1}.(F_1 \sqcap \exists S . G_1) \sqcap
\exists S^{2n+1}.(F_2 \sqcap \exists S . G_2) 
  \end{array}
  $$
  It can be verified that $C_{q_c}$ is as required. 

\smallskip
\noindent 
(4) Note that whenever $D \sqsubseteq C_{q_c}$ is in $\Tmc_c$, then
$D$ uses symbols from $\Sigma_c$, only. One can construct an 
FO-rewriting of $q_c$ relative to $\Tmc_c$ and $\Sigma_c$ that has
the form 
$$
  q_0(x) \vee \bigvee_{D \sqsubseteq C_{q_c} \in \Tmc_c} q_D(x)
$$
where $q_D$ is $D$ viewed as a CQ. To define $q_0$, let $\Tmc^0_c$ be
the result of removing from $\Tmc_c$ all CIs of the form $D
\sqsubseteq C_{q_c}$. Note that the recursion depth of $\Tmc^0_c$ is
bounded by $2n+1$. We can thus choose
$$
  q_0(x) = \bigvee_{\Amc \in \Amf} q_\Amc(x)
$$
where \Amf is the set of all pseudo tree $\Sigma_c$-ABoxes \Amc of
depth at most ${2n+1}$, width at most $|q_c|$, outdegree at most
$|\Tmc_c|$, and with root $a_0$ such that $\Amc,\Tmc_c \models
q_c[a_0]$ and where $q_\Amc$ is \Amc viewed as a CQ.
\end{proof}

\subsection{2ExpTime lower bounds}

We consider Boolean (connected) CQs. We reduce the word problem of
exponentially space bounded alternating Turing machines (ATMs), see
\citeA{DBLP:journals/jacm/ChandraKS81}.  An \emph{Alternating Turing
  Machine (ATM)} is of the form $M =
(Q,\Sigma,\Gamma,q_0,\Delta)$. The set of \emph{states} $Q = Q_\exists
\uplus Q_\forall \uplus \{q_a\} \uplus \{q_r\}$ consists of
\emph{existential states} in $Q_\exists$, \emph{universal states} in
$Q_\forall$, an \emph{accepting state} $q_a$, and a \emph{rejecting
  state} $q_r$; $\Sigma$ is the \emph{input alphabet} and $\Gamma$ the
\emph{work alphabet} containing a \emph{blank symbol} $\square$ and
satisfying $\Sigma \subseteq \Gamma$; $q_0 \in Q_\exists \cup
Q_\forall$ is the \emph{starting} state; and the \emph{transition
  relation} $\Delta$ is of the form
$$
  \Delta \; \subseteq \; Q \times \Gamma \times Q \times \Gamma
  \times \{ L,R \}.
$$
We write $\Delta(q,\sigma)$ to denote $\{ (q',\sigma',M) \mid
(q,\sigma,q',\sigma',M) \in \Delta \}$ and assume w.l.o.g.\ that
the state $q_0$ cannot be reached by any transition.

A \emph{configuration} of an ATM is a word $wqw'$ with $w,w' \in
\Gamma^*$ and $q \in Q$. The intended meaning is that the one-side
infinite tape contains the word $ww'$ with only blanks behind it,
the machine is in state $q$, and the head is on the symbol just after
$w$. The \emph{successor configurations} of a configuration $wqw'$
are defined in the usual way in terms of the transition relation
$\Delta$. A \emph{halting configuration} is of the form $wqw'$ with $q
\in \{ q_a, q_r \}$.

A \emph{computation} of an ATM $M$ on a word $w$ is a
(finite or infinite) sequence of configurations $K_0,K_1,\dots$ such
that $K_0=q_0w$ and $K_{i+1}$ is a successor configuration of $K_i$
for all $i \geq 0$.  The ATMs considered in the following have only
\emph{finite} computations on any input. Since this case is simpler
than the general one, we define acceptance for ATMs with finite
computations, only.
Let $M$ be such an ATM. A halting configuration is
\emph{accepting} iff it is of the form $wq_aw'$.  For other
configurations $K=w q w'$, acceptance depends on $q$: if $q \in
Q_\exists$, then $K$ is accepting iff at least one successor
configuration is accepting; if $q \in Q_\forall$, then $K$ is
accepting iff all successor configurations are accepting.  Finally,
the ATM $M$ with starting state $q_0$ \emph{accepts} the
input $w$ iff the \emph{initial configuration} $q_0w$ is accepting.
We use $L(M)$ to denote the language accepted by
$M$.

\smallskip

There is an exponentially space bounded ATM $M$ whose word
problem is {\sc 2Exp\-Time}-hard and we may assume that the length of
every computation path of $M$ on $w \in \Sigma^n$ is bounded
by $2^{2^{n}}$, and all the configurations $wqw'$ in such computation
paths satisfy $|ww'| \leq 2^{n}$, see \citeA{DBLP:journals/jacm/ChandraKS81}.  

\smallskip

\begin{lemma}
\label{lem:2expbasic}
Given an input $w$ to $M$, 
one can construct in polynomial time an
$\ELI$ TBox~$\Tmc_w$, a Boolean connected CQ $q_w$, and
an ABox signature $\Sigma_w$ such that, for a selected concept name
$A^* \notin \Sigma_w$,
\begin{enumerate}

\item $M$ accepts $w$ iff there is a $\Sigma_w$-ABox \Amc 
such that $\Amc,\Tmc_w \models 
\exists x
  \, 
  A^*(x)$
  and $\Amc,\Tmc_w \not\models q_w$;
  
\item $M$ accepts $w$ iff there is a $\Sigma_w$-ABox \Amc 
and an $a \in \mn{Ind}(\Amc)$ 
such that $\Amc,\Tmc_w \models 
  A^*(a)$
  and $\Amc,\Tmc_w \not\models q_w$;  


\item $q_w$ is FO-rewritable relative to $\Tmc_w$ and $\Sigma_w$;

\item there is an \ELI-concept $C_{q_w}$ such that $C_{q_w}^\Imc \neq
  \emptyset$ implies $\Imc \models q_w$. 

\end{enumerate}
\end{lemma}

\begin{theorem}
\label{thm:app:cont}
  Containment in $(\ELI,\text{CQ})$ is 2\ExpTime-hard.
\end{theorem}
\begin{proof}
  Let $w$ be an input to $M$, $\Tmc_w$, $q_w$, and $\Sigma_w$ as in
  Lemma~\ref{lem:2expbasic}. By Point~1 of Lemma~\ref{lem:2expbasic},
  $(\Tmc_w,\Sigma_w,\exists x \, A^*(x)) \not\subseteq (\Tmc_w,\Sigma_w,q_w)$ over
  $\Sigma_w$-ABoxes iff $M$ accepts $w$.
\end{proof}

\begin{theorem}
\label{thm:app:FOrewr}
  FO-rewritability in $(\ELI,\text{CQ})$ is 2\ExpTime-hard.
\end{theorem}
\begin{proof}
  Let $w$ be an input to $M$ and $\Tmc_w$, $q_w$, $\Sigma_w$ as in
  Lemma~\ref{lem:2expbasic}. We obtain a TBox \Tmc by extending
  $\Tmc_w$ with the following:
$$
\begin{array}{rcll}
\exists r . A &\sqsubseteq& A \\[1mm]
A \sqcap B \sqcap A^*& \sqsubseteq& C_{q_w} 
\end{array}
$$
where $A$, $B$, and $r$ do not occur in $\Tmc_w$ and $q_w$, $A^*
\notin \Sigma_w$ is the concept name from Lemma~\ref{lem:2expbasic}
and $C_{q_w}$ the concept from Point~4 of that lemma. Set $\Sigma =
\Sigma_w \cup \{ A,B,r \}$.
It remains to prove the following. 
\\[2mm]
{\bf Claim}. $M$ accepts $w$ iff $q_w$ is not
FO-rewritable relative to \Tmc and $\Sigma$.
\\[2mm]
First assume that $M$ accepts $w$. By Point~2 of
Lemma~\ref{lem:2expbasic}, there is a $\Sigma_w$-ABox \Amc 
and $a_0 \in \mn{Ind}(\Amc)$ such that 
$\Amc,\Tmc_w \models A^*(a_0)$ and 
$\Amc,\Tmc_w \not\models q_w$.  
Since every \ELI TBox is
unraveling tolerant \citeA{lutz-2012} and by compactness, we can assume
w.l.o.g.\ that \Amc is tree-shaped with root $a_0$. 
Let $\ell$ be
the depth of~\Amc. For each $k > \ell$, let $\Amc_k$ be the ABox
obtained by extending \Amc with
$$
   B(a_0), r(a_0,a_{1}) ,\dots,r(a_{k-1},a_{k}), A(a_k)
$$
where $a_k,\dots,a_1$ do not occur in \Amc. Note that $\Amc_k$ is
tree-shaped and of depth at least $k$. Since $\Amc,\Tmc_w \models
A^*(a_0)$, we have $\Amc_k,\Tmc \models C_{q_w}(a_0)$. Applying Point
4 of Lemma~\ref{lem:2expbasic}, we obtain $\Amc_k,\Tmc \models
q_w$. To prove that $q_w$ is not FO-rewritable relative to \Tmc and
$\Sigma$, by Theorem~\ref{lem:char} it suffices to show that
$\Amc|_{>0},\Tmc \not\models q_w$ and $\Amc_k|_{\leq k-1},\Tmc
\not\models q_w$. Note that neither $\Amc_k|_{>0}$ nor $\Amc_k|_{\leq
  k-1}$ contains an $r$-path from an individual satisfying $B$ to an
individual satisfying $A$. On such ABoxes, $\Tmc$ can be replaced with
$\Tmc_w$ since the left-hand sides of the second additional concept
inclusions in \Tmc will never apply, and that concept inclusion is the
only (additional) one whose right-hand side contains symbols from
$\Tmc$ and $q_w$. It thus suffices to show that $\Amc_k|_{>0},\Tmc_w
\not\models q_w$ and $\Amc_k|_{k-1},\Tmc_w \not\models q_w$.  This
follows from $\Amc,\Tmc_w \not\models q_w$ and the fact that $A$, $B$,
and $r$ (the only symbols in assertions from $(\Amc_k|_{>0}) \setminus
\Amc$ and $(\Amc_k|_{k-1}) \setminus \Amc$) occur neither in $\Tmc_w$
nor in~$q_w$.

\smallskip

Now assume that $M$ does not accept $w$. By Point~3 of
Lemma~\ref{lem:2expbasic}, there is an FO-rewriting $\widehat q_w$ of
$q_w$ relative to $\Tmc_w$ and~$\Sigma_w$.
We argue that
%
%
$\widehat q_w$ is also an FO-rewriting of $q_w$ relative to \Tmc and
$\Sigma$.  

First assume $\Amc \models \widehat q_w$ for some
$\Sigma$-ABox~\Amc. Since $\widehat q_w$ uses only symbols from
$\Sigma_w$, this means that $\Amc' \models \widehat q_w$ where $\Amc'$
is the reduct of $\Amc$ to symbols in $\Sigma_w$. Thus
$\Amc',\Tmc_w \models q_w$, implying $\Amc,\Tmc \models q_w$.

Conversely, assume that $\Amc,\Tmc \models q_w$ for some $\Sigma$-ABox
\Amc. Using canonical models and the construction of
  \Tmc, one can show that this implies (i)~$\Amc,\Tmc_w \models q_w$
  or (ii)~$\Amc,\Tmc_w \models \exists x \, A^*(x)$.  In Case~(i), we
get $\Amc',\Tmc_w \models q_w$, where $\Amc'$ is the $\Sigma_w$-reduct
of $\Amc$. 
Thus $\Amc' \models \widehat q_w$, which implies
$\Amc \models \widehat q_w$. In Case~(ii), Point~1 of
Lemma~\ref{lem:2expbasic} yields $\Amc,\Tmc_w \models q_w$ and thus we
can proceed as in Case~(i).
\end{proof}

\subsection{Proof of Lemma~\ref{lem:2expbasic}}

Let $w=\sigma_0\cdots \sigma_{m-1} \in \Sigma^*$ be an input to
$M$.  We show how to construct a TBox $\Tmc_w$, query $q_w$,
and ABox signature $\Sigma_w$ that satisfy Points~1 to~4 of
Lemma~\ref{lem:2expbasic}. 
We first use a UCQ for $q_w$, which results
in a simpler reduction, and in a second step show how to replace the
UCQ with a CQ.

In the reduction, we represent each configuration of a computation of
$M$ by the leaves of a \emph{configuration tree} that has depth $n+2$
and whose edges are represented by the role composition $S=r^-;r$,
similarly to the representation of the $2^{n} \times 2^{n}$-torus in
the previous reduction. The trees representing configurations are then
interconnected to a \emph{computation tree} which represents the
computation of $M$ on $w$. This is illustrated in
Figure~\ref{fig:mod2}, where
\begin{figure}[t]
  \begin{center}
    \framebox[1\columnwidth]{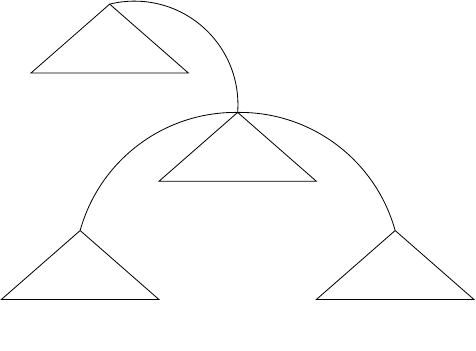}
    \caption{Representing ATM computations.}
    \label{fig:mod2}
  \end{center}
\end{figure}
the tree $T_1$ represents an existential configuration and thus has
only one successor tree $T_2$, connected via the same role
composition $S$ that is also used inside configuration trees. In contrast,
$T_2$ represents a universal configuration with two successor
configurations $T_3$ and~$T_4$. 

The above description is actually an oversimplification. In fact,
every configuration tree stores two configurations instead of only
one: the current configuration and the previous configuration in the
computation. The query $q_w$ to be defined later on makes sure that
the previous configuration stored in a configuration tree is identical
to the current configuration stored in its predecessor configuration
tree. The actual transitions of $M$ are then represented locally
inside configuration trees. This is illustrated by a sequence of
existential configurations in Figure~\ref{fig:tree} where each $C_i$
represents a stored configuration, ``\mn{step}'' denotes a transition of
$M$, and ``$=$'' denotes identity of stored configurations.

Since the role composition $S$ used to connect configuration trees is
symmetric, it is difficult to distinguish predecessor configuration
trees from successor configuration trees. To break this symmetry, we
represent the current and next configuration stored in configuration
trees using six different sets of concept names. This is also
indicated in Figure~\ref{fig:tree} where $C_i$ means that we use
the $i$-th set of concept names for representing the stored 
configuration.
\begin{figure}[t]
  \begin{center}
    \framebox[1\columnwidth]{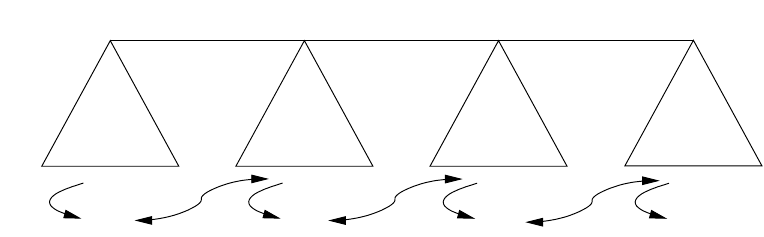}
    \caption{Representing ATM computations.}
    \label{fig:tree}
  \end{center}
\end{figure}

We next construct the TBox $\Tmc_w$, which is used to verify the
existence of an accepting computation tree of $M$ on input $w$ in the
ABox, apart from the described copying of stored configurations which
will be achieved by the query $q_w$ later on.  The ABox signature
$\Sigma_w$ consists of the following symbols:
\begin{enumerate}

\item concept names $A_0,\dots,A_{n-1}$ and
  $\overline{A}_0,\dots,\overline{A}_{n-1}$ that serve as bits in the
  binary representation of a number between 0 and $2^{n}-1$,
  identifying the position of tape cells (that is, leaves in
  configuration trees);

\item for each $\sigma \in \Gamma$, the concept names $A^i_\sigma$, $1 \leq i \leq 6$;

\item for each $\sigma \in \Gamma$ and $q \in Q$, the concept names
  $A^i_{q,\sigma}$, $1 \leq i \leq 6$; 

\item the concept names $\overline{H}$, $W$, and $\overline{W}$
 which stand for ``cell without head'', ``cell being written to reach
current configuration'', and ``cell not being written to reach current
configuration'';


\item a concept name $A_{q,\sigma,M}$ for each $q \in Q$, $\sigma \in
  \Gamma$, and $M \in \{L,R\}$ to describe transitions of $M$;

\item a concept name $I$ that marks the initial configuration;

\item concept names $L_0,\dots,L_{n}$ to identify the levels of
  configuration trees and concept names $F_1,F_2,G_1,G_2$ to identify
  certain other nodes;

\item the role name $r$ used in the composition $S$.

\end{enumerate}
The concept names $A^i_\sigma$ are used to represent the symbols on the tape that
are currently not under the head and $A^i_{q,\sigma}$ to mark tape
cells under the head, indicating the head position, the current state,
and the symbol under the head.

We start with verifying single configuration trees. Such trees come in
three different types, depending on the set of concept names that we
use to represent the current and previous configuration stored. This
is shown in Figure~\ref{fig:tree}. Type~0 means that the previous
configuration is represented by concept names of the form $A^1_\sigma$
and $A^1_{q,\sigma}$ and the current configuration by concept names
$A^2_\sigma$ and $A^2_{q,\sigma}$, type~1 uses $A^3_\sigma$ and
$A^3_{q,\sigma}$ for the previous configuration, and so on. We start
with verifying configuration trees of type~0. Intuitively, nodes on
levels 0 to $n$ form the configuration tree proper while nodes on
levels $n+1$ and $n+2$ form gadgets appended to the tree nodes on
level $n$, similarly to what is shown in Figure~\ref{fig:mod}. We
identify each node on level $n+1$ with one of the concept names $F_1$,
$F_2$ and each node on level $n+2$ with one of the concept names
$G_1,G_2$. In contrast to Figure~\ref{fig:mod}, there are only two
nodes below each level $n$ node, one labeled $F_1$ and one labeled
$F_2$. Moreover, every $F_\ell$ node must have a $G_\ell$-node
successor.  Each $G_\ell$-node represents a tape cell, and the
position of that cell is encoded in binary by the concept names
$A_i,\overline{A}_i$. The $G_1$- and $G_2$-node below the same level
$n$ node must both have the same position and, for similar reasons as
in the previous reduction, the $F_\ell$ nodes in between receive a
complementary labeling regarding these concept names and also
regarding the concept names $A^\ell_\sigma, A^\ell_{q,\sigma}$.  At
$G_1$-nodes, the concept names $A^1_\sigma$ and $A^1_{q,\sigma}$ are
used to store information and at $G_2$-nodes, we use the concept names
$A^2_\sigma$ and $A^2_{q,\sigma}$. Thus, $G_1$-nodes represent the
previous configuration while $G_2$-nodes representing the current
configutation. The concept names $\overline{H}, W, \overline{W}$ are
used for the current configuration, only.

The verification of configuration trees is again bottom-up, starting
at level $n+2$ nodes:
$$
\begin{array}{c}
  A_i  \sqsubseteq  \mn{ok}_i \qquad
  \overline{A}_i  \sqsubseteq  \mn{ok}_i  \qquad
  A^1_\sigma \sqsubseteq  \Gamma\mn{ok}_1 \quad 
  A^1_{q,\sigma} \sqsubseteq \Gamma\mn{ok}_1 
\\[1mm]
  \mn{ok}_0 \sqcap \cdots \sqcap \mn{ok}_{n-1} \sqcap \Gamma\mn{ok}_1 \sqcap G_1
 \sqsubseteq G_1\mn{ok} \\[1mm]
  A^2_\sigma \sqcap \overline{H} \sqcap \overline{W} \sqsubseteq  \Gamma\mn{ok}_2 \quad 
  A^2_\sigma \sqcap \overline{H} \sqcap W \sqsubseteq  \Gamma\mn{ok}_2 \\[1mm]
  A^2_{q,\sigma} \sqcap \overline{W} \sqsubseteq \Gamma\mn{ok}_2
\\[1mm]
  \mn{ok}_0 \sqcap \cdots \sqcap \mn{ok}_{n-1} \sqcap \Gamma\mn{ok}_2 \sqcap G_2
 \sqsubseteq G_2\mn{ok}
\end{array}
$$
where $i$ ranges over $0..n-1$, $q$ over the elements of $Q$ and
$\sigma$ over the elements of $\Gamma$.  Note that every $G_\ell$-node
must be labeled with at least one of $A_i$ and $\overline{A}_i$ for
each $i$ and by at least one concept name of the form $A^\ell_\sigma$
or $A^\ell_{q,\sigma}$. If $\ell=2$, then an $A^\ell_\sigma$-label (as
opposed to an $A^\ell_{q,\sigma}$-label) is acceptable only if there
is also an $\overline{H}$-label. For $\ell=2$, there must also be a
$W$- or $\overline{W}$-label, the former only being acceptable if the
head is not on the current cell.
%
%
We now verify $F_\ell$-nodes:
$$
\begin{array}{r@{\;}c@{\;}l}
  A_i \sqcap \exists S . (G_\ell\mn{ok} \sqcap \overline{A}_i) & \sqsubseteq & \mn{ok}_{\ell,i} \\[1mm]
  \overline{A}_i \sqcap \exists S . (G_\ell\mn{ok} \sqcap A_i) & \sqsubseteq & \mn{ok}_{\ell,i} \\[1mm]
  \midsqcap_{\beta \in (\Gamma \cup (Q \times \Gamma)) \setminus \{ \alpha \} } A^\ell_\beta 
\sqcap \exists S . (G_\ell\mn{ok} \sqcap A^\ell_\alpha) & \sqsubseteq & \Gamma\mn{ok}'_{\ell} \\[1mm]
  \mn{ok}_{\ell,0} \sqcap \cdots \sqcap \mn{ok}_{\ell,n-1} \sqcap
  \Gamma\mn{ok}'_\ell \sqcap F_\ell &\sqsubseteq& F_\ell\mn{ok}
\end{array}
$$
where $\ell$ ranges over 1,2, $i$ over $0..n-1$, and $\alpha$ over
$\Gamma \cup (Q \times \Gamma)$.  We have not yet guaranteed that
$G_\ell$-nodes make true at most one of $A_i$ and $\overline{A}_i$ for
each $i$, at most one concept name of the form $A^\ell_\alpha$, and
not simultaneously $W$ and $\overline{W}$, or $\overline{H}$ and a
concept name of the form $A^2_{q,\sigma}$, or $W$ and a concept name
$A^2_{q,\sigma}$.  Moreover, the first three lines may speak about
different $S$-successors. It is thus not clear that they achieve the
intended complementary labeling.  We fix these problems by adding the
following
inclusions:
%
$$
\begin{array}{r@{\;}c@{\;}l}
\exists S.(G_\ell \sqcap A_i) \sqcap \exists S.(G_\ell \sqcap \overline{A_i}) 
  & \sqsubseteq& C_{q_w} \\[1mm]
\exists S.(G_\ell \sqcap A^\ell_\alpha) \sqcap \exists S.(G_\ell \sqcap A^\ell_{\beta})
  & \sqsubseteq& C_{q_w} \\[1mm]
\exists S.(G_\ell \sqcap W) \sqcap \exists S.(G_\ell \sqcap \overline{W}) 
  & \sqsubseteq& C_{q_w} \\[1mm]
\exists S.(G_\ell \sqcap A^2_{q,\sigma}) \sqcap \exists S.(G_\ell \sqcap \overline{H}) 
  & \sqsubseteq& C_{q_w} \\[1mm]
\exists S.(G_\ell \sqcap A^2_{q,\sigma}) \sqcap \exists S.(G_\ell
\sqcap W) 
  & \sqsubseteq& C_{q_w} 
\end{array}
$$
where $\ell$ ranges over $1,2$, $i$ over $0..n-1$, $\alpha,\beta$
take distinct values from $\Gamma \cup (Q \times \Gamma)$, $q$
ranges over $Q$, and $\sigma$ over $\Gamma$. Moreover,
$C_{q_w}$ is an \ELI-concept to be defined later that will satisfy
Point~4 of Lemma~\ref{lem:2expbasic}, that is, make the query $q_w$
true.

\smallskip



We now verify the existence of level $n$ of the tree, identified by
the concept name $L_n$. Nodes here need to have $S$-successors in
$F_1$ and $F_2$ that are again labeled complementarily regarding the
concept names $A_i,\overline{A}_i$ (in other words, the $L_n$ node
agrees with the labeling of the $G_1$- and $G_2$-node leaves below
it):
%
\begin{align*} 
  A_i \sqcap \exists S . (F_\ell\mn{ok} \sqcap \overline{A}_i) & \sqsubseteq  
  \mn{ok}'_{\ell,i} \\
  \overline{A}_i \sqcap \exists S . (F_\ell\mn{ok} \sqcap A_i) & \sqsubseteq  
  \mn{ok}'_{\ell,i} \\
  \mn{ok}'_{1,0} \sqcap \cdots \sqcap \mn{ok}'_{1,n-1} \sqcap 
  \mn{ok}'_{2,0} \sqcap \cdots \sqcap \mn{ok}'_{2,n-1} \sqcap L_n 
&\sqsubseteq L_n\mn{ok}   \\
\exists S.(F_\ell \sqcap A_i) \sqcap \exists S.(F_\ell \sqcap \overline{A}_i) 
  & \sqsubseteq C_{q_w} \\
\exists S.(F_\ell \sqcap A^\ell_\alpha) \sqcap \exists S.(F_\ell \sqcap 
A^\ell_{\beta})
  & \sqsubseteq C_{q_w} 
\end{align*}
where $\ell$ ranges over $1..2$, $i$ over $0..n-1$, and $\alpha,\beta$
take distinct values from $\Gamma \cup (Q \times \Gamma)$. 

We next verify the existence of levels $n-1$ to 0 of the configuration
tree. We exploit that we have already stored the position of the leaves
in the concept names $A_i$, $\overline{A}_i$ at $L_{n}$-nodes. Each
node on level $i$ branches on the concept names $A_i, \overline{A}_i$
and keeps the choice of $A_j, \overline{A}_j$ for all $j< i$:
%
%
$$
%
\begin{array}{r@{\;}c@{\;}l}
   \exists S . (L_{i+1}\mn{ok} \sqcap A_{i}) \sqcap 
   \exists S . (L_{i+1}\mn{ok} \sqcap \overline{A}_{i}) & \sqsubseteq & \mn{succ}_i \\[1mm]
   A_j \sqcap \exists S . (L_{i+1}\mn{ok} \sqcap A_j) &\sqsubseteq& \mn{ok}''_{i,j} \\[1mm]
   \overline{A}_j \sqcap \exists S . (L_{i+1}\mn{ok} \sqcap \overline{A}_j) &\sqsubseteq&
   \mn{ok}''_{i,j} \\[1mm]
   \mn{succ}_i \sqcap \mn{ok}''_{i,0} \sqcap \cdots \sqcap
   \mn{ok}''_{i,i-1} \sqcap L_i
   &\sqsubseteq& L_{i}\mn{ok} \\[1mm]
  \exists S.(L_{i+1} \sqcap A_j) \sqcap \exists S.(L_{i+1} \sqcap \overline{A_j}) 
  & \sqsubseteq& C_{q_w}
\end{array}
$$
where $i$ ranges over $0..n-1$ and $j$ over $0..i-1$. We also want
that configuration trees have exactly one leaf labeled with a concept
name of the form $A^2_{q,\sigma}$ and exactly one leaf labeled with 
$W$. We start with enforcing the ``at most one'' part of ``exactly one'':
%
\begin{align*} 
  F_2 \sqcap \exists S . (G_2 \sqcap A^2_{q,\sigma})  & \sqsubseteq  H 
  \\
  L_n \sqcap \exists S . (F_2 \sqcap H) & \sqsubseteq  H 
  \\
  L_i \sqcap \exists S . (L_{i+1} \sqcap H) & \sqsubseteq H \\[-0.5mm]
  L_i \sqcap \exists S . (L_{i+1} \sqcap A_{i} \sqcap H) \sqcap
  \exists S . (L_{i+1} \sqcap \overline{A}_{i} \sqcap H) & \sqsubseteq
   C_{q_w} \\
  F_2 \sqcap \exists S . (G_2 \sqcap W)  & \sqsubseteq  W'
  \\
  L_n \sqcap \exists S . (F_2 \sqcap W') & \sqsubseteq  W'
  \\
  L_i \sqcap \exists S . (L_{i+1} \sqcap W') & \sqsubseteq W' \\[-0.5mm]
  L_i \sqcap \exists S . (L_{i+1} \sqcap A_{i} \sqcap W') \sqcap
  \exists S . (L_{i+1} \sqcap \overline{A}_{i} \sqcap W') & \sqsubseteq
   C_{q_w}
\end{align*}
where $i$ ranges over $0..n-1$ and $q,\sigma$ over $Q \times
\Gamma$. Note that we use $A_i$ and $\overline{A}_{i}$ to distinguish
left successors and right successors in the tree: when we see a label
$A^2_{q,\sigma}$ at a $G_2$-leaf, we propagate the marker $H$ up the 
tree and additionally make sure that, at no node of the tree, we have
an $H$-marker coming both from the left successor and from the right
successor. We deal with $W$ in a similar way, propagating the marker
$W'$. 

The ``at least one'' part of ``exactly one'' requires some changes to
the concept inclusions already given, which we only sketch. We have
not included theses changes in the original version of the inclusions
above to avoid cluttering the presentation. Essentially, we have to
keep track of where we have already seen a concept name of the form
$A^2_{q,\sigma}$ in a $G_2$-leaf and where we have already seen a
$G_2$-leaf labeled $W$. For simplicity, let us concentrate on the
latter.  We replace the concept inclusion
$$
  \mn{ok}_{2,0} \sqcap \cdots \sqcap \mn{ok}_{2,n-1} \sqcap
  \Gamma\mn{ok}'_2 \sqcap F_2 \sqsubseteq F_2\mn{ok}
$$ 
above with
\begin{align*}
  \mn{ok}_{2,0} \sqcap \cdots \sqcap \mn{ok}_{2,n-1} \sqcap \Gamma\mn{ok}'_2 
  \sqcap F_2 \sqcap & \\
  \exists S.(G_2 \sqcap W) & \sqsubseteq  F^W_2\mn{ok} \\[1mm]
  \mn{ok}_{2,0} \sqcap \cdots \sqcap \mn{ok}_{2,n-1} \sqcap \Gamma\mn{ok}'_2 
  \sqcap F_2 \sqcap & \\
  \exists S.(G_2 \sqcap \overline{W}) & \sqsubseteq F^{\overline{W}}_2\mn{ok} 
\end{align*}
Note that we have replaced $F_2\mn{ok}$ in the conclusion with $F^W_2\mn{ok}$ and
$F^{\overline{W}}_2\mn{ok}$, recording whether or not there is a $G_2$-node
satisfying $W$ below. The information that we have seen $W$ is then
propagated propagated up the tree, which requires replacing each of
$L_n\mn{ok},\dots,L_1\mn{ok}$ with two versions, $L^W_i\mn{ok}$ and
$L^{\overline{W}}_i\mn{ok}$. On each level, we set $L^{\overline{W}}_{i}\mn{ok}$ if
both successors are labeled with $L^{\overline{W}}_{i+1}\mn{ok}$ and
$L^W_{i}\mn{ok}$ if one successor is labeled with $L^W_{i+1}\mn{ok}$, but not
both. In fact, if both successors are labeled $L^W_{i+1}\mn{ok}$, then
neither $L^W_{i}\mn{ok}$ nor $L^{\overline{W}}_{i}\mn{ok}$ will be set and this is
exactly how we can ensure that there is at most one $G_2$-leaf labeled
$W$. It can be enforced in an analogous way that there is a $G_2$
labeled with a concept name of the form $A^2_{q,\sigma}$. In fact, we
have to deal with both $W$ and these concept names simultaneously,
using concept names such as $L^{W,H}_i\mn{ok}$ indicating that we are at
a tree node on level $i$ below which there is a $G_2$-leaf satisfying
$W$ and a $G_2$-leaf satisfying a concept name  $A^2_{q,\sigma}$.
Details are omitted.

At this point, we have essentially finished the verification of
configuration trees of type 0 (we will comment on the other types
below) and move on to verify computation trees, also in a bottom-up
fashion.  To be a proper part of a computation tree, a configuration
must describe an accepting halting configuration or have
successor configuration trees as required by the transition
relation. For type 0 configuration trees, the former case is covered by
$$ 
\begin{array}{r@{\;}c@{\;}l}
L_0\mn{ok} \sqcap \exists S^{n+2}. (G_2 \sqcap A^2_{q_a,\sigma}) &\sqsubseteq&
\mn{tree}_0 \\[1mm]
L_0 \sqcap \exists S^{n+2} . (G_2 \sqcap A^2_{\alpha}) 
\sqcap
\exists S^{n+2} . (G_2 \sqcap A^2_{\beta}) &\sqsubseteq&
C_{q_w}
\end{array}
$$
where $q_a$ is the accepting state, $\sigma$ ranges over all elements
of $\Gamma$, and $\alpha$ and $\beta$ are distinct elements of $Q
\times \Gamma$. For the latter case and existential states, we add
$$
L_0\mn{ok} \sqcap \exists S^{n+2} . (G_2 \sqcap
  A^2_{q_\exists,\sigma_0} ) \sqcap \exists S . (\mn{tree}_1 \sqcap A_{q_1,\sigma_1,M_1}) \sqsubseteq 
  \mn{tree}_0 
$$
for all $q_\exists \in Q_\exists$, $\sigma_0 \in \Gamma$, and $(q_1,\sigma_1,M_1) \in
\Delta(q_\exists,\sigma_0)$; for universal states, we add
%
\begin{align*}
  L_0\mn{ok} \sqcap\exists S^{n+2} . (G_2 \sqcap A^2_{q_\forall,\sigma_0} ) 
  \, \sqcap \quad& \\[-0.5mm]
  \exists S.(\mn{tree}_1 \sqcap A_{q_1,\sigma_1,M_1}) \, \sqcap \cdots \sqcap 
  \exists S . (\mn{tree}_1 \sqcap A_{q_k,\sigma_k,M_k}) \quad& \\
  \sqsubseteq \;  \mn{tree}_0&
\end{align*}
for all $q_\forall \in Q_\forall$ and $\sigma_0 \in \Gamma$
when $\Delta(q_\forall,\sigma_0) = \{ (q_1,\sigma_1,M_1), \dots,
(q_k,\sigma_k,M_k) \}$. Note that we have used the concept names
$A_{q,\sigma,M}$ as markers here. We still need to enforce that
they really represent the transition in the configuration tree at
whose root they are located. We do this as follows. Each marker
state is the actual state in the current configuration:
$$
  A_{q_1,\sigma_1,M} \sqcap \exists S^{n+2}. (G_2 \sqcap 
  A^2_{q_2,\sigma_2}) \sqsubseteq C_{q_w}
$$
for all distinct $q_1,q_2 \in Q$, all $\sigma_1,\sigma_2 \in \Gamma$,
and all $M \in \{L,R\}$. Each marker symbol is the actual symbol
written in the current configuration:
$$
  A_{q_1,\sigma_1,M} \sqcap \exists S^{n+2}. (G_2 \sqcap W \sqcap
  A^2_{\sigma_2}) \sqsubseteq C_{q_w}
$$
for all distinct $\sigma_1,\sigma_2 \in \Gamma$, all $q_1 \in Q$ and
all $M \in \{L,R\}$. Each marker movement is the actual movement in
the current configuration. To achieve this, we first say that the
$W$-marker is exactly where the head was before:
$$
  L_n \sqcap \exists S^2 . (G_1 \sqcap A^1_{q,\sigma}) 
\sqcap \exists S^2 . (G_2 \sqcap \overline{W}) 
\sqsubseteq C_{q_w}
$$
for all $q \in Q$ and $\sigma \in \Sigma$. Now, right moves
are ensured in the following way:
%
\begin{align*}
   A_{q,\sigma,R} \sqcap \exists S^i . [ L_i \sqcap \exists S . (L_{i+1} 
   \sqcap \overline{A}_i \sqcap \exists S . (L_{i+2} \sqcap A_i \,\sqcap 
   \qquad & \\
   \exists S . \cdots \sqcap \exists S . (L_n \sqcap A_n \sqcap \exists S^2. 
   (G_2 \sqcap \overline{W}) \cdots ) \, \sqcap \qquad & \\
   \exists S . (L_{i+1} \sqcap A_i \sqcap \exists S . (L_{i+2} \sqcap   
   \overline{A}_i \,\sqcap \qquad & \\
   \exists S . \cdots \sqcap \exists S . (L_n \sqcap \overline{A}_n \sqcap 
   \exists S^2. (G_2 \sqcap A^2_{q,\sigma}) \cdots ) ] \sqsubseteq C_{q_w}&
\end{align*}
for all $q \in Q$, $\sigma \in \Gamma$, and $0 \leq i < n$. Note that
this prevents having a leaf labeled with $\overline{W}$ and a leaf to
the immediate right labeled with $A^2_{q,\sigma}$. We ensure that the
leaves are immediate neighbors by going one step to the right and then
only to the left for the first leaf and one step to the left and then
only to the right for the second leaf. We also have to forbid the case
where we want to do a right move, but are already on the right-most
tape cell:
%
\begin{align*}
  L_0 \sqcap \exists S^{n+2} . (G_2 \sqcap A^2_{q_1,\sigma_1} \sqcap A_0 \sqcap
  \cdots \sqcap A_{n-1}) \; \sqcap & \\
  \exists S . (L_0 \sqcap A_{q_2,\sigma_2,R}) & \sqsubseteq C_{q_w}
\end{align*}
for all $q_1,q_2 \in Q$ and $\sigma_1,\sigma_2 \in \Gamma$.  Left
moves can be dealt with in a similar way. To implement the transition
correctly, it remains to state that cells which are not written do not
change their content. This is straightforward:
$$
\begin{array}{r@{\;}c@{\;}l}
  L_n \sqcap \exists S^2 . (G_1 \sqcap A^1_{\sigma_1}) 
\sqcap \exists S^2 . (G_2 \sqcap A^2_{\sigma_2} \sqcap \overline{W}) 
&\sqsubseteq& C_{q_w}  \\[1mm]
  L_n \sqcap \exists S^2 . (G_1 \sqcap A^1_{\sigma_1}) 
\sqcap \exists S^2 . (G_2 \sqcap A^2_{q,\sigma_2} \sqcap \overline{W}) 
&\sqsubseteq& C_{q_w}  \\[1mm]
\end{array}
$$
where $q \in Q$ and distinct $\sigma_1,\sigma_2 \in \Gamma$. 

We need analogous concept inclusions to verify trees of type~1 and~2,
setting concept names $\mn{tree}^1$ and $\mn{tree}^2$ instead of
$\mn{tree}^0$, and to interlink these trees in the computation
tree. The main difference is that we replace the concept names $A^i_a$
and $A^i_{q,a}$ with $i \in \{1,2\}$ with concept names that have
different values for $i$, as described above. Details are omitted.

To complete the verification of (accepting) computation trees, it
remains to set the concept name $A^*$ from Lemma~\ref{lem:2expbasic}
when we reach the initial configuration. We expect that the root of
the initial configuration tree is marked with $I$ and put
$$
  I \sqcap \mn{tree}^j \sqsubseteq A^*
$$
for all $j \in \{0,1,2\}$. Of course, we also need to make sure that
the tree marked by $I$ really represents the initial configuration.
In particular, we expect to see the initial state $q_0$, that the
first $n$ tape cells are filled with the input $w$ and that all other
tape cells are labeled with the blank symbol. All this is easy to
achieve. As an example, assume that the first symbol of $w$ is
$\sigma$. Then put
$$
   I \sqcap \exists S^{n+2} . (G_2 \sqcap \overline{A}_0 \sqcap \cdots
   \sqcap \overline{A}_{n-1} \sqcap 
A^2_{\alpha})
   \sqsubseteq C_{q_w}
$$
for every $\alpha \in \Gamma \cup (Q \times \Gamma)$ that is different
from $(q_0,\sigma)$. To prepare for a simpler formulation of the
query, we add the final inclusions
$$
G_1 \sqsubseteq G \qquad G_2 \sqsubseteq G
$$
which allows us to use $G$ for identifying $G_\ell$-nodes,
independently of the value of $\ell$.

This ends the definition of the TBox $\Tmc_w$. To finish the
reduction, it remains to ensure that configurations are properly
copied between configuration trees, as initially described. The
\emph{$i$-configuration of a configuration tree} is the configuration
represented at the leaves of that tree using the concept names
$A^i_\sigma$ and $A^i_{q,\sigma}$, $i \in \{1,\dots,6\}$. Note that
configuration trees of type 0 have 1- and 2-configurations, trees of
type 1 have 3- and 4-configurations, and trees of type 2 have 5- and
6-configurations. We say that two configuration trees are
\emph{neighboring} if their roots are connected by the role
composition $S$. We have to ensure the following:
\begin{enumerate}

\item[($\dagger$)] if $T$ and $T'$ are neighboring configuration trees,
  then the $i$-configuration of $T$ (if existant) coincides with the
  $j$-configuration of $T'$ (if existant), for all $(i,j) \in \{(2,3),(4,5),(6,1)\}$.




\end{enumerate}
For each of the listed pairs $(i,j)$, condition ($\dagger$) will be
ensured with a UCQ, and the final UCQ $q_w$ is the disjunction of
these. For simplicity, we concentrate on the case $(i,j)=(2,3)$. A bit
more verbosely, Condition~($\dagger$) can then be rephrased as follows:
\begin{itemize}

\item[($\ddagger$)] if $a$ and $b$ are leaves in neighboring configuration
  trees of type~0 and type~1, respectively, and $a$ and $b$ are
  labeled identically regarding the concept names
  $A_i,\overline{A}_i$, then there are no distinct $\alpha,\beta \in
  \Gamma \cup (Q \times \Gamma)$ such that $a$ is labeled with
  $A^2_\alpha$ and $b$ with $A^3_\beta$.

\end{itemize}
We use one CQ $q$ for each choice of $\alpha$ and $\beta$ such that
$q$ has a match precisely if there is the undesired labeling described
in~($\ddagger$).
\begin{figure}[t!]
  \begin{center}
    \framebox[1\columnwidth]{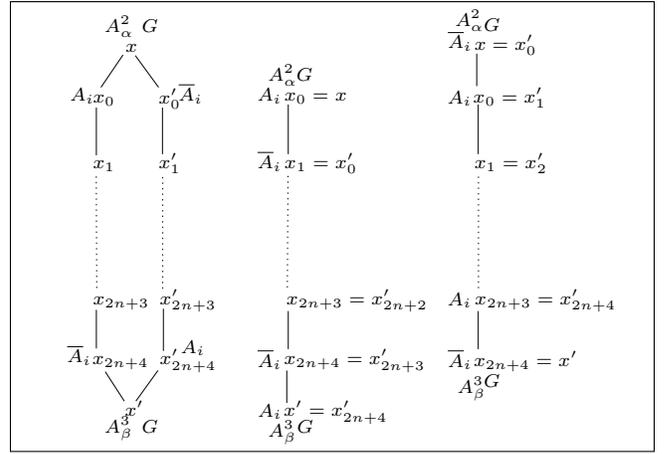}
    \caption{Component query and two identifications.}
    \label{fig:q1b}
  \end{center}
\end{figure}
We construct $q$ from component queries $p_0,\dots,p_{n-1}$, which all
take the form of the query show on the left-hand side of
Figure~\ref{fig:q1b}. Note that all edges are $S$-edges and that the
only difference between the component queries is which concept names
$A_i$ and $\overline{A}_i$ are used. All variables are quantified
variables. We assemble $p_0,\dots,p_{n-1}$ into the desired query $q$
by taking variable disjoint copies of $p_0,\dots,p_{n-1}$ and then
identifying (i)~the $x$-variables of all components and (ii)~the
$x'$-variables of all components.

To see why $q$ achieves~($\ddagger$), first note that the variables $x$ and
$x'$ must be mapped to leaves of configuration trees because of their
$G$-label. Call these leaves $a$ and~$a'$. Since $x$ is labeled with
$A^2_\alpha$ and $x'$ with $A^3_\beta$, $a$ and $a'$ must be in
different trees. Since they are connected to $x$ in the query, both
$x_0$ and $x'_0$ must then be mapped either to $a$ or to its
predecessor; likewise, $x_{2n+4}$ and $x'_{2n+4}$ must be mapped
either to $a'$ or to its predecessor.  Because of the labeling of $a$
and $a'$ and the predecessors in the configuration tree with $A_i$ and
$\overline{A}_i$, we are actually even more constrained: exactly one
of $x_0$ and $x'_0$ must be mapped to $a$, and exactly one of
$x_{2n+4}$ and $x'_{2n+4}$ to~$a'$. Since the paths between leaves in
different configuration trees in the computation tree have length at
least $2n+5$ and $q$ contains paths from $x_0$ to $x_{2n+4}$ and from
$x'_0$ to $x'_{2n+4}$ of length $2n+4$, only the following cases are
possible:
\begin{itemize}

\item $x_0$ is mapped to $a$, $x'_0$ to the predecessor of $a$,
  $x'_{2n+4}$ to $a'$, and $x_{2n+4}$ to the predecessor of $a'$;

\item $x'_0$ is mapped to $a$, $x_0$ to the predecessor of $a$,
  $x_{2n+4}$ to $a'$, and $x'_{2n+4}$ to the predecessor of $a'$.

\end{itemize}
This gives rise to the two variable identifications in each query
$p_i$ shown in Figure~\ref{fig:q1b}.
Note that the first case implies
that $a$ and $a'$ are both labeled with $A_i$ while they are both
labeled with $\overline{A}_i$ in the second case. In summary, $a$ and
$a'$ must thus agree on all concept names $A_i$, $\overline{A}_i$.
Note that with the identification $x_0=x$ (resp.\ $x'_0=x$), there is
a path from $x$ to $x'$ in the query of length $2n+5$. Thus, $a$ and
$a'$ are in neighboring configuration trees. Since $a_1$ must satisfy
$A^2_\sigma$ and $a_2$ must satisfy $A^3_\beta$ due to the labeling of
$x$ and $x'$, we have achieved~($\ddagger$).

We now show how to replace the UCQ used in the reduction with a 
CQ. This requires the following changes:
\begin{enumerate}

\item the $F$-nodes in configuration trees receive additional labels:
  when a $G$-node is labeled with $A^i_\alpha$, then its predecessor
  $F$-node is labeled with $A^i_\beta$ for all $\beta \in
  (\Gamma \cup (Q \times \Gamma)) \setminus \{ \alpha \}$ and
  with $A^j_\beta$ for all $j \in \{ 1,\dots,6\} \setminus \{i \}$ and 
  all $\beta \in  \Gamma \cup (Q \times \Gamma)$;

\item the roots of configuration trees receive an additional label
  $R_0$ or $R_1$, alternating with neighboring trees;

\item the query construction is modified.

\end{enumerate}
Points~1 and~2 are important for the CQ to be constructed to work
correctly and can be achieved in a straightforward way by modifying
$\Tmc_w$, details are omitted. We thus concentrate on Point~3. The
desired CQ $q$ is again constructed from component queries. We use $n$
components as shown in Figure~\ref{fig:q1b}, except that the
$A^2_\alpha$ and $A^3_\beta$-labels are dropped. We further add the
component (partially) shown in Figure~\ref{fig:cq} where again $x$ and $x'$ are
the variables shared with the other components, and where we assume
that $C_0,\dots,C_{m-1}$ are all concept names of the form
$A^i_\alpha$, $i \in \{1,\dots,6\}$ and $\alpha \in \Gamma \cup (Q \times \Gamma)$.
\begin{figure}[t!]
  \begin{center}
    \framebox[1\columnwidth]{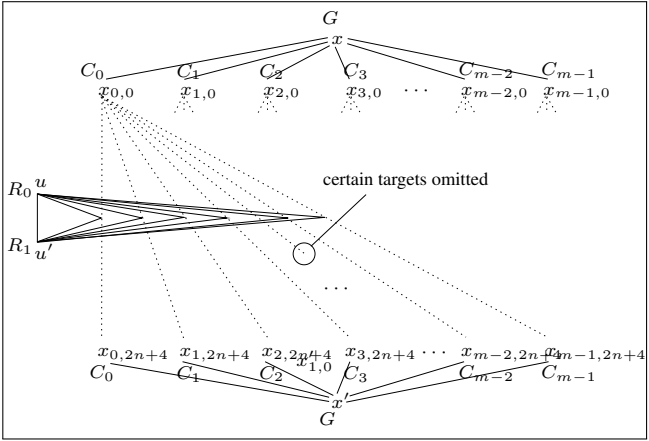}
    \caption{Additional component for CQ.}
    \label{fig:cq}
  \end{center}
\end{figure}
The dotted edges denote $S$-paths of length $2n+4$. There is an
$S$-path from every variable $x_{i,0}$ to every variable $x_{j,2n+4}$
except when $x_{i,0}$ is labeled with a concept name
$C_i=A^\ell_\sigma$ and $x_{j,2n+4}$ with $C_j=A^k_{\sigma'}$ such that
$(\ell,k) \in \{(2,3),(4,5),(6,1)\}$ and $\sigma\neq \sigma'$.  The
variables $u$ and $u'$ are connected with the middle point of each
$S$-path, that is, with the variable on the path which has distance
$n+2$ to the $x_{\ell,0}$ variable where the path starts and also
distance $n+2$ to the $x_{k,2n+4}$ variable where it ends.

We have to argue that the CQ $q$ just constructed
achieves~($\ddagger$).  As before, $x$ and $x'$ must be mapped to
leaves of configuration trees because of their $G$-label. Call these leaves
$a$ and $a'$. All $x_{i,0}$ must then be mapped to $a$ or its
predecessor, and all $x_{i,2n+4}$ must be mapped to $a'$ or its
predecessor. In fact, due to the labeling of $a$ and $a'$ and their
predecessors in the configuration tree (see Point~1 above), exactly one
variable $x_{i,0}$ from $x_{0,0},\dots,x_{m-1,0}$ is mapped to $a$
while all others are mapped to the predecessor of $a$; likewise,
exactly one variable $x_{j,2n+4}$ from $x_{0,2n+4},\dots,x_{m-1,2n+4}$ is
mapped to $a'$ while all others are mapped to the predecessor of
$a'$. To achieve ($\ddagger$), we have to argue that $x_{i,0}$ and
$x_{j,2n+4}$ are labeled with concept names $C_i = A^\ell_\sigma$ and
$C_j = A^k_{\sigma'}$ where $(\ell,k) \in \{(2,3),(4,5),(6,1)\}$ and $\sigma\neq \sigma'$, and
that $a$ and $a'$ are in neighboring computation trees.

We start with the former. Assume to the contrary that $x_{i,0}$ and
$x_{j,2n+4}$ are not labeled with concept names in the described
way. Then they are connected in $q$ by a path of length $2n+4$ whose
middle point $y$ is connected to the variables $u$ and $u'$. In a
match to a computation tree, there are four possible targets for $u$
and $u'$ and for the predecessor $y_{-1}$ of $y$ on the connecting
path and the successor $y_{+1}$ of $y$ on that path:
\begin{enumerate}

\item $u,y_{-1}$ map to the same target, and so do $u'$ and $y$;

\item $u,y$ map to the same target, and so do $u'$ and $y_{+1}$;

\item $u',y_{-1}$ map to the same target, and so do $u$ and $y$;

\item $u',y$ map to the same target, and so do $u$ and $y_{+1}$.

\end{enumerate}
However, options~1 and~3 are impossible because there would have to be
a path of length $n+1$ from a node labeled $R_0$ or $R_1$ to the leaf
$a$. Similarly, options~2 and~4 are impossible because there would
have to be a path of length $n+1$ from a node labeled $R_0$ or $R_1$
to the leaf $a'$. Thus, we have shown that $x_{i,0}$ and $x_{j,2n+4}$
are labeled with concept names as described.

The labeling of $x_{i,0}$ and $x_{j,2n+4}$ with concept names $C_i =
A^\ell_\sigma$ and $C_j = A^k_{\sigma'}$ where $(\ell,k) \in
\{(2,3),(4,5),(6,1)\}$ together with the labeling scheme of
Figure~\ref{fig:tree} also means that $a$ and $a'$ (to which $x_{i,0}$
and $x_{j,2n+4}$ are mapped) are not in the same configuration
tree. Moreover, they cannot be in configurations trees that are
further apart than one step because under the assumption that
$x=x_{i,0}$ and $x'=x_{j,2n+4}$, there is a path of length $2n+5$ in
the query from $x$ to $x'$. Note that we can identify $u$ with the
$2n+2$nd variable on any such path and $u'$ with the $2n+3$rd variable
(or vice versa) to admit a match in neighboring configuration trees.

\begin{lemma}
  $\Tmc_w$, $q_w$, $\Sigma_w$, and $A^*$ satisfy Points~1 to~4 from
  Lemma~\ref{lem:2expbasic}.
\end{lemma}
\begin{proof}(sketch)
  We have to show the following:
  \begin{enumerate}

  \item If $M$ accepts $w$, then there is a $\Sigma_w$-ABox \Amc and
    $a \in \mn{Ind}(\Amc)$ such
    that $\Amc,\Tmc_w \models A^*(a)$ and
    $\Amc,\Tmc_w \not\models q_w$.

  \item If $M$ does not accept $w$, then for any $\Sigma_w$-ABox \Amc,
    $\Amc,\Tmc_w \models \exists x \, A^*(x)$ implies 
    $\Amc,\Tmc_w \models q_w$. 

  \item $q_w$ is FO-rewritable relative to $\Tmc_w$ and $\Sigma_w$. 

\item There is an \ELI-concept $C_{q_w}$ such that $d \in C^\Imc$
  implies $\Imc \models q_w$. 

\end{enumerate}
(1) Take as \Amc the computation tree of $M$ on $w$ viewed as an ABox,
including correct copying of configurations between neighboring
configuration trees. Let $a$ be the root of \Amc, marked with the concept $I$. The verification of
computation trees by $\Tmc_w$ yields $\Amc,\Tmc_w \models 
A^*(a)$. Since the copying of configurations is as intended, we have
$\Amc,\Tmc_w \not\models q_w$.

    \smallskip
    \noindent 
    (2) Since the verification of (homomorphic images of) computation
    trees by $\Tmc_w$ is sound,
    $\Amc,\Tmc_w \models \exists x \, A^*(x)$ implies that \Amc
    contains a homomorphic image of a computation tree. Note that this
    tree has the initial configuration of $M$ on $w$ as the root,
    locally (within configuration trees) respects the transition
    relation of $M$, and has only accepting configurations as leaves.
    Since $M$ does not accept $w$, the tree must fail to correctly
    copy configurations between neighboring configuration trees.
    Consequently, $\Amc,\Tmc_w \models q_w$.

%

    \smallskip
    \noindent 
    (3) 
    The query $q_w$ contains only concept and role names that do not
    occur on the right-hand side of concept inclusions except those of
    the form $D \sqsubseteq C_{q_w}$. In fact, the FO-rewriting of
    $q_w$ relative to $\Tmc_w$ and $\Sigma_w$ is the UCQ $\widehat
    q_w$ that consists of the CQ $q_w$ and (essentially) one CQ $q_D$
    for each inclusion $D \sqsubseteq C_{q_w}$, where $q_D$ is the
    CQ-representation of the formula $\exists x \, D(x)$. This is a
    slight oversimplification, e.g.\ due to our use of the markers $H$
    and $W'$ used for enforcing that each configuration tree has at
    most one leaf labeled with a concept name of the form
    $A^2_{q,\sigma}$. However, it is not hard to see that we can
    ``expand away'' these marker concepts, which results in a UCQ to
    be included in $\widehat q_w$. In particular, the markers are
    propagated only along the boundedly many levels of configuration
    trees, so the resulting UCQ is finite.

\smallskip
    \noindent 
(4)  Select distinct $a,b \in \Gamma$ and set
  $$
  \begin{array}{r@{\,}c@{\,}l}
  G_1 &=& G \sqcap \overline{A}_0 \sqcap \cdots \sqcap \overline{A}_n \sqcap A^2_a \\[1mm]
  G_2 &=& G \sqcap \overline{A}_0 \sqcap \cdots \sqcap \overline{A}_n \sqcap A^3_b \\[1mm]
  F_1 &=& A_0 \sqcap \cdots \sqcap A_n \sqcap \midsqcap_{c \in (\Gamma
          \cup (Q\times \Gamma))
          \setminus \{ a \}} A^2_c \, \sqcap \\[1mm]
          && \midsqcap_{\alpha \in \Gamma
          \cup (Q\times \Gamma), \ j \in \{1,3,4,5,6\}} A^j_\alpha \\[1mm]
  F_2 &=& A_0 \sqcap \cdots \sqcap A_n \sqcap \midsqcap_{c \in (\Gamma
          \cup (Q\times \Gamma))
          \setminus \{ b \}} A^3_c \, \sqcap \\[1mm]
          && \midsqcap_{\alpha \in \Gamma
          \cup (Q\times \Gamma), \ j \in \{1,2,4,5,6\}} A^j_\alpha\\[1mm]
  C_{q_w} &=& R_0 \sqcap \exists S^{2n+1}.(F_1 \sqcap \exists S . G_1)
              \, \sqcap \\[1mm]
              && \exists S . (R_1 \sqcap \exists S^{2n+1}.(F_2 \sqcap \exists S . G_2))
  \end{array}
  $$
  It can be verified that $C_{q_w}$ has the stated property.
\end{proof}

\subsection{Adaptation to Datalog}

Our aim is to prove Theorem~\ref{thm:dlog}. We first introduce the
relevant notions. A \emph{Datalog rule} takes the form
$$
R_1(\xbf_1) \wedge \cdots \wedge R_n(\xbf_n) 
\rightarrow
  R_0(\xbf_0) 
$$
where $R_0,\dots,R_n$ are relation names and $\xbf_0,\dots,\xbf_n$ are
tuples of variables such that the length of each $\xbf_i$ matches the
arity of $R_i$ and $\xbf_0 \subseteq \xbf_1 \cup \cdots \cup \xbf_n$.
For brevity, we shall speak of relations rather than of relation
names.  We call $R_0(\xbf_0)$ the \emph{head} of the rule and
$R_1(\xbf_1) \wedge \cdots \wedge R_n(\xbf_n)$ the \emph{body}.  A
\emph{Datalog program} is a set of Datalog rules with a distinguished
relation \mn{goal} that occurs only in rule heads. A relation is
called \emph{extensional} or \emph{EDB} if it occurs only in rule
bodies; it is called \emph{intensional} or \emph{IDB} if it occurs in
at least one rule head. The \emph{EDB schema} of a program is the set
of all EDB relations in it.  A Datalog program is \emph{monadic} if
all IDB relations with the possivel exception of \mn{goal} are unary;
it is \emph{Boolean} if \mn{goal} has arity zero. We will concentrate
on Boolean monadic Datalog programs. Moreover, we will only use unary
and binary EDB relations which correspond to concept and role names
from the ABox signature, respectively. IDB relations then correspond
to concept names that are not in the ABox signature. For the semantics
of Datalog and the definition of boundedness of a Datalog program, we
refer to~\cite{alice}.  We evaluate Datalog programs over
$\Sigma$-ABoxes where $\Sigma$ is the EDB schema of the program.  Note
that the rule body of a Datalog program is a CQ. Tree-shapedness of a
CQ $q$ is defined in the same way as for an ABox in
Section~\ref{sect:charact}, that is, $q$ viewed as an undirected graph
must be a tree without multi-edges.

\smallskip

For convenience, we repeat the theorem to be proved.

\THMdatalog*

We start with Point~1, first establishing it for rooted UCQs (a
disjunction of rooted CQs) and then strengthening to CQs. Recall the
reduction of the exponential torus tiling problem presented in
Section~\ref{app:conexplower}. Let $P$ be the tiling problem that is
\NExpTime-complete and $c$ an input for $P$. We have shown how to
construct in polynomial time an $\ELI$ TBox~$\Tmc_c$, a rooted CQ
$q_c(x)$, and an ABox signature $\Sigma_c$ such that, for a selected
concept name $A^* \notin \Sigma_c$, $P$ has a solution given $c$ iff
$(\Tmc_c,\Sigma_c,A^*) \not\subseteq (\Tmc_c,\Sigma_c,q_c)$ over
$\Sigma_c$-ABoxes.
We
show how to convert $\Tmc_c$ and $q_c$ into a Boolean monadic Datalog
program $\Pi_c$ and a rooted UCQ $p_c$, both over EDB schema
$\Sigma_c$, such that $P$ has a solution given $c$ iff $\Pi_c
\not\subseteq p_c$.

It is standard to convert an \ELI-concept $C$ into a CQ $q_C(x)$ that
is equivalent in the sense that for all interpretations \Imc and $d
\in \Delta^\Imc$, we have $d \in C^\Imc$ iff $\Imc \models q_C[d]$.
We omit the details and only mention as an example that
$$
C=\exists r .  \exists s . A \sqcap \exists s . B 
$$
is converted into 
$$
r(x,y) \wedge s(y,z) \wedge A(z) \wedge s(x,u) \wedge B(u). 
$$
Thus, a CI of the form $C \sqsubseteq A$ can be viewed as the
monadic Datalog rule $C_q(x) \rightarrow A(x)$. 

The monadic Datalog program $\Pi_c$ contains the following rules:
\begin{enumerate}

\item $A(x) \rightarrow \widehat A(x)$ for each $A \in \Sigma_c$;

\item for each CI $D \sqsubseteq A$ in $\Tmc_c$ with $A$ a concept
  name different from $A^*$: $q_{D'}(x) \rightarrow A(x)$;

\item for each CI $D \sqsubseteq A^*$: $q_{D'}(x) \rightarrow
  \mn{goal}(x)$.

\end{enumerate}
where $D'$ is obtained from $D$ by replacing each concept name $A \in
\Sigma_c$ with $\widehat A$. This renaming, as well as the rules in
Point~1 above, achieve the separation between EDB and IDB relations
required in Datalog. The rooted UCQ $p_c$ is the disjunction of
\begin{enumerate}

\item the CQ $q_c$;

\item the CQ $q_D$ for each $D \sqsubseteq C_{q_c}$ in $\Tmc_c$.

\end{enumerate}
It can be verified that $p_c$ is indeed formulated over EDB schema
$\Sigma_c$. To show that $\Pi_c$ and $p_c$ are as desired, it remains
to establish the following. 
\begin{lemma}
  A $\Sigma_c$-ABox \Amc and individual name $a$ witness $(\Tmc_c,\Sigma_c,A^*)
  \not\subseteq (\Tmc_c,\Sigma_c,q_c)$ iff they witness $\Pi_c
  \not\subseteq p_c$.
\end{lemma}
\begin{proof}
  First, let $\Amc$ and $a$ be a witness of $(\Tmc_c,\Sigma_c,A^*)
  \not\subseteq (\Tmc_c,\Sigma_c,q_c)$. Then $\Amc,\Tmc_c \models
  A^*[a]$ and $\Amc,\Tmc_c \not\models q_c[a]$. By the latter, 
  \begin{itemize}

  \item[$(*)$] CIs from $\Tmc_c$ that are of the form $D \sqsubseteq
    C_{q_c}$ never apply.

  \end{itemize}
  Consequently and by definition of $\Pi_c$, from $\Amc,\Tmc_c \models
  A^*[a]$ we obtain $\Amc \models \Pi_c[a]$. By ($*$), the only CQ $q$
  from $p_c$ that could satisfy $\Amc \models q[a]$ is $q_c$. However,
  this is not the case since $\Amc,\Tmc_c \not\models q_c[a]$.

  \smallskip

  Now let $\Amc$ and $a$ witness $\Pi_c \not\subseteq p_c$. Then $\Amc
  \models \Pi_c[a]$ and $\Amc \not\models p_c[a]$. From the latter, we
  get $\Amc \not\models q_c[a]$ and $\Amc \not\models D[a]$ whenever
  $D \sqsubseteq C_{q_c}$ is in $\Tmc_c$. Consequently and since both
  $q_c$ and all such concepts $D$ contain only symbols that never
  occur on the right-hand side of a CI in $\Tmc_c$ (except when they
  are of the form $D \sqsubseteq C_{q_c}$), we must have $\Amc,\Tmc_c
  \not\models q_c[a]$. It thus remains to show $\Amc,\Tmc_c \models
  A^*[a]$. However, this is immediate from $\Amc 
  \models \Pi_c[a]$ and the construction of $\Pi_c$. 
\end{proof}
As the next step, we show how to replace the rooted UCQ $p_c$ with a
rooted CQ $p'_c$. The general idea is to replace disjunction with
conjunction.  Let the CQs in $p_c$ be $q_1(x),\dots,q_k(x)$ and let
$q_i(x_i)$ be $q_i(x)$ with the answer variable $x$ renamed to
$x_i$. Introduce additional role names $g_0,\dots,g_k$ that are
included in $\Sigma_c$. Then set
$$
\begin{array}{rcl}
  p'_c(x) &=& g_0(x,x_0) \wedge g_1(x_0,x_1) \wedge \cdots \wedge g_k(x_0,x_k) \, \wedge \\[1mm]
  && q_1(x_1) \wedge \cdots \wedge q_k(x_k). 
\end{array}
$$
%
To make the new query work, we need to install additional gadgets in
the toris tree. Recall that every element of the 

In particular, we want that for each $i \in
\{1,\dots,k\}$, the root of the torus tree has a $g_i$-predecessor
$a_i$ which in turn has, for each $j \in \{1,\dots,i-1,i+1,\dots,k\}$,
a $g_j$-successor that is the root of an ABox which has exactly the
shape of $q_j$. Further, the torus tree gets a new root $a_0$ that has
a $g_0$-edge to each of the individuals $a_1,\dots,a_k$; note that the
torus ``tree'' is actually no longer a tree.  Then a query $q_i$
matches at the root of the original torus tree iff $p'_c$ matches at
the new root $a_0$. The additional parts of the torus ``tree'' need to
be verified in the derivation of $\mn{goal}$ in $\Pi_c$ (which is
essentially identical to the derivation of $L_0\mn{ok}$ in
$\Tmc_c$). Given that $\Pi_c$ is a Datalog program and that the rule
bodies need not be tree-shaped, it is straightforward to modify
$\Pi_c$ to achieve this.

\medskip

For Point~2 of Theorem~\ref{thm:dlog}, we again start with a UCQ in
the first step and improve to a CQ in a second step. The first step is
exactly analogous to the construction of $\Pi_c$ and $p_c$ above.
Recall the reduction of the word problem of exponentially
space-bounded ATMs in Section~\ref{app:2explower}. Let $M$ be the ATM
whose word problem is 2\ExpTime-hard and let $w$ be an input to
$M$. We have shown how to construct in polynomial time an $\ELI$
TBox~$\Tmc_w$, a Boolean CQ $q_w$, and an ABox signature $\Sigma_w$
such that, for a selected concept name $A^* \notin \Sigma_w$, $M$
accepts $w$ iff $(\Tmc_w,\Sigma_w,\exists x \, A^*(x)) \not\subseteq
(\Tmc_w,\Sigma_w,q_w)$ over $\Sigma_w$-ABoxes. We can convert $\Tmc_w$ and $q_w$ into a
monadic Datalog program $\Pi_w$ and a UCQ $p_w$ in exactly the same
way in which we had constructed $\Pi_c$ and $p_c$ above. Note in
particular that all CIs in $\Tmc_w$ of the form $D \sqsubseteq
C_{q_w}$ are such that $D$ contains only symbols from $\Sigma_w$, and
that also $q_w$ contains only symbols from $\Sigma_w$. Thus, $\Pi_w$
and $p_w$ are both over EDB schema $\Sigma_w$, as required. It is
straightforward to establish the following lemma.
\begin{lemma}
  A $\Sigma_w$-ABox \Amc witnesses $(\Tmc_w,\Sigma_w,A^*)
  \not\subseteq (\Tmc_w,\Sigma_w,q_w)$ iff it witnesses $\Pi_w
  \not\subseteq p_w$.
\end{lemma}
It remains to replace the UCQ $p_w$ with a CQ $p'_w$. The idea is
again similar to the proof of Point~1. However, we now want to avoid
introducing rules into $\Pi_w$ whose bodies are not tree-shaped.  This
is possible since we work with Boolean queries here. 

Apart from the original Boolean CQ $q_w$, let the CQs in $p_w$ be
$q_1(x),\dots,q_k(x)$ and let $q_i(x_i)$ be $q_i(x)$ with the answer
variable $x$ renamed to $x_i$. Moreover, let $q_{k+1}(u)$ be $q_w()$
with $u$ made an answer variable and let $q_{k+2}(u')$ be $q_w()$ with
$u$ made an answer variable, see Figure~\ref{fig:cq} for details.
Introduce additional role names $g_1,\dots,g_{k+2}$ that are included in
$\Sigma_w$. Then set
$$
\begin{array}{rcl}
  p'_w() &=& g_1(x_0,x_1) \wedge \cdots \wedge g_k(x_0,x_{k+2}) \, \wedge \\[1mm]
  && q_1(x_1) \wedge \cdots \wedge q_{k+2}(x_{k+2}). 
\end{array}
$$
%
To make the new query work, we need to install additional gadgets in
the computation tree. 
%
In particular, we want that for each $i \in \{1,\dots,k+1\}$, each
node of the computation tree has a $g^-_i$-successor which in turn
has, for each $j \in \{1,\dots,i-1,i+1,\dots,k+1\}$, a $g_j$-successor
that is the root of a tree-shaped ABox in which $q_j$ has a match.
Then a query $q_i$ matches in the computation tree iff $p'_w$ matches
in it. The additional parts of the computation tree need to be
verified in the derivation of $\mn{goal}$ in $\Pi_w$ (which is
essentially identical to the derivation of $A^*$ in $\Tmc_w$). This is
easy to achieve, but we still have to say what exactly the tree shape
ABoxes look like in which $q_1,\dots,q_{k+2}$ have a match. The
queries $q_1,\dots,q_k$ are tree-shaped by definition (and use only
symbols from $\Sigma_c$) and thus we can simply use these queries used
as an ABox. For $q_{k+1}=q_w(u)$, we use the concept $C_{q_w}$ viewed
as an ABox. And finally, for $q_{k+2}=q_w(u')$, we use the ABox
obtained from $C_{q_w}$ by swapping the concept names $R_0$ and $R_1$.

\newpage

\bibliographystyleA{named}

\renewcommand\refname{References of Appendix}


\begin{thebibliography}{}

\bibitem[\protect\citeauthoryear{Ajtai and Gurevich}{1994}]{AG94}
Mikl{\'{o}}s Ajtai and Yuri Gurevich.
\newblock Datalog vs First-Order Logic.
\newblock J. Comput. Syst. Sci.,49(3): 562--588, 1994.

\bibitem[\protect\citeauthoryear{Baader \bgroup \em et al.\egroup
  }{2003}]{baader-2003-dl-handbook}
Franz Baader, Diego Calvanese, Deborah~L. McGuinness, Daniele Nardi, and
  Peter~F. Patel-Schneider, editors.
\newblock {\em The Description Logic Handbook: Theory, Implementation, and
  Applications}. Cambridge University Press, 2003.

\bibitem[\protect\citeauthoryear{Baader \bgroup \em et al.\egroup
  }{2005}]{baader-2005}
Franz Baader, Sebastian Brandt, and Carsten Lutz.
\newblock Pushing the $\mathcal{E\kern-0.1emL}$ envelope.
\newblock In {\em Proc.\ of IJCAI}, pages 364--369, 2005.

\bibitem[\protect\citeauthoryear{Baget \bgroup \em et al.\egroup
  }{2011}]{montpellier-2011}
Jean-Fran\c{c}ois Baget, Michel Lecl{\`e}re, Marie-Laure Mugnier, and Eric
  Salvat.
\newblock On rules with existential variables: Walking the decidability line.
\newblock {\em Artif.\ Intell.}, 175(9-10):1620--1654, 2011.

\bibitem[\protect\citeauthoryear{Benedikt \bgroup \em et al.\egroup
  }{2012}]{DBLP:conf/icalp/BenediktBS12}
Michael Benedikt, Pierre Bourhis, and Pierre Senellart.
\newblock Monadic Datalog Containment.
\newblock In {\em Proc.\ of ICALP}, pages 79--91, 2012.

\bibitem[\protect\citeauthoryear{Benedikt \bgroup \em et al.\egroup
  }{2015}]{DBLP:conf/lics/BenediktCCB15}
Michael Benedikt, Balder ten Cate, Thomas Colcombet, and Michael Vanden Boom.
\newblock The Complexity of Boundedness for Guarded Logics.
\newblock In {\em Proc.\ of LICS}, pages 293--304, 2015.

\bibitem[\protect\citeauthoryear{Bienvenu and
  Ortiz}{2015}]{DBLP:conf/rweb/BienvenuO15}
Meghyn Bienvenu and Mag\-da\-le\-na Ortiz.
\newblock Ontology-mediated query answering with data-tractable description
  logics.
\newblock In {\em Proc.\ of Reasoning Web}, volume 9203 of {\em LNCS}, pages
  218--307, 2015.

\bibitem[\protect\citeauthoryear{Bienvenu \bgroup \em et al.\egroup
  }{2010}]{DBLP:conf/dlog/BienvenuELOS10}
Meghyn Bienvenu, Thomas Eiter, Carsten Lutz, Magdalena Ortiz, and Mantas
  Simkus.
\newblock Query answering in the description logic {S}.
\newblock In {\em Proc.\ of DL}, volume 573 of {\em CEUR-WS,} 2010.

\bibitem[\protect\citeauthoryear{Bienvenu \bgroup \em et al.\egroup
  }{2012}]{LuWo-KR12}
Meghyn Bienvenu, Carsten Lutz, and Frank Wolter.
\newblock Query containment in description logics reconsidered.
\newblock In {\em Proc\ of KR}, pages 221--231, 2012.

\bibitem[\protect\citeauthoryear{Bienvenu \bgroup \em et al.\egroup
  }{2013}]{ijcai-2013}
Meghyn Bienvenu, Carsten Lutz, and Frank Wolter.
\newblock First order-rewritability of atomic queries in {Horn} description
  logics.
\newblock In {\em Proc.\ of IJCAI}, pages 754--760, 2013.

\bibitem[\protect\citeauthoryear{Bienvenu \bgroup \em et al.\egroup
  }{2014}]{todswe}
Meghyn Bienvenu, Balder ten Cate, Carsten Lutz, and Frank Wolter.
\newblock Ontology-based data access: a study through disjunctive datalog,
  {CSP}, and {MMSNP}.
\newblock {\em Proc.\ of TODS}, 39, 2014.

\bibitem[\protect\citeauthoryear{Bourhis and Lutz}{2016}]{BouLu-KR16}
Pierre Bourhis and Carsten Lutz.
\newblock Containment in monadic disjunctive datalog, MMSNP, and expressive
  description logics.
\newblock In {\em Proc.\ of KR}, 2016.

\bibitem[\protect\citeauthoryear{Calvanese \bgroup \em et al.\egroup
  }{2007}]{calvanese-2007}
Diego Calvanese, Giuseppe~De Gia\-co\-mo, Domenico Lembo, Maurizio Lenzerini, 
and
  Riccardo Rosati.
\newblock Tractable reasoning and efficient query answering in description
  logics: The {DL-Lite} family.
\newblock {\em J.\ Autom. Reasoning}, 39(3):385--429, 2007.

\bibitem[\protect\citeauthoryear{Calvanese \bgroup \em et al.\egroup
  }{2009}]{DBLP:conf/rweb/CalvaneseGLLPRR09}
Diego Calvanese, Giuseppe {De Giacomo}, Domenico Lembo, Maurizio Lenzerini,
  Antonella Poggi, Mariano Rodriguez{-}Muro, and Riccardo Rosati.
\newblock Ontologies and databases: The DL-Lite approach.
\newblock In {\em Proc.\ of Reasoning Web}, volume 5689 of {\em LNCS}, pages
  255--356, 2009.

\bibitem[\protect\citeauthoryear{Chaudhuri and Vardi}{1994}]{DBLP:conf/pods/ChaudhuriV94}
Surajit Chaudhuri and Moshe Y. Vardi.
\newblock On the complexity of equivalence between recursive and nonrecursive 
datalog programs
\newblock In {\em Proc.\ of PODS}, pages 107--116, 1994.

\bibitem[\protect\citeauthoryear{Civili and
  Rosati}{2015}]{DBLP:conf/cilc/CiviliR15}
Cristina Civili and Riccardo Rosati.
\newblock On the first-order rewritability of conjunctive queries over binary
  guarded existential rules.
\newblock In {\em Proc.\ of CILC}, volume 1459 of {\em CEUR-WS}, pages 25--30, 
2015.

\bibitem[\protect\citeauthoryear{Cosmadakis \bgroup \em et al.\egroup
  }{1988}]{DBLP:conf/stoc/CosmadakisGKV88}
Stavros~S. Cosmadakis, Haim Gaifman, Paris~C. Kanellakis, and Moshe~Y. Vardi.
\newblock Decidable optimization problems for database logic programs
  (preliminary report).
\newblock In {\em Proc.\ of STOC}, pages 477--490, 1988.

\bibitem[\protect\citeauthoryear{Eiter \bgroup \em et al.\egroup
  }{2012}]{DBLP:conf/aaai/EiterOSTX12}
Thomas Eiter, Magdalena Ortiz, Mantas Simkus, Trung{-}Kien Tran, and Guohui
  Xiao.
\newblock Query rewriting for {Horn-SHIQ} plus rules.
\newblock In {\em Proc.\ of AAAI}, 2012.

\bibitem[\protect\citeauthoryear{Gottlob \bgroup \em et al.\egroup
  }{2013}]{DBLP:conf/icalp/GottlobPT13}
Georg Gottlob, Andreas Pieris, and Lidia Tendera.
\newblock Querying the guarded fragment with transitivity.
\newblock In {\em Proc.\ of ICALP}, volume 7966 of {\em LNCS}, pages 287--298, 
2013.

\bibitem[\protect\citeauthoryear{Hansen \bgroup \em et al.\egroup
  }{2015}]{IJCAI15}
Peter Hansen, Carsten Lutz, Inan{\c{c}} Seylan, and Frank Wolter.
\newblock Efficient query rewriting in the description logic {EL} and beyond.
\newblock In {\em Proc.\ of IJCAI}, pages 3034--3040, 2015.


\bibitem[\protect\citeauthoryear{Kaminski \bgroup \em et al.\egroup
  }{2014}]{kaminski14}
Mark Kaminski, Yavor Nenov, and Bernardo Cuenca Grau.
\newblock Computing datalog rewritings for disjunctive datalog programs and
  description logic ontologies.
\newblock In {\em Proc.\ of RR}, pages 76--91, 2014.

\bibitem[\protect\citeauthoryear{Kontchakov \bgroup \em et al.\egroup
  }{2013}]{DBLP:conf/rweb/KontchakovRZ13}
Roman Kontchakov, Mariano Rodriguez-Muro, and Michael Zakharyaschev.
\newblock Ontology-based data access with databases: A short course.
\newblock In {\em Proc.\ of Reasoning Web}, pages 194--229, 2013.

\bibitem[\protect\citeauthoryear{Lutz and Wolter}{2012}]{lutz-2012}
Carsten Lutz and Frank Wolter.
\newblock Non-uniform data complexity of query answering in description logics.
\newblock In {\em Proc.\ of KR}, 2012.

\bibitem[\protect\citeauthoryear{Lutz}{2008}]{DBLP:conf/cade/Lutz08}
Carsten Lutz.
\newblock The complexity of conjunctive query answering in expressive
  description logics.
\newblock In {\em Proc.\ of IJCAR}, volume 5195 of {\em LNCS}, pages 179--193, 
2008.

\bibitem[\protect\citeauthoryear{P{\'e}rez-Urbina \bgroup \em et al.\egroup
  }{2010}]{perezurbina10tractable}
H{\'e}ctor P{\'e}rez-Urbina, Boris Motik, and Ian Horrocks.
\newblock Tractable query answering and rewriting under description logic
  constraints.
\newblock {\em J.\ Applied Logic}, 8(2):186--209, 2010.

\bibitem[\protect\citeauthoryear{Rosati}{2007}]{rosati07on}
Riccardo Rosati.
\newblock On conjunctive query answering in EL.
\newblock In {\em Proc.\ of DL}, pages 451--458, 2007.

\bibitem[\protect\citeauthoryear{Trivela \bgroup \em et al.\egroup
  }{2015}]{DBLP:journals/ws/TrivelaSCS15}
Despoina Trivela, Giorgos Stoilos, Alexandros Chortaras, and Giorgos~B. Stamou.
\newblock Optimising resolution-based rewriting algorithms for {OWL}
  ontologies.
\newblock {\em J.\ Web Sem.}, 33:30--49, 2015.

\end{thebibliography}

\begin{thebibliography}{}

\bibitem[\protect\citeauthoryear{Abiteboul \bgroup \em et al.\egroup
  }{1995}]{alice}
Serge Abiteboul, Richard Hull, and Victor Vianu
\newblock Foundations of Databases: The Logical Level
\newblock Addison-Wesley, 1995.


\bibitem[\protect\citeauthoryear{Baader \bgroup \em et al.\egroup
  }{2010}]{DBLP:conf/kr/BaaderBLW10}
Franz Baader, Meghyn Bienvenu, Carsten Lutz, and Frank Wolter.
\newblock Query and predicate emptiness in description logics.
\newblock In {\em Proc.\ of KR}, pages 192--202, 2010.

\bibitem[\protect\citeauthoryear{Chandra \bgroup \em et al.\egroup
  }{1981}]{DBLP:journals/jacm/ChandraKS81}
Ashok~K. Chandra, Dexter Kozen, and Larry~J. Stockmeyer.
\newblock Alternation.
\newblock {\em J. {ACM}}, 28(1):114--133, 1981.

\bibitem[\protect\citeauthoryear{Hustadt \bgroup \em et al.\egroup
  }{2007}]{journals/jar/HustadtMS07}
Ullrich Hustadt, Boris Motik, and Ulrike Sattler.
\newblock Reasoning in description logics by a reduction to disjunctive
  datalog.
\newblock {\em J. Autom. Reasoning}, 39(3):351--384, 2007.

\bibitem[\protect\citeauthoryear{Kazakov}{2009}]{conf/ijcai/Kazakov09}
Yevgeny Kazakov.
\newblock Consequence-driven reasoning for Horn {SHIQ} ontologies.
\newblock In {\em Proc.\ of IJCAI}, pages 2040--2045, 2009.

\bibitem[\protect\citeauthoryear{Lutz}{2007}]{Lutz-DL-07}
Carsten Lutz.
\newblock Inverse roles make conjunctive queries hard.
\newblock In {\em Proc.\ of DL}, volume 250 of {\em CEUR-WS}, 2007.

\end{thebibliography}
\end{document}